%% file: main.tex
\documentclass{article}
\PassOptionsToPackage{numbers, compress}{natbib}

\usepackage{arxiv}
\usepackage{times}

\title{When Both Layers Learn: Training Dynamics of \\ Representing Linear Models via ReLU Networks}

\author{%
Berk Tinaz$^{1}$ \quad
Changzhi Xie$^{2}$ \quad
Mahdi Soltanolkotabi$^{1,2}$\\
$^{1}$Department of Electrical and Computer Engineering\\
$^{2}$Department of Computer Science\\
University of Southern California\\
Los Angeles, CA, USA\\
\texttt{\{tinaz, changzhi, soltanol\}@usc.edu}
}

\usepackage{natbib}
\usepackage{microtype}
\usepackage{booktabs} 
\usepackage{float}
\usepackage[hyperfootnotes=false]{hyperref}
\usepackage{amsmath}
\usepackage{amssymb}
\usepackage{mathtools}
\usepackage{amsthm}
\usepackage{graphicx, subcaption}
\usepackage[normalem]{ulem}
\usepackage[utf8]{inputenc} 
\usepackage[T1]{fontenc}    
\usepackage{booktabs}       
\usepackage{amsfonts}       
\usepackage{nicefrac}       
\usepackage{microtype}      
\usepackage{xcolor}         
\usepackage{bm}
\usepackage{pifont}
\usepackage{multirow}
\usepackage{footmisc}

\usepackage{caption}
\usepackage{bm}
\usepackage{bbm}
\usepackage{graphicx}
\usepackage{tcolorbox}

\input{macros}

\begin{document}
\maketitle

\input{sec/abstract}  
\input{sec/intro}
\input{sec/theoretical_results}
\input{sec/experiments}
\input{sec/related_work}
\input{sec/overview}
\input{sec/acknowledgments}
\newpage
\bibliographystyle{ieeenat_fullname}
\bibliography{references.bib,Bibfiles.bib,Bibfiles2.bib,literature.bib}

\newpage
\appendix
\input{sec/appendix}
\end{document}

%% file: macros.tex
\newtheorem{theorem}{Theorem}
\newtheorem{corollary}[theorem]{Corollary}
\newtheorem{lemma}[theorem]{Lemma}

\newtheorem{definition}[theorem]{Definition}

\newcommand{\calL}{{\mathcal{L}}}

\newcommand{\R}{\mathbb{R}}

\newcommand{\E}{\mathbb{E}}

\newcommand{\cbr}[1]{\left\{#1\right\}}
\newcommand{\rbr}[1]{\left(#1\right)}
\newcommand{\bbr}[1]{\left[#1\right]}
\newcommand{\inner}[2]{\left\langle #1,#2 \right\rangle}

\renewcommand{\E}{\mathbb{E}}
\newcommand{\x}{\bm{x}}
\renewcommand{\v}{\bm{v}}
\newcommand{\W}{\bm{W}}
\newcommand{\w}{\bm{w}}

\renewcommand{\a}{\bm{a}}
\renewcommand{\b}{\bm{b}}

\newcommand{\dirw}{\overline{\bm{w}}}

\newcommand{\1}{\mathbbm{1}}
\newcommand{\projw}[1]{\mtx{P}_{\vct{w_{#1}}^\perp}}
\newcommand{\vct}[1]{\bm{#1}}
\newcommand{\mtx}[1]{\bm{#1}}

\newcommand{\fnorm}[1]{\left\| #1\right\|_F}
\newcommand{\twonorm}[1]{\left\| #1\right\|}

\renewcommand{\v}{\bm{v}}

\newcommand{\h}{\bm{h}}

\renewcommand{\a}{\bm{a}}

\renewcommand{\b}{\bm{b}}

\newcommand{\Sph}{\mathbb{S}}
\newcommand{\Pbb}{\mathbb{P}}
\newcommand{\ip}[2]{\langle #1,#2\rangle}

\newcommand{\opnorm}[1]{\left\lVert #1\right\rVert}

%% file: sec/abstract.tex
\begin{abstract}%
In this paper, we study the gradient descent dynamics for jointly training \emph{both} layers of a one-hidden-layer ReLU network to fit a linear target function. Concretely, we consider a realizable setting where inputs are drawn i.i.d.~from a Gaussian distribution and labels follow a planted linear model. This stylized framework captures salient features of end-to-end training in inverse problems and certain auto-encoder models. Despite its apparent simplicity, the dynamics remain poorly understood, in part because the loss landscape contains multiple non-strict saddle points, making it unclear why gradient descent from random initialization reliably escapes bad stationary regions. We provide a detailed characterization of the optimization landscape and prove that gradient descent from a moderately small random initialization-\emph{simultaneously training both layers}-converges to a global minimizer at a linear rate with order-wise optimal sample complexity. Our analysis tracks the trajectory through three phases: an \emph{alignment phase} in which hidden weights progressively align with the planted direction while the output weights maintain the correct sign pattern; a \emph{growth phase} in which the norms of both layers increase while preserving alignment; and a \emph{local refinement phase} in which the aligned neurons rapidly converge to the planted direction, yielding fast local convergence. To rigorously show that GD avoids non-strict saddles, we develop trajectory-level control arguments for the end-to-end dynamics. In addition, we establish novel uniform concentration results that hold along the entire trajectory, and are essential for obtaining order-wise optimal sample complexity. We corroborate our theory with extensive experiments across a range of configurations.
\end{abstract}

%% file: sec/intro.tex
\section{Introduction}
\subsection{Motivation}
End-to-end training of neural networks (NNs) via Gradient Descent (GD) has recently achieved remarkable success on many tasks. Of particular interest, these models have been adopted to solve inverse problems by taking the measurements as input and mapping them directly to the desired signal with successful scientific applications in computer vision \citep{ledig2017photo, wang2018image}, MRI reconstruction \citep{sriram2020end, fabian2022humus}, sparse-view computed tomography (CT) \citep{jin2017deep}, and phase retrieval \citep{hand2018phase}. These models not only fit the training data but also appear to capture useful features and nuanced priors that enable them to generalize to unseen test examples. Despite this empirical success, the reasons behind the success of NNs for end-to-end training and how they can extract useful features from data remain unclear.

Perhaps the most classical form of end-to-end training is that arising in autoencoder type problems, where the goal is to teach a neural network to learn a linear mapping (e.g., identity for autoencoders). Surprisingly, the dynamics of training such a model are not well understood for nonlinear models. For linear networks, a classical result by \citet{baldi_hornik} provided a complete characterization, showing how gradient descent recovers the principal components of the data. In contrast, understanding the dynamics of non-linear encoders has remained an open and challenging problem, even for simple target functions.  In this paper, we aim to take a step towards a systematic understanding of the training dynamics of such problems by addressing the following question:

\begin{tcolorbox}
\noindent\textit{How do the dynamics of training ReLU neural networks with gradient descent starting from random initialization facilitate learning simple priors and structures such as linear target functions?}
\end{tcolorbox}

Understanding this question requires reasoning not only about the final solution reached by GD, but about the entire trajectory of the optimization process. Recent empirical work suggests that several phenomena observed during neural network training, including grokking (or delayed generalization) \citep{power2022grokkinggeneralizationoverfittingsmall}, are closely tied to the temporal evolution of gradient descent. In such settings, models may fit the training data well before exhibiting improved generalization, indicating that learning can unfold through distinct stages over the course of optimization. This perspective motivates a careful, trajectory-level analysis even in simple problem settings.

Despite significant recent progress in understanding neural networks (especially shallow networks) \citep{chizat2019lazy, soltanolkotabi2018theoretical, jacot2018neural, du2018gradient, ongie2019function} (See Section \ref{sec:related_work} for in-depth discussion on related work), many aspects of the dynamics of GD and how it facilitates learning remain mysterious even in seemingly simple settings. A particularly simple one involves learning linear target functions via GD, that is, teaching a one-hidden-layer network to mimic the output of a simple linear model. Surprisingly, understanding the dynamics of GD in this simple setting has remained elusive. Although there are many results on learning specific target functions such as ReLUs \citep{xu2023overparameterization, soltanolkotabi2017learning} and polynomials \citep{damian2022neural}, these results typically exclude linear function classes. In fact, many of the existing papers use a pre-processing step or alter the early optimization trajectory to avoid complications arising from the dynamics of learning linear functions or genuinely training both layers \citep{damian2022neural}. This is in part due to the fact that the optimization landscape of learning linear target functions contains multiple non-strict saddle points (i.e. where the gradient vanishes and the Hessian is PSD but has a $0$ eigenvalue) requiring a subtle trajectory analysis to ensure GD avoid these bad points (See Section \ref{linchallenge} for further details). We note that despite the simple formulation, quite a few interesting scenarios, including autoencoder training dynamics, are captured in this framework.
\\
\\
Our main contributions are as follows:
\begin{itemize}
    \item We present one of the first works that analyzes training dynamics of learning \emph{both} layers in a one-hidden-layer ReLU network in a practical regime. That is, we do not use pre-processing or alter the early optimization trajectory to avoid complications that arise from non-linear training dynamics of optimizing both layers.
    \item We develop a theory for running GD on the NN with moderately small initialization, demonstrating exact convergence to the ground truth at a linear rate and with an optimal sample complexity that scales linearly in the number of parameters.  That is, we show that the inner weights of the NN recover the target directions \textit{exactly}, while the outer layer maintains the correct sign pattern. 
    \item As detailed further in Section \ref{linchallenge} the training landscape studied in this paper contains multiple non-strict saddles. To prove that the trajectory of GD from moderately small random initialization avoids these bad stationary points, we develop new techniques to control the GD trajectory which we combine with intricate uniform concentration bounds. In particular, our refined analysis tracks the trajectory through three phases (\emph{alignment}, \emph{growth}, and \emph{local refinement} phases). We believe our refined trajectory analysis may have broader implications for the analysis of non-convex optimization problems involving non-strict saddles. 
    \item  Since gradient descent repeatedly reuses the same finite dataset across all phases, the iterates become statistically dependent on the samples. We address this by proving new uniform concentration bounds for the gradient along the entire optimization trajectory, holding simultaneously for all iterates encountered by GD. A key component of our uniform concentration result is that the accuracy of the concentration increases as we get closer and closer to the global optima. These refined bounds are a key technical ingredient for achieving order-optimal sample complexity.
    \item We further corroborate our results with various experimental investigations.
\end{itemize}

\subsection{Problem Formulation} \label{sec:problem_formulation}
We first state the general family of problems of interest in this paper. 
\\
\\
\noindent\textbf{Data Model} -- We assume there are $n$ pairs of training data consisting of input features $\vct{x}_i \in \R^d$ and corresponding targets $y_i \in \R$. As mentioned before, we consider the class of linear models where the relationship between $\vct{x}_i$ and $y_i$ is given by the equation: $y_i = \vct{a}^T \vct{x}_i$ where $\a \in \R^d$ is the labeling vector. Conceptually, $\a$ is the target \textit{direction} that our predictor should \textit{learn}. For our theoretical analysis we assume the data points $\vct{x}_i$ are drawn i.i.d.~according to a standard normal distribution $\mathcal{N}(\vct{0},\mtx{I}_d)$.
\\
\\
\noindent\textbf{Network Model} -- We consider one-hidden-layer neural networks of the form $f\rbr{\vct{v}, \vct{W}, \vct{x}} := \vct{v}^T \phi\rbr{\mtx{W}\vct{x}}$ as our predictor. Here $k$ denotes the number of hidden-units, $\mtx{v} \in \R^{k}$ is the outer layer of the neural network, $\mtx{W} \in \R^{k \times d}$ is the inner layer of the neural network, and $\phi\rbr{\vct{z}}$ is the activation function. We refer to individual rows of $\vct{v}$/$\mtx{W}$ as $v_i$/$\vct{w}_i$ respectively, In this paper, we specifically consider neural networks with ReLU activation functions i.e. $\phi\rbr{\vct{z}} = \text{ReLU}\rbr{\vct{z}} = \max\rbr{0, \vct{z}}$, where $\max$ is applied to the input vector $\vct{z}$ element-wise. Furthermore, we focus on the exact parametrized setting, i.e. $k=2$, as a step towards understanding the behavior of the over-specified/parameterized setting with $k>2$ neurons.
\\
\\
\noindent\textbf{Training Loss} -- We minimize the squared loss between the target and the prediction
\begin{equation}\label{eq:emp_loss}
    \widehat{\calL}\rbr{\v, \W} = \frac{1}{2n} \sum_{i=1}^n\rbr{\v^T\phi\rbr{\W\x_i} - y_i}^2
\end{equation}
using gradient descent. For part of our theoretical analysis of GD, we also consider the population loss (i.e.~infinite data asymptotics as $n\rightarrow \infty$) with $\vct{x}$ drawn randomly from an isotropic Gaussian distribution $\vct{x} \sim \mathcal{N}\rbr{\vct{0}, \mtx{I_d}}$. Concretely, the population loss is given by
\begin{equation}\label{eq:pop_loss}
    \calL\rbr{\v, \W} = \frac{1}{2}\E_{\vct{x}}\bbr{\rbr{\sum_{i=1}^k v_i \phi \rbr{\vct{w}_i^T \vct{x}} - \a^T\x}^2}.
\end{equation}
\section{Landscape Analysis: Why is learning linear functions with ReLUs challenging?}\label{linchallenge}
\begin{figure}[h]
    \centering
    \begin{minipage}[t]{0.32\textwidth}
        \centering
        \includegraphics[width=\linewidth]{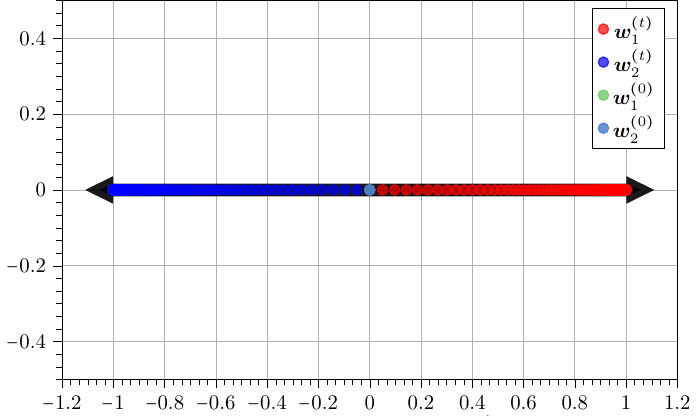}\\
        \textbf{(a)} Trajectory of convergence to the global optimum.
    \end{minipage}
    \hfill
    \begin{minipage}[t]{0.31\textwidth}
        \centering
        \includegraphics[width=\linewidth]{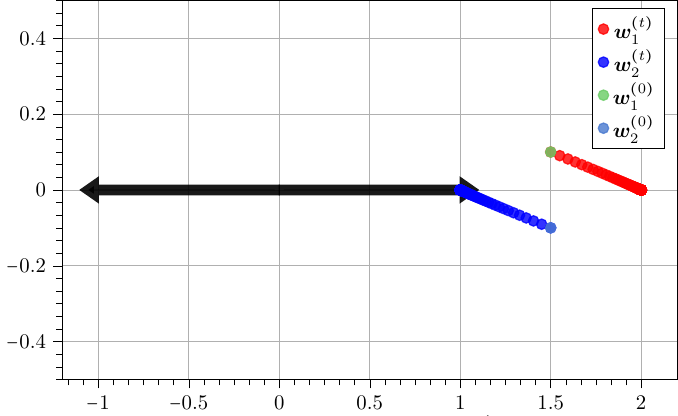}\\
        \textbf{(b)} Trajectory of convergence to one of the non-strict saddle points.
    \end{minipage}
    \hfill
    \begin{minipage}[t]{0.33\textwidth} 
        \centering    
        \includegraphics[width=\linewidth]{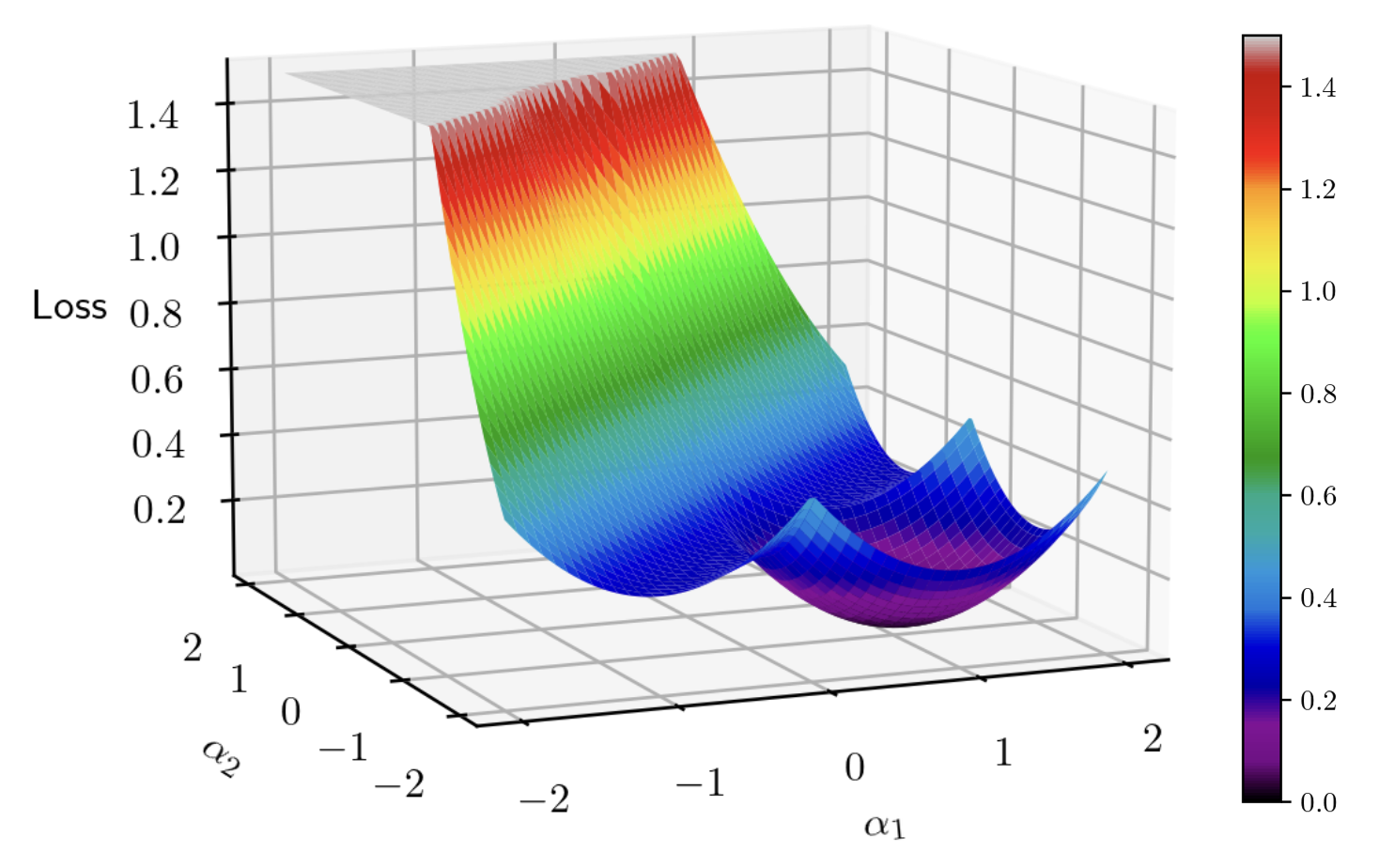}\\
        \textbf{(c)} 3D visualization of the loss landscape.
    \end{minipage}
    \centering
    \caption{\textbf{GD trajectories and population loss landscape.}
We run gradient descent on the population loss for a one-hidden-layer ReLU network with two hidden units and fixed output weights $\vct{v}=\begin{bmatrix}1,-1\end{bmatrix}^T$.
Panels (a) and (b) use two different initializations of $\vct{w}_1^{(0)}$ and $\vct{w}_2^{(0)}$.
To visualize the dynamics in 2D, we plot each neuron in the plane spanned by $\vct{a}$ (black arrows indicate $\pm\vct{a}$) and a randomly chosen direction orthogonal to $\vct{a}$ (y-axis).
In (a), a small initialization near the origin converges to the global optimum.
In (b), initializing near $1.5\vct{a}$ leads GD to $(\vct{w}_1,\vct{w}_2)=(\vct{a},2\vct{a})$, a non-strict saddle of~\eqref{eq:pop_loss}.
Finally, (c) plots the population landscape in the reduced slice $\w_1=\alpha_1\a$, $\w_2=\alpha_2\a$.}
    \label{fig:trajectories_global_and_local_optima}
\end{figure}

Despite the simplicity of the target function, the gradient descent dynamics in this setting are surprisingly subtle. The difficulty is that the loss landscape is riddled with non-strict saddle points. Indeed, infinitely many of them—creating large flat directions where naive intuition about descent can fail. The next theorem makes this phenomenon precise for the population loss.
\begin{theorem}[Landscape Characterization] \label{thm:landscape} For $v_1, v_2 > 0$, the stationary points of the population loss \eqref{eq:pop_loss} are either
\begin{enumerate}
    \item \underline{global optima}: $v_1 \w_1 = \a, \quad v_2 \w_2 = -\a$,
    \item or \underline{non-strict saddles}: $ v_1 \w_1 = \rbr{c+1} \a, \ v_2 \w_2 = c \a, \quad \text{where} \ c > 0 \ \text{or} \ c < -1$.
\end{enumerate}
\end{theorem}
We provide the proof of Theorem \ref{thm:landscape} in Appendix \ref{seq:landscape_calc}.

Theorem \ref{thm:landscape} above shows that, beyond the global minima, the population loss contains a continuum of stationary points forming non-strict saddle manifolds parameterized by $c$. In particular, for every $c>0$ and every $c<-1$, the equations $v_1 \w_1 = (c+1)\a$ and $v_2 \w_2 = c \a$ define a stationary point with flat directions in the loss. Thus the landscape is highly degenerate: instead of isolated critical points, there are infinitely many saddle regions that gradient descent can enter and move along without encountering negative curvature. This proliferation of flat saddles is the primary geometric obstruction to analyzing the global behavior of gradient descent.

In Figure~\ref{fig:trajectories_global_and_local_optima}, we illustrate how the initialization determines whether GD converges to a global optimum or drifts toward a non-strict saddle. Figure~\ref{fig:trajectories_global_and_local_optima}(a) shows a trajectory that converging to the global optimum, while Figure~\ref{fig:trajectories_global_and_local_optima}(b) shows a trajectory that stalls near a saddle. To further visualize the landscape, Figure~\ref{fig:trajectories_global_and_local_optima}(c) fixes $v_1=v_2=1$ and plots the loss in the reduced two-dimensional slice $\w_1=\alpha_1\a$, $\w_2=\alpha_2\a$. In this slice, the gradient vanishes along the $\alpha_1-\alpha_2=1$ \textit{valley}, even though the loss remains strictly positive.

%% file: sec/theoretical_results.tex
\section{Main Result: Convergence of the Gradient Descent Trajectory} \label{sec:theory_empirical}
We now present our main result, which characterizes the training dynamics when \textit{both} layers of a ReLU network are trained in the practical empirical regime.
\begin{theorem}[Convergence of GD Trajectory] \label{thm:main}
Suppose we have $n$ feature vectors $\cbr{\x_1, \cdots, \x_n}$ that are sampled i.i.d. according to a Gaussian distribution $\vct{x}_i\sim\mathcal{N}(\vct{0},\mtx{I}_d)$. We assume the corresponding outputs are generated according to a linear target function of the form $y_i=\vct{a}^T\vct{x}_i$, where  $\vct{a}\in\R^d$ is an arbitrary weight vector. To learn this linear function, we fit a one-hidden-layer ReLU network with two hidden nodes
\begin{align*}
\vct{x}\mapsto \vct{v}^T\text{ReLU}(\mtx{W}\vct{x})=v_1\text{ReLU}(\vct{w}_1^T\vct{x})-v_2\text{ReLU}(\vct{w}_2^T\vct{x}).
\end{align*}
by minimizing the empirical loss
\begin{align*}
    \widehat{\calL}\rbr{\v, \W} = \frac{1}{2n} \sum_{i=1}^n \left(\v^T\text{ReLU}\rbr{\mtx{W}\vct{x}_i} - \a^T\vct{x}_i\right)^2.
\end{align*}
over $\vct{v}=\begin{bmatrix} v_1, -v_2\end{bmatrix}^T$ and $\mtx{W}=\begin{bmatrix} \vct{w}_1, \vct{w}_2\end{bmatrix}^T\in\R^{2\times d}$ using gradient descent with step size $\mu:=\frac{\bar{\mu}}{\|\a\|}$ with $\bar{\mu}\le \frac{\mu_0}{\ln\rbr{\frac{\sqrt{\twonorm{\a}}}{\sigma}}}$:
\begin{align*}
    \W^{(\tau+1)}=\W^{(\tau)}-\mu \nabla_{\W}\widehat{\mathcal{L}}(\v^{(\tau)},\W^{(\tau)}), \quad \v^{(\tau+1)}=\v^{(\tau)}-\mu \nabla_{\v}\widehat{\mathcal{L}}(\v^{(\tau)},\W^{(\tau)})
\end{align*} 
Assume the initialization
\[
\w_1^{(0)},\w_2^{(0)}\sim\mathcal N\!\Bigl(0,\tfrac{\sigma^2}{d}\mtx{I}_d\Bigr),
\qquad
v_1^{(0)},v_2^{(0)}\sim \frac{\sigma}{\sqrt{d}}\xi,\quad
\xi^2\sim\chi_d^2,
\]
with $\sigma\le\sigma_0\sqrt{\|\a\|}$, $\chi_d^2$ a chi-squared distribution with $d$ degrees of freedom, and define $\W^*=[\a,-\a]^T$.
As long as the number of training samples satisfies $n \geq Cd$, then with probability at least $1-Ce^{-cd}$ there exists 
$T\ge c'\frac{1}{\bar{\mu}}\ln\rbr{\frac{\sqrt{\twonorm{\a}}}{\sigma}}$
such that for all iterations $\tau > T$,
\begin{align*}
    \twonorm{\text{diag}\rbr{\begin{bmatrix}
        v_1^{(\tau)} \\ v_2^{(\tau)}
    \end{bmatrix}}\W^{\rbr{\tau}}-\W^*}_F^2 \leq \tilde{c} \rbr{1-c\bar{\mu}}^{\rbr{\tau - T}} \twonorm{\text{diag}\rbr{\begin{bmatrix}
        v_1^{(T)} \\ v_2^{(T)}
    \end{bmatrix}}\W^{\rbr{T}}-\W^*}_F^2.
\end{align*}
Here, $\mu_0, \sigma_0, c, \tilde{c}, c', C$ are fixed numerical constants independent of any problem dimensions. 
\end{theorem}

This theorem shows that gradient descent can provably train a fully end-to-end one-hidden-layer ReLU network to learn a linear target from finitely many samples, despite the highly degenerate and saddle-rich optimization landscape. In particular, the result gives a global convergence guarantee for simultaneous optimization of the hidden and output weights; going beyond analyses that rely on effectively fixed features or only local perturbations around initialization. Starting from a small random initialization with the standard $\sigma/\sqrt d$ scaling---consistent with common ``default'' initializations used in practice---the two student neurons $\w_1$ and $\w_2$ rapidly align with the ground-truth direction $\a$, after which the \emph{effective parameters} $\text{diag}\rbr{\begin{bmatrix} v_1 \\ v_2 \end{bmatrix}}\W$ converge geometrically to the planted solution. This linear-rate convergence kicks in after only a short burn-in period of $T\ge c'\frac{1}{\bar{\mu}}\ln\rbr{\frac{\sqrt{\twonorm{\a}}}{\sigma}}$ iterations and requires only $n\ge Cd$ samples, which is information-theoretically optimal.

We note that the \emph{same} finite dataset is reused across all iterations of gradient descent. Controlling the resulting dependence between the iterates and the samples requires uniform concentration arguments that hold \emph{along the entire trajectory}, i.e., controlling population--empirical deviations simultaneously for all iterates visited by GD rather than at a fixed parameter value. Finally, the $\chi^2$-based initialization for the output weights is used for technical convenience. One can alternatively initialize $v_1$ and $v_2$ as Gaussians with variance $\sigma^2/2$; the same qualitative convergence behavior persists, but the success probability degrades to a fixed constant, rather than $1-Ce^{-cd}$. This constant failure with Gaussian initialization is unavoidable for $k=2$. For large $k$, scaling the initialization with $\frac{\sigma^2}{k}$ yields failure probability decaying as $e^{-k}$. We therefore  use a slight non-Gaussian modification in the initialization to demonstrate that except for this $k=2$ artifact our result holds with much higher probability.

%% file: sec/experiments.tex
\section{Experiments} \label{sec:experiments}
We run experiments on various output dimension (denoted with $r=1$ vs. $r>1$), and initialization scale (small vs large). In this section we show experimental results for single output $r=1$ case and refer the reader to Appendix \ref{apx:multidimension} for multi-output $r>1$ results. We use PyTorch for experiments and unless mentioned otherwise, network weights are initialized with Xavier Normal initialization (for a matrix $\mtx{W} \in \R^{r \times d}$,$\mtx{W}_{ij} \sim \mathcal{N}\rbr{0, \frac{2}{r+d}}$). 

In order to change the initialization scale, we multiply the default initialization with a positive scalar $\sigma$. For \textit{small} initialization experiments, we use $\sigma=10^{-8}$, otherwise it is set to $\sigma=1$. We set $d=100$ and $\mu=0.1$. All experiments are run on a server with an Intel Xeon Gold 5220R CPU. We would like to stress that even though the visualizations in this paper are based on a single trial, we ran these experiments for different random seeds and the behavior of the visualizations did not change.

In experiments w.l.o.g.~we choose $\vct{a} = \vct{e}_1$ where $\vct{e}_1$ is the first standard basis in $\R^d$. This does not effect the results due to the rotational symmetry of isotropic Gaussian distribution of which $\vct{x}$ are drawn from. Note that this implies $\twonorm{\vct{a}}=1$ in our experiments. Finally, in this section we focus our experiments on the population loss. Similar results continue to hold in the empirical case with moderate sample sizes i.e. when $n\ge crd$ with $c$ a sufficiently large constant.

When the model is exactly parameterized with two hidden nodes ($k=2$), we empirically see that the model cannot converge to the global optima consistently. When it does, $\vct{v}^{\rbr{\infty}}$ indeed becomes $\pm 1$ and $\vct{w}_1^{\rbr{\infty}}$ and $\vct{w}_2^{\rbr{\infty}}$ recover $\pm \vct{a}$ exactly. For the remaining time, the GD iterates converge to one of the many stationary points of this problem similar to the depiction in Figure \ref{fig:trajectories_global_and_local_optima} (part b). We further observe that iterates get stuck only when $\vct{v}_1^{(0)}$ and $\vct{v}_2^{(0)}$  both have the same signs which happens with probability $\frac{1}{2}$. This is also the reason for why we fix the correct sign pattern at initialization in Theorem \ref{thm:main}.

When $k>2$, the probability of \textit{all}
$\vct{v}_i$'s having the same sign decreases rapidly. Therefore, iterates typically converge to the global optima. However, in this case global minima is not unique anymore. To demonstrate this, consider the case where there are four hidden units ($k=4$) instead of two. The trajectory of the inner weights across GD iterations is depicted in Figure \ref{fig:trajectories}. 

\begin{figure}[h]
    \centering
    \begin{minipage}[t]{0.45\textwidth}
        \centering
        \includegraphics[width=0.95\linewidth]{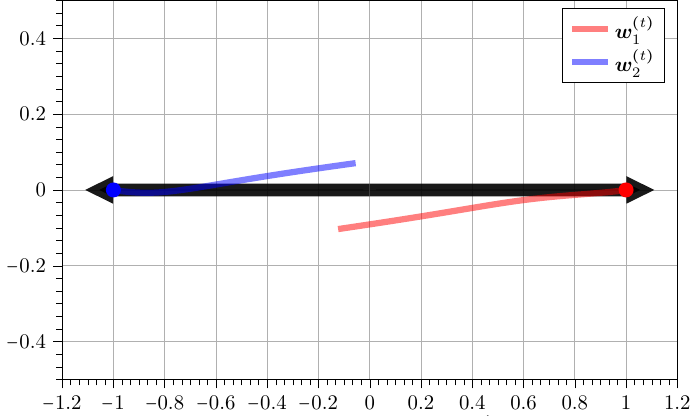}\\
        \textbf{(a)} Trajectory of neurons when $k=2$.
    \end{minipage}
    \hfill
    \begin{minipage}[t]{0.45\textwidth}
        \centering
        \includegraphics[width=0.95\linewidth]{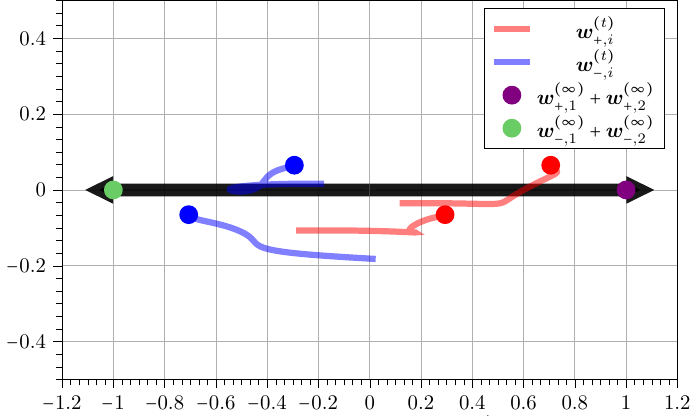}\\
        \textbf{(b)} Trajectory of neurons when $k>2$.
    \end{minipage}
    \centering
    \caption{\textbf{Trajectory of neurons for different values of} $k$. We run gradient descent updates on the population loss. A randomly selected orthogonal direction to $\vct{a}$ is shown for the y-axis in order to visualize the neurons in 2D. Black arrows indicate $\pm \vct{a}$ direction. We use colors \textcolor{red}{red} and \textcolor{blue}{blue} to indicate whether $\vct{v}_i$ corresponding to $\vct{w}_i$ is positive or negative respectively. Points at the end of each trajectory denotes the final weight GD converges to.}
    \label{fig:trajectories}
\end{figure}

We observe that while no individual $\vct{w}_i$ align itself with $\pm \vct{a}$ direction, grouping hidden units based on their corresponding signs in $\vct{v}$ and summing them recovers $\pm\vct{a}$ \textit{exactly} (\textcolor{purple}{purple} and \textcolor{green}{green} points in Figure \ref{fig:trajectories}). Although not depicted here, we have tried various values for $k>2$ and the observation that grouping weights recover $\pm \vct{a}$ was consistent. This suggests that combining node aggregation technique from \citep{li2024feature} with our proof strategy may extend our results for the $k > 2$ setting. We leave this to future work.

%% file: sec/related_work.tex
\section{Related Work} \label{sec:related_work}
There is a large body of work on developing global convergence guarantees for nonconvex problems. We review this literature and compare the differences with the setting discussed in this paper.
\\
\\
\noindent\textbf{Nonconvex low-rank matrix recovery}: In low-rank matrix recovery, numerous studies have shown that nonconvex gradient descent, when initiated with spectral initialization, can effectively solve low-rank reconstruction problems across various domains. This includes phase retrieval \citep{wirtinger_flow,truncated_wirtinger_flow,chen_implicit_regularization}, matrix sensing \cite{tu2016low}, blind deconvolution \cite{ling_blind_deconv,ling_demixing}, and matrix completion \cite{xiaodong_matrixcompletion}. In practice, random initialization is frequently employed instead of specialized spectral initialization methods. As a result, more recent literature \cite{landscape_phaseretrieval,ge2016matrix_spurious,zhang2019sharp}, have turned to analyzing the loss landscape. These studies demonstrate that, despite their non-convex nature, these loss landscapes remain well-behaved under certain assumptions. Specifically, they contain no spurious local minima (i.e., all minimizers are global minima), and saddle points exhibit a strict direction of negative curvature (also known as strict saddle points) \cite{sun2015nonconvex}. Then specialized truncation or saddle escaping algorithms such as trust region, cubic regularization \cite{nesterov2006cubic, nocedal2006trust} or noisy (stochastic) gradient-based methods \cite{jin2017escape, ge2015escaping, raginsky2017non, zhang2017hitting} are deployed to provably find a global optimum. In contrast to the above literature, the landscape of our loss contain non-strict saddle points. Furthermore, we do not seek any modification to the initialization or the GD updates. Indeed, our result holds with moderately small initialization. As mentioned  earlier, we are able to establish this result by developing intricate control of the GD updates throughout the trajectory.
This trajectory-level perspective (i.e. multi-phase analysis) is also explored in recent works on gradient descent dynamics and implicit bias under large learning rates \citep{wang2022largelearningratetames, wang2023goodregularitycreateslarge}, see also additional prior work \citep{stoger2021, soltanolkotabi23a-implicit-balancing} on this topic. However, these works focus on matrix factorization and more general nonconvex objectives rather than neural network training.
\\
\\
\noindent\textbf{Gradient-based analysis for neural networks:} A recent line of work is concerned with connecting the analysis of neural network training with the so-called neural tangent kernel (NTK) \cite{jacot2018neural,oymak2019overparameterized,oymak2020towards,du2019gradient,arora2019fine}. The core idea is that with sufficiently large initialization, a neural network can be approximated by its linearization around the origin. This approximation facilitates linking neural network analysis to the well-established theory of kernel methods. This approach is sometimes referred to as lazy training since, under such initialization, the network parameters remain close to their initial values throughout training. However, some research suggests that NTK-based analysis alone may not fully account for the practical success of neural networks. For instance, \cite{chizat2019lazy} presents empirical evidence indicating that reducing the initialization size can lead to lower test error. Similarly, \cite{ghorbani2020neural} observes a performance gap between neural networks and their NTK counterparts, with the gap widening when the covariance matrix is isotropic. We note that in an NTK analysis the parameters stay close to the initialization which is not the case in our setting. Furthermore, an NTK analysis that relies on linearization can not deal with trajectory analysis that avoids local optima. Indeed, an NTK analysis will not yield the directional convergence established in this paper. So in this sense our result can be viewed as going beyond the lazy training in NTK theory. 
\\
\\
\noindent\textbf{Beyond NTK and learning of specific target functions.} Recent work carries out analysis of neural networks beyond NTK regime including \cite{damian2022neural,ba2022high, lee2024neural, xu2023overparameterization}. Many of these results also focus on learning specific target functions such as ReLUs \citep{xu2023overparameterization}, \citep{soltanolkotabi2017learning} and polynomials \citep{damian2022neural}. These results however typically exclude linear function classes and do not directly involve analysis that requires avoiding bad stationary points explicitly. In fact, many of the existing papers use a pre-processing step or alter the early optimization trajectory to avoid complications arising from the dynamics of learning linear functions \citep{damian2022neural}. In contrast, our focus is directly dealing with such intricacies. 

Among these papers, perhaps the closest to ours in spirit is \citep{xu2023overparameterization} which studies the problem of fitting an overparameterized ReLU network to a single ReLU target function with a one dimensional output. Our one-dimensional result can be viewed as a generalization of this work (in particular their exact parametrization result) where the target function has two ReLUs with a particular pattern. This is due to the fact that any linear function of the form $\vct{a}^T\vct{x}$ can also be written as a difference of two ReLUs: $v_1\text{ReLU}\rbr{\frac{1}{v_1}\vct{a^T}\vct{x}} - v_2\text{ReLU}\rbr{\frac{-1}{v_2}\vct{a^T}\vct{x}}$ for any $v_1,v_2 >0$. The addition of this new ReLU with a negative sign introduces non-strict saddle points and various intricacies in the landscape necessitating a completely different analysis. However, compared to \citep{xu2023overparameterization} we do not study the effect of overparameterization theoretically. Our empirical results in Section \ref{sec:experiments} suggest that such an extension may be possible.

We highlight that besides \citet{xu2023overparameterization}, there are several other works on learning a single neuron \cite{yehudai2022learningsingleneurongradient, vardi2022learningsingleneuronbias, chistikov2023learningneuronshallowrelu} and variants \cite{brutzkus2017globallyoptimalgradientdescent}. As explained before, such results cannot be used to analyze linear targets due to the interaction terms between positive and negative ReLU neurons. Furthermore, we note that the landscape for fitting a single ReLU is fundamentally different as it contains only a single basin of attraction (albeitt a non-convex one). In contrast, as discussed earlier the landscape in our problem include non-strict saddle points significantly complicating gradient descent analysis.

We would also like to discuss the difference between our work and a few other papers \citet{zhong2017recoveryguaranteesonehiddenlayerneural,zhang2018learningonehiddenlayerrelunetworks, zhu2025how, ren2025emergencescalinglawssgd} that have planted one-hidden layer models. These papers differ in at least one of three ways focusing on (1) local analysis, (2) have sub-optimal sample complexity, and/or (3) assume non-negative outer layer weights. For instance, \citet{zhong2017recoveryguaranteesonehiddenlayerneural} utilize tensor initialization, performing a local analysis rather than a global GD analysis. This local analysis however can not be used to analyze the linear target setting. Indeed, as noted in Remark 4.3 of their work, their analysis requires $\W^*$ to be full-rank which does not hold in the linear setting (where the rows of the weight matrix are negatives of each other leading to a minimum singular value is zero). Furthermore, this result also requires resampling the data points at each iteration to ensure convergence of gradient descent where as we use the same samples across all iterations. On a related note, their sample complexity has polynomial dependency on many problem parameters (Theorem 4.2) whereas our proof only requires sample size linear in input dimension $d$. 

Similarly, \citet{zhang2018learningonehiddenlayerrelunetworks} provide a local analysis of GD when the outer layer weights are fixed to be all ones. They also utilize results of \citet{zhong2017recoveryguaranteesonehiddenlayerneural} and share similar limitations in terms of the rank requirement on $\W^*$. Thus this result can not be used in the linear target setting even for a local analysis. While they improve the sample complexity of \citep{zhang2018learningonehiddenlayerrelunetworks} by getting rid of the resampling trick, they still end up with a sample complexity polynomial in width of the network. 

\citet{wu2018no} consider the setting when student and teacher networks both have 2 neurons. In particular, when the teachers are \textit{orthogonal}, and the outer weights are all ones; they demonstrated an interesting result that the landscape is benign and all saddles are strict. In contrast, the landscape in our problem include non-strict saddle points significantly complicating gradient descent analysis. In more recent work, \citet{ren2025emergencescalinglawssgd} study the complexity of learning a planted model with \textit{orthogonal} planted directions, quadratic activations, and non-negative outer weights. They obtain interesting results on the scaling laws of the MSE loss via a multi-phase analysis. However, this problem setting is substantially different due to the difference between the activation and the orthogonal weights in the planted model that makes the landscape benign per above discussion. More recently, \citet{zhu2025how} also consider learning multiple \textit{orthogonal} ReLU neurons in a teacher-student framework with outer layer weights fixed to all ones. As just discussed, having orthogonal teacher weights leads to a much more benign landscape. Moreover, assumptions in the aforementioned works strictly exclude the linear target setting, where the outer layer must contain negative coefficients. Furthermore, they impose strong restrictions on the initialization. Specifically, they look at the convergence after ``weak alignment" where for each student neuron there exists only one teacher neuron that is not near perpendicular. Our results on the other hand can handle random initializations where student neurons \textit{could} be perpendicular to the target direction. That said, their analysis can handle over-parametrization ($k \gg k^*$) and teacher networks with more than $2$ neurons. 

In recent and independent work, \citet{boursier2025simplicitybiasoptimizationthreshold} also consider the problem of learning linear target functions. The authors demonstrate an interesting result: despite over-parametrization, the sum of positive (resp. negative) neurons aligns with the OLS estimator obtained from the ``positive'' (resp. negative) subset of the data. To prove this, the authors impose heavy restrictions on the data distribution (in particular, Conditions 3 and 4 in their paper) to essentially align the data with the target direction and avoid changes in the activation cone. We quote the authors:

“However, item 3 is quite restrictive: it is needed to ensure that the volume of the activation cone containing $\beta^*$ does not vanish when $n \rightarrow \infty$. A similar assumption is considered by Chistikov et al. (2023); Tsoy and Konstantinov (2024), for similar reasons. Additionally, Condition 4 ensures that $\mathbb{E}_x[xx^T]\beta^*$ and $\beta^*$ are in the same activation cone. This assumption allows the training dynamics to remain within a single cone after the early alignment phase, significantly simplifying our analysis.”

In contrast, we demonstrate feature learning in the linear target setting by performing a full characterization of GD dynamics with a generic data distribution and initialization without any of the restrictive assumptions mentioned above.

%% file: sec/overview.tex
\section{Overview and Key Ideas of the Proof} \label{sec:proof_overview}
In this section, we outline the main ideas underlying our analysis. As mentioned previously, a major challenge is that the optimization landscape is riddled with non-strict saddle points that gradient descent can get stuck in. Thus, our analysis requires a very refined control of the trajectory to guarantee that the iterates escape these saddle regions. We will show that the trajectory of full-batch gradient descent partitions into three distinct phases discussed below. Figure~\ref{fig:pop_3_phase} illustrates the three phases and their interaction.

\begin{figure}[h]
    \centering
    \includegraphics[width=0.9\linewidth]{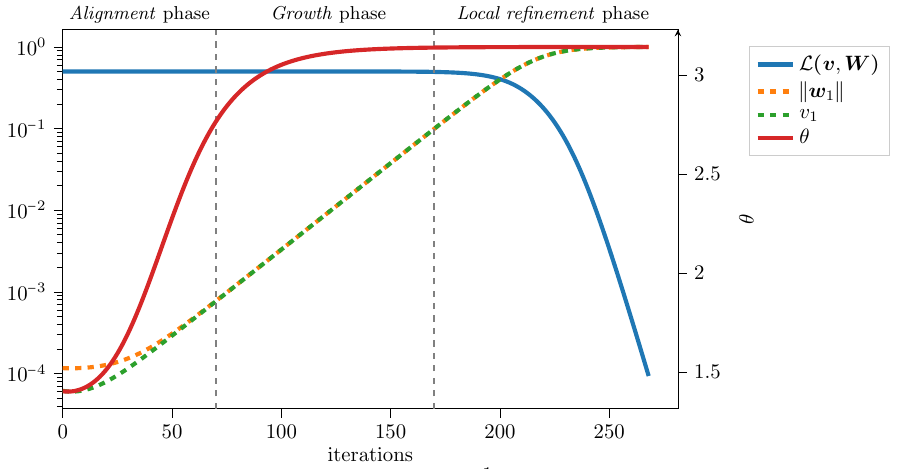}
    \caption{\textbf{Phases of GD Trajectory}. We run gradient descent updates on the population loss with small initialization $\sigma=10^{-4}$. We track the population loss $\calL$ (\textcolor{blue}{blue}), norms $v_1, \twonorm{\w_1}$ (\textcolor{green}{green} and \textcolor{yellow}{yellow}), and the $\theta$ -- i.e. angle between $\w_1$, $\w_2$ -- (\textcolor{red}{red}). For visualization purposes $\theta$ uses the right vertical axis. $v_2$ and $\twonorm{\w_2}$ behave similarly but omitted for clarity.}
    \label{fig:pop_3_phase}
\end{figure}

\begin{description}
    \item[(1) Alignment phase (Section \ref{phaseI}).] Starting from a small random initialization, we show that the hidden weights progressively align with the planted direction while the output weights maintain the correct sign pattern. 

    \item[(2) Growth phase (Section \ref{phaseII}).] Once sufficient alignment has been established, we prove that the norms of both the hidden and output layers grow in a coordinated fashion while preserving this alignment. This phase drives the effective parameters toward the correct scale and pushes the iterates away from flat saddle regions of the loss landscape. A key technical challenge here is to show that gradient descent does not drift into spurious stationary points despite the non-strict nature of these saddles.

    \item[(3) Local refinement phase (Section \ref{phaseIII}).] After the alignment and growth phases we enter a well-behaved region of the planted solution, where the dynamics become locally well-conditioned. In this phase, We show that the aligned neurons then converge rapidly to the ground-truth direction, and the effective parameters enjoy a linear rate of convergence to the global minimizer.
\end{description}

Throughout all three phases, the same finite dataset is reused across iterations. To control the resulting dependence between the iterates and the samples, we establish new trajectory-level uniform concentration bounds that hold simultaneously for all points visited by gradient descent. These results are crucial for obtaining order-wise optimal sample complexity. We give an overview of these uniform concentration results in Section \ref{unifconcen}. Before we detail the specific phases of the trajectory, we also need to establish two sets of key identities. The first set demonstrates a specific property of balancedness between the inner and outer weights (Section \ref{imbalance}). The second set concerns the stability of our training dynamics, ensuring that the evolution is monotonic in the sense that once the iterates enter a new phase, they do not revert to a previous one (Section \ref{stability}). We begin with some quick notation used throughout our proofs.

\subsection{Notation}
In this section we gather some simple notation used in our proofs. As a reminder we use $\widehat{\mathcal{L}}$ and $\mathcal{L}$ to denote the empirical and population losses, respectively. We use 
\begin{align*}
 \Delta\mathcal{G}_1:=\frac{2}{v_1}\left(\nabla_{\w_1} \widehat{\mathcal{L}}-\nabla_{\w_1}\mathcal{L}\right)\quad\text{and}\quad  \Delta\mathcal{G}_2:=\frac{2}{v_2}\left(\nabla_{\w_2} \widehat{\mathcal{L}}-\nabla_{\w_2}\mathcal{L}\right) 
\end{align*}
to denote the scaled difference between the empirical and population gradients. Finally, we use $\theta$ to denote the angle between $\w_1$ and $\w_2$. We also define $\theta_1$ to be the angle between $\w_1$ and $\a$, $\theta_2$ to be the angle between $\w_2$ and $-\a$. We note that all lemmas stated in this proof overview our under the assumptions of the main theorem, we avoid repeating these assumptions repeatedly for readability.

\subsection{Controlling the imbalance term}\label{imbalance}
A crucial identity used throughout our proofs is that from moderately small initialization the norms of the inner and outer weights remain close to each other. Concretely, we define the imbalance term as $b_{1}^{(\tau)} := \twonorm{\w_1^{(\tau)}}^2 - \left(v_1^{(\tau)}\right)^2$ and $b_{2}^{(\tau)} := \twonorm{\w_2^{(\tau)}}^2 - \left(v_2^{(\tau)}\right)^2$. A constant bound for the absolute value of these terms is required to prove that the norms remain bounded throughout the training process (see Lemma \ref{lemma:norm_control}).

While the imbalance is invariant in gradient flow \citep{ji2019gradientdescentalignslayers}, the discretization in gradient descent introduces a small drift given by:
\begin{align*}
    b_{1}^{(\tau+1)} = & b_{1}^{(\tau)} + \mu^2\rbr{\twonorm{\nabla_{\w_1}\widehat{\mathcal{L}}}^2 - \left(\nabla_{v_1}\widehat{\mathcal{L}}\right)^2}.
\end{align*}
A simple constant bound on this drift is insufficient for our analysis, as the errors could accumulate to infinity over an infinite number of iterations. 

To address this, we prove a stronger result: the drift in each step is bounded by the distance between the effective weights $\text{diag}\rbr{\begin{bmatrix} v_1 \\ v_2 \end{bmatrix}}\W$ and the planted solution. Since $\text{diag}\rbr{\begin{bmatrix} v_1 \\ v_2 \end{bmatrix}}\W$ converges to the planted solution exponentially fast in Phase 3, the total accumulated drift remains finite even as $\tau \to \infty$. Concretely, we prove the following lemma. 
\begin{lemma}[Imbalance bound] \label{lemma:imbalance_bound}
    Assume that $v_1^{(\tau)}, v_2^{(\tau)} > 0$. For any $i\in \{1,2\}$, we have
    \begin{align*}
    \left|b_{i}^{(\tau+1)} - b_{i}^{(\tau)}\right| & \le c_6\mu^2\rbr{\rbr{v_i^{(\tau)}}^2 + \twonorm{\w_i^{(\tau)}}^2}\rbr{\twonorm{v_1^{(\tau)}\w_1^{(\tau)} -\a}^2 + \twonorm{v_2^{(\tau)}\w_2^{(\tau)} +\a}^2}.
    \end{align*}
    Here, we set the constant as $c_6 =6$.
\end{lemma}
This lemma is proven in Section \ref{apx:imbalance_bound}.
\subsection{Stability of Training Dynamics}\label{stability}
In this section, we establish two key stability properties that serve as the foundation for our proof. These results ensuring that the evolution is monotonic in the sense that once the iterates enter a new phase, they do not revert to a previous phase. The first lemma ensures that once the angle becomes small (at the end of the first phase) it continues to remain sufficiently small.
\begin{lemma}[Angle stays small] \label{lemma:angle_stay_small}
Assume that $\mu \le \frac{c_0}{\twonorm{\a}}$. For any iteration $\tau$ such that $0< v_1^{(\tau)}, v_2^{(\tau)} \le c_2\sqrt{\twonorm{\a}}$, $\theta_1^{(\tau)},\theta_2^{(\tau)} \le c_4$, $\left|b_1^{(\tau)}\right|, \left|b_2^{(\tau)}\right| \le \gamma\twonorm{\a}$ and $\twonorm{\Delta\mathcal{G}_1^{(\tau)}}, \twonorm{\Delta\mathcal{G}_2^{(\tau)}} \le c_5\twonorm{\a}$, we have 
\begin{align*}
    \theta_1^{(\tau+1)},\theta_2^{(\tau+1)} \le c_4.
\end{align*}
Here, we set the constants as $c_0 \le \frac{1}{2}, c_2 = 2, c_5 = \frac{1}{50}, c_4 = \frac{\pi}{20}, \gamma = \frac{1}{4}$.
\end{lemma}
This lemma is proven in Section \ref{apx:angle_stay_small}.
The second result ensures the norms remain bounded.
\begin{lemma}[Norms remain bounded] \label{lemma:norm_control}
Assume that $\mu \le \frac{c_0}{\twonorm{\a}}$. For any $0\le \beta \le \frac{1}{4}$ and iteration $\tau$ such that $\beta \sqrt{\twonorm{\a}}< v_1^{(\tau)}, v_2^{(\tau)} \le c_2\sqrt{\twonorm{\a}}$, $\theta_1^{(\tau)},\theta_2^{(\tau)} \le c_4$, $\left|b_1^{(\tau)}\right|, \left|b_2^{(\tau)}\right| \le \gamma\twonorm{\a}$ and $\twonorm{\Delta\mathcal{G}_1^{(\tau)}}, \twonorm{\Delta\mathcal{G}_2^{(\tau)}} \le c_5\twonorm{\a}$, we have 
\begin{align*}
    \beta \sqrt{\twonorm{\a}} < v_1^{(\tau+1)}, v_2^{(\tau+1)} \le c_2\sqrt{\twonorm{\a}}.
\end{align*}
Here, we set the constants as $c_0 \le \frac{4}{25}, c_2 = 2, c_5 = \frac{1}{50}, c_4 = \frac{\pi}{20}, \gamma = \frac{1}{4}$.
\end{lemma}
This lemma is proven in Section \ref{apx:norm_control}.

\subsection{Overview of \emph{Alignment} Phase}\label{phaseI}
The primary objective of Phase 1 is to demonstrate that $\w_1$ becomes approximately aligned with $\a$ (and $\w_2$  with $-\a$) within a constant number of steps, which is crucial for Phases 2 and 3. By symmetry, we focus on $\w_1$.

Our key observation is that the gradient update is dominated by the signal direction. Specifically, the update can be decomposed as:
\begin{align*}
    \w_1^{(\tau +1)} = \w_1^{(\tau)} + \mu \left(\frac{v_1}{2}\boldsymbol{a} + \vct{\zeta}^{(\tau)}\right).
\end{align*}
where the remainder term $\vct{\zeta}^{(\tau)}$ consists of terms involving the weights and the empirical noise. Since we use small initialization, these weight-dependent terms are much smaller than the signal term. This implies that the projection of $\w_1$ onto the signal direction $\a$ grows much faster than its projection onto the orthogonal subspace. Specifically, we have the following lemma proven in Section \ref{apx:angle_alignment}:
\begin{lemma}[Angle alignment] \label{lemma:angle_alignment} 
Assume that $\mu \le \frac{c_0}{\twonorm{\a}}$, $\sigma\le\sigma_0 \sqrt{\twonorm{\a}}$. After $T_1 = \lceil\frac{c_9}{\mu \twonorm{\a}}\rceil$ iterations, it holds that
\begin{align*}
    \theta_1^{(T_1)},\theta_2^{(T_1)} \le c_4
\end{align*}
with probability at least $1- Ce^{-cd}$.
Moreover, this alignment is achieved while maintaining that:
\begin{align*}
    c_3 \sigma \le v_1^{(T_1)},v_2^{(T_1)} \le & c_2 \sqrt{\twonorm{\a}}, \quad \left|b_1^{(T_1)}\right|, \left|b_2^{(T_1)}\right| \le c_{10} \twonorm{\a}.   
\end{align*}
Here, we set the constants as $c_0 \le 1, c_4= \frac{\pi}{20}, c_9 = \frac{64}{\tan c_4}, c_2 = 2, c_3 = \frac{1}{4}$. With $\alpha = \frac{65}{\tan c_4}$, we have $c_{10} = \frac{1}{4e^{2\alpha}}, \sigma_0 = \frac{1}{8e^{2\alpha}}$.
\end{lemma}

\subsection{Overview of \emph{Growth} Phase}\label{phaseII}
In Phase~3, we show that the effective weights $\text{diag}\rbr{\begin{bmatrix} v_1 \\ v_2 \end{bmatrix}}\W$ converge to the planted solution at an exponential rate. A key ingredient is a Polyak--Lojasiewicz (PL) inequality for the population loss (Lemma~\ref{lemma:pl_ineq}), which lower-bounds the squared gradient norm in terms of the suboptimality gap. Importantly, the PL constant depends on the magnitudes of $v_1$ and $v_2$.

At the end of Phase 1, $|v_1|$ and $|v_2|$ remain at their initialization scale, so the PL inequality only yields a weak contraction and therefore a slow convergence rate. The main goal of Phase~2 is to grow $v_1$ and $v_2$ to a sufficiently large scale, thereby strengthening the PL constant and enabling fast linear convergence in Phase~3.

\begin{lemma}[Norm growth] \label{lemma:norm_growth} 
Assume that $\mu \le \frac{c_0}{\twonorm{\a}\ln\rbr{\frac{\sqrt{\twonorm{\a}}}{\sigma}}}$. After $T_2 = \lceil\frac{c_8}{\mu \twonorm{\a}}\ln\rbr{\frac{\sqrt{\twonorm{\a}}}{\sigma}}\rceil$ iterations, we have
\begin{align*}
    v_1^{(T_1 + T_2)},v_2^{(T_1 + T_2)} \ge & c_7 \sqrt{\twonorm{\a}}, \quad \left|b_1^{(T_1+T_2)}\right|, \left|b_2^{(T_1+T_2)}\right| \le c_{11} \twonorm{\a}.
\end{align*}
Here, we set the constants as $c_0 \le \frac{1}{10^8}, c_7= \frac{1}{4}, c_8 = 32, c_{11} = \frac{1}{50}$.
\end{lemma}
This lemma is proven in Section \ref{apx:norm_growth}.

\subsection{Overview of the local \emph{Refinement} Phase}\label{phaseIII}
In the final \emph{local refinement} phase, we show that the effective weights
$\text{diag}\rbr{\begin{bmatrix}
        v_1^{(\tau)} \\ v_2^{(\tau)}
    \end{bmatrix}}\W^{\rbr{\tau}}$ converge to the planted solution $\W^*=[\a,-\a]^T$.
Rather than tracking parameters directly, we first prove that along empirical gradient descent
the \emph{population} loss decreases rapidly, and then convert this decay into the stated
parameter convergence.
The full argument (Proof of Theorem~\ref{thm:main} in Section \ref{pfmain}) is technical: it couples a
population-level gradient-descent analysis with trajectory-uniform concentration bounds (next section).
For clarity, we sketch only the population argument here.
This population reduction is essential because the empirical loss is not smooth and as discussed
below, even the population loss is not uniformly smooth.
The proof proceeds in two parts.
\\
\\
\noindent \textbf{Part 1 (PL inequality)} 
In this step we will show the following PL inequality proven in Section \ref{apx:proof_pl_ineq}.
\begin{lemma}[PL Inequality for the population loss] \label{lemma:pl_ineq}
 For $v_1, v_2 > 0 $,
 \begin{align*}
     \twonorm{\nabla_{\w_1} \calL\rbr{\v,\W}}^2 + \twonorm{\nabla_{\w_2} \calL\rbr{\v,\W}}^2 \geq \alpha \min\rbr{v_1^2, v_2^2} \calL\rbr{\v,\W}
 \end{align*} holds with $\alpha=0.05$ as long as $\theta > \frac{\pi}{2}$.
\end{lemma}
To prove this PL inequality we first show that it can be deduced by establishing the PL inequality when $v_1=v_2=1$ via a clever reduction argument. To prove the latter we define
\begin{align} \label{eq:correlation_h}
    h\rbr{\w_1,\w_2,\a} = \twonorm{\nabla_{\w_1} \calL\rbr{\W}}^2 + \twonorm{\nabla_{\w_2} \calL\rbr{\W}}^2 - \alpha \calL\rbr{\W}.
\end{align}
Note that since we set $v_1=v_2=1$ the loss is now only a function of $\w_1$ and $\w_2$. Also note that to prove the PL inequality it sufficies to show that $h$ is always positive. To do this, in our proof we show that $\Tilde{h}\rbr{\w_1,\w_2} = \min\limits_{\a}  h\rbr{\w_1,\w_2,\a}$ is always positive. The way we establish this is by showing that $\frac{1}{\|\w_2\|^2}\Tilde{h}\rbr{\w_1,\w_2}$ is only a function of $\theta$ (the angle between $\w_1$ and $\w_2$) and $\frac{\twonorm{\w_1}}{\twonorm{\w_2}}$. Since this is now only a function of two variables $\theta$ and $\frac{\twonorm{\w_1}}{\twonorm{\w_2}}$ it is easy to establish non-negativity as long as $\theta > \frac{\pi}{2}$.  The latter holds at the end of the growth phase and continues to remain large utilizing the stability of the dynamics established in Section \ref{stability} (concretely, Lemma \ref{lemma:angle_stay_small} shows the angles with planted directions remain small which implies the angle between the weight vectors remain large).

\noindent \textbf{Part 2 (Gradient smoothness)} -- While the population loss is not smooth in the entire domain, we show that in the region of the local refinement phase it is indeed smooth. Leaving the \emph{growth} phase, we have lower/upper bounds on $v_1, v_2, \twonorm{\w_1}, \twonorm{\w_2}$. Additionally, due to the stability analysis in Section \ref{stability} we continue to have lower/upper bounds on these quantities. In the next lemma we show that assuming such lower/upper bounds the population loss is indeed smooth. This lemma is proven in Section \ref{smoothnessbnd}.
\begin{lemma}[Smoothness of the population loss] \label{lemma:grad_smooth}  $\fnorm{\nabla^2 \mathcal{L}\rbr{\v, \W}} \leq L \twonorm{\a}$ holds for all $\v \in \R^{2}, \W\in\R^{2\times d}$ such that $c_1 \sqrt{\twonorm{\a}} \leq v_1, v_2, \twonorm{\w_1}, \twonorm{\w_2} \leq c_2 \sqrt{\twonorm{\a}}$ holds. Here $c_1, c_2, L$ are fixed constants.
\end{lemma} 
Showing geometric decrease of the \emph{population} loss under a PL inequality and smoothness is a
classical optimization result. Our setting is more delicate because we run \emph{empirical} gradient
descent rather than its population counterpart. See the proof of Theorem~\ref{thm:main} in
Section \ref{pfmain} for how we combine the trajectory-uniform concentration bounds (next section) with the PL and smoothness properties of the population loss stated above to obtain a geometric decrease in the population loss.

\subsection{Uniform Concentration}\label{unifconcen}
In this section, we provide an overview of the novel uniform concentration result that we have established which is key to our near optimal sample complexity. In particular, the concentration holds along the entire trajectory of GD and is used across the three phases. We provide the setup next. Let $\x_1,\dots,\x_n \overset{\mathrm{iid}}{\sim} \mathcal{N}(0,\mtx{I}_d)$ in $\R^d$. For $\w,\w^*\in\Sph^{d-1}$ define
\begin{equation}\label{eq:defM}
M(\w,\w^*) := \frac1n\sum_{i=1}^n \1_{\{\ip{\x_i}{\w}\ge0\}} \1_{\{\ip{\x_i}{\w^*}\ge0\}} \x_i\x_i^T.
\end{equation}
We prove a high-probability bound on
\begin{align*}
    \sup_{\w,\w^*\in\Sph^{d-1}} \opnorm{M(\w,\w^*)-\E \bbr{M(\w,\w^*)}}.
\end{align*}

\begin{lemma}[Uniform Concentration] \label{lem:unif_concentration} Fix $\delta\in(0,1/2)$. There exist universal constants $C,c>0$ such that the following holds. If
\begin{equation}\label{eq:sample_complexity}
n \ \ge\ C \, d \, \frac{\log^2(1/\delta)}{\delta^2},
\end{equation}
then with probability at least $1-3e^{-cd}$,
\begin{align*}
    \sup_{\w,\w^*\in\Sph^{d-1}} \opnorm{M(\w,\w^*)-\E \bbr{M(\w,\w^*)}} \ \le\ \delta.
\end{align*}

\end{lemma}
This lemma is proven in Section \ref{apx:unif_concentration}. Notably, this lemma allows us to establish separate high-probability bounds for the deviations of the gradient components with respect to $\w_1$ and $\w_2$. Concretely, it allows us to show the following lemma proven in Section \ref{apx:comp_deviation}.

\begin{lemma}[Component-wise Gradient Deviation Bounds] \label{lem:component_deviation}
Fix $\delta \in (0, 1/2)$. Define the error vectors $\mathbf{h}_1 = v_1 \w_1 - \a$ and $\mathbf{h}_2 = v_2 \w_2 + \a$. Under the sample complexity $n \ \ge\ C \, d \, \frac{\log^2(1/\delta)}{\delta^2}$, with probability at least $1 - 3e^{-cd}$, the following bounds hold simultaneously:
\begin{align}
     \twonorm{\nabla_{\w_1} \widehat{\mathcal{L}} - \nabla_{\w_1} \mathcal{L}} &\le v_1 \delta (\twonorm{\mathbf{h}_1} + \twonorm{\mathbf{h}_2}), \\
     \twonorm{\nabla_{\w_2} \widehat{\mathcal{L}} - \nabla_{\w_2} \mathcal{L}} &\le v_2 \delta (\twonorm{\mathbf{h}_1} + \twonorm{\mathbf{h}_2}).
\end{align}
\end{lemma}
This lemma highlights a particularly favorable self-regularizing property of the dynamics: the
empirical--population gradient deviation scales \emph{linearly} with the current error
$\twonorm{\h_1}+\twonorm{\h_2}$. As the iterates approach the global optimum and the errors shrink, the
concentration bounds automatically tighten, yielding increasingly accurate gradient estimates
along the trajectory. In other words, concentration improves precisely when it is most needed in
the local refinement regime, enabling stable geometric convergence in this region.

%% file: sec/acknowledgments.tex
\section*{Acknowledgements}
This work was partially supported by AWS credits through an Amazon Faculty Research Award, a NAIRR Pilot Award, and generous funding by Coefficient Giving. M. Soltanolkotabi is supported by the Packard Fellowship in Science and Engineering, a Sloan Research Fellowship in Mathematics, NSF CAREER Award \#1846369, DARPA FastNICS program, NSF CIF Awards \#1813877 and \#2008443, and NIH Award DP2LM014564-01.

%% file: sec/appendix.tex
\section{Useful Calculations}
\input{sec/apx/useful_calculations}

\section{Proof of Key Lemmas in the Population Setting}
For the simplicity of notation, let $\w_1 = \w_1^{\rbr{\tau}}, \w_2 = \w_2^{\rbr{\tau}}$.

\subsection{Proof of Gradient Smoothness Towards the Global Optima in the Population Case (Lemma~\ref{lemma:popgrad_smooth})} \label{apx:proof_grad_smooth}
\input{sec/apx/old/proof_pop_directional_smoothness}

\subsection{Proof of the PL Inequality in the Population Case (Lemma \ref{lemma:pl_ineq})} \label{apx:proof_pl_ineq}
\input{sec/apx/proof_pop_pl_ineq}

\subsection{Bound on the Smoothness of the Population Loss (Lemma \ref{lemma:grad_smooth})}\label{smoothnessbnd}
\input{sec/apx/proof_pop_smoothness}

\subsection{Population Loss Lower Bound (Lemma \ref{lem:loss_lower_bound})} \label{apx:pop_lower_bound_lem}
\input{sec/apx/proof_pop_lower_bound_lemma}

\section{Proof of Key Lemmas in the Empirical Setting}

\subsection{Bound for Imbalance Term (Lemma \ref{lemma:imbalance_bound})} \label{apx:imbalance_bound}
\input{sec/apx/imbalance_bound}

\subsection{Proof of the Stability of Angles (Lemma \ref{lemma:angle_stay_small})} \label{apx:angle_stay_small}
\input{sec/apx/angle_stay_small}

\subsection{Proof of the Stability of Norms (Lemma \ref{lemma:norm_control})} \label{apx:norm_control}
\input{sec/apx/norm_control}

\subsection{Proof of Phase 1 (Lemma \ref{lemma:angle_alignment})} \label{apx:angle_alignment}
\input{sec/apx/angle_alignment}

\subsection{Proof of Phase 2 (Lemma \ref{lemma:norm_growth})} \label{apx:norm_growth}
\input{sec/apx/norm_growth}

\subsection{Uniform Concentration (Lemma \ref{lem:unif_concentration})} \label{apx:unif_concentration}
\input{sec/apx/uniform_concentration}

\subsection{Concentration of Gradient Component Deviations (Lemma \ref{lem:component_deviation})} \label{apx:comp_deviation}
\input{sec/apx/grad_concentration}

\section{Proof of Main Theorems}
\subsection{Proof of Theorem \ref{thm:landscape} for Landscape Characterization} \label{seq:landscape_calc}
\input{sec/apx/proof_landscape}

\subsection{Proof of Theorem \ref{thm:main} for Convergence of the GD Trajectory}\label{pfmain}
To prove this theorem first we note that after $T_1 = \lceil\frac{c_9}{\mu \twonorm{\a}}\rceil$ iterations of GD (i.e. \textit{alignment} phase), using Lemma \ref{lemma:angle_alignment} from Section \ref{phaseI} we have with high probability
\begin{align*}
 \theta_1^{(T_1)},\theta_2^{(T_1)} \le c_4, \quad \text{and} \quad c_1 \sigma \le v_1^{(T_1)},v_2^{(T_1)} \le c_2 \sqrt{\twonorm{\a}}.
\end{align*}
Using Lemma \ref{lemma:norm_growth}, after $T = T_1 + T_2$ iterations, we have
\begin{align*}
    \theta_1^{(T)},\theta_2^{(T)} \le c_4, \quad v_1^{(T)},v_2^{(T)} \ge c_7 \sqrt{\twonorm{\a}}, \quad \text{and} \quad \left|b_1^{(T)} \right|, \left|b_2^{(T)} \right| \leq c_{11} \twonorm{\a}.
\end{align*}
Using the definition of the imbalance term, $\twonorm{\w_i^{(T)}}^2 = (v_i^{(T)})^2 + b_i^{(T)}$, we evaluate the weights at the end of the growth phase ($T = T_1 + T_2$). From Lemma \ref{lemma:norm_growth}, we have $v_i^{(T)} \ge c_7 \sqrt{\twonorm{\a}}$ and $|b_i^{(T)}| \le c_{11} \twonorm{\a}$. Noting that $c_7^2 > c_{11}$ for $c_7= \frac{1}{4}, c_{11} = \frac{1}{50}$; this implies:
\begin{align} \label{ineq:w_bounds_at_T}
    (c_7^2 - c_{11}) \twonorm{\a} \le \twonorm{\w_1^{(T)}}^2, \twonorm{\w_2^{(T)}}^2 \le (c_2^2 + c_{11}) \twonorm{\a}.
\end{align}
We now establish that these bounds hold uniformly for all $\tau \geq T$. Lemma \ref{lemma:angle_stay_small} ensures that since the angles $\theta_1^{(T)}, \theta_2^{(T)}$ are small, they remain bounded by $c_4$ for all subsequent iterations. Lemma \ref{lemma:norm_control} ensures that if the norms $v_i$ start in the interval $[c_7\sqrt{\twonorm{\a}}, c_2\sqrt{\twonorm{\a}}]$, they remain within a fixed range $[c_7 \sqrt{\twonorm{\a}}, c_2\sqrt{\twonorm{\a}}]$ for all $\tau > T$. Finally, as established in the convergence analysis below, the error $\twonorm{\vct{h}_1}^2 + \twonorm{\vct{h}_2}^2$ decays geometrically. By Lemma \ref{lemma:imbalance_bound}, the total drift in the imbalance terms is summable, keeping $|b_i^{(\tau)}|$ uniformly bounded by a constant $\gamma \twonorm{\a}$ for all $\tau \ge T$.

Consequently, there exist universal constants $c_{\min}$ and $c_{\max}$ such that for all $\tau \ge T$:
\begin{align} \label{ineq:norm_remain_bounded}
    c_{\min} \sqrt{\twonorm{\a}} \le v_1^{(\tau)}, v_2^{(\tau)}, \twonorm{\w_1^{(\tau)}},\twonorm{\w_2^{(\tau)}} \le c_{\max} \sqrt{\twonorm{\a}}.
\end{align}
From \eqref{ineq:norm_remain_bounded}, the conditions for the PL inequality (Lemma \ref{lemma:pl_ineq}) and smoothness (Lemma \ref{lemma:grad_smooth}) hold uniformly for all $\tau \geq T$, where the constants $\alpha$ and $L$ now depend on $c_{\min}$ and $c_{\max}$. Specifically, we have:
\begin{align*}
    \twonorm{\nabla_{\w_1} \calL\rbr{\v^{(\tau)},\W^{(\tau)}}}^2 + \twonorm{\nabla_{\w_2} \calL\rbr{\v^{(\tau)},\W^{(\tau)}}}^2 \geq \alpha c_{\min}^2 \twonorm{\a} \calL\rbr{\v^{(\tau)},\W^{(\tau)}},
 \end{align*} 
and
\begin{align*}
    \fnorm{\nabla^2 \mathcal{L}\rbr{\v^{(\tau)}, \W^{(\tau)}}} \leq L \twonorm{\a}.
\end{align*}
Next, we keep track of the \textit{population} loss while performing GD updates on the \textit{empirical} loss. Let $\vct{\theta}$ denote $\begin{bmatrix}
    \v \\ \text{vect}\rbr{\W}
\end{bmatrix}$, and also define errors vectors $\vct{h}_1 = v_1 \w_1 - \a$, $\vct{h}_2 = v_2 \w_2 + \a$. We note that $\fnorm{\text{diag}\rbr{\begin{bmatrix}
        v_1 \\ v_2
    \end{bmatrix}}\W - \W^*}^2 = \twonorm{\vct{h}_1}^2 + \twonorm{\vct{h}_2}^2$. For all $\tau \geq T$ we have
{\allowdisplaybreaks
\begin{align*}
    \calL\rbr{\vct{\theta}^{(\tau+1)}} = & \calL\rbr{\vct{\theta}^{(\tau)} - \mu \nabla \widehat{\calL}\rbr{\vct{\theta}^{(\tau)}}} \\
    \stackrel{(a)}{\le} & \calL\rbr{\vct{\theta}^{(\tau)}} - \mu \inner{\nabla \calL \rbr{\vct{\theta}^{(\tau)}}}{\nabla \widehat{\calL}\rbr{\vct{\theta}^{(\tau)}}} + \frac{L \twonorm{\a}}{2}\mu^2 \twonorm{\nabla \widehat{\calL}\rbr{\vct{\theta}^{(\tau)}}}^2 \\
    =&  \calL\rbr{\vct{\theta}^{(\tau)}} - \mu \twonorm{\nabla \calL\rbr{\vct{\theta}^{(\tau)}}}^2 -  \mu \inner{\nabla \calL \rbr{\vct{\theta}^{(\tau)}}}{\nabla \widehat{\calL}\rbr{\vct{\theta}^{(\tau)}} - \nabla \calL \rbr{\vct{\theta}^{(\tau)}}} \\
    & \quad + \frac{L \twonorm{\a}}{2}\mu^2 \twonorm{\nabla \widehat{\calL}\rbr{\vct{\theta}^{(\tau)}}}^2 \\
    \stackrel{(b)}{\le} & \calL\rbr{\vct{\theta}^{(\tau)}} - \mu \twonorm{\nabla \calL\rbr{\vct{\theta}^{(\tau)}}}^2 + \eta\mu \twonorm{\nabla \calL \rbr{\vct{\theta}^{(\tau)}}}^2 + \frac{\mu}{\eta} \twonorm{\nabla \widehat{\calL}\rbr{\vct{\theta}^{(\tau)}} - \nabla \calL \rbr{\vct{\theta}^{(\tau)}}}^2 \\
    & \quad + \frac{L \twonorm{\a}}{2}\mu^2 \twonorm{\nabla \widehat{\calL}\rbr{\vct{\theta}^{(\tau)}}}^2 \\
    \stackrel{(c)}{\le} & \calL\rbr{\vct{\theta}^{(\tau)}} - \mu \rbr{1 - \eta} \twonorm{\nabla \calL\rbr{\vct{\theta}^{(\tau)}}}^2 + \frac{\mu}{\eta} \twonorm{\nabla \widehat{\calL}\rbr{\vct{\theta}^{(\tau)}} - \nabla \calL \rbr{\vct{\theta}^{(\tau)}}}^2 \\
    & \quad + L \twonorm{\a}\mu^2 \rbr{\twonorm{\nabla \calL\rbr{\vct{\theta}^{(\tau)}}}^2 + \twonorm{\nabla \widehat{\calL}\rbr{\vct{\theta}^{(\tau)}} - \nabla \calL\rbr{\vct{\theta}^{(\tau)}}}^2} \\
    =& \calL\rbr{\vct{\theta}^{(\tau)}} - \mu \rbr{1 - \eta - \mu L \twonorm{\a}} \twonorm{\nabla \calL\rbr{\vct{\theta}^{(\tau)}}}^2 \\
    &\quad\quad+ \rbr{\frac{\mu}{\eta} + \mu^2 L \twonorm{\a}}\twonorm{\nabla \widehat{\calL}\rbr{\vct{\theta}^{(\tau)}} - \nabla \calL \rbr{\vct{\theta}^{(\tau)}}}^2 \\
    \stackrel{(d)}{\le}& \calL\rbr{\vct{\theta}^{(\tau)}} - \alpha c_1^2 \twonorm{\a} \mu \rbr{1 - \eta - \mu L \twonorm{\a}} \calL\rbr{\vct{\theta}^{(\tau)}} \\
    &\quad\quad+ \rbr{\frac{\mu}{\eta} + \mu^2 L \twonorm{\a}}\twonorm{\nabla \widehat{\calL}\rbr{\vct{\theta}^{(\tau)}} - \nabla \calL \rbr{\vct{\theta}^{(\tau)}}}^2 \\
    \stackrel{(e)}{\le}& \rbr{1 - \alpha c_{\min}^2 \twonorm{\a} \mu \rbr{1 - \eta - \mu L \twonorm{\a}}}\calL\rbr{\vct{\theta}^{(\tau)}} \\ 
    & \quad + \rbr{\frac{\mu}{\eta} + \mu^2 L \twonorm{\a}} \rbr{\rbr{v_1^{(\tau)}}^2 + \rbr{v_2^{(\tau)}}^2 + \twonorm{\w_1^{(\tau)}}^2 + \twonorm{\w_2^{(\tau)}}^2}\delta^2 \rbr{\twonorm{\vct{h}_1^{(\tau)}}^2 + \twonorm{\vct{h}_2^{(\tau)}}^2} \\
    \stackrel{(f)}{\le}& \rbr{1 - \alpha c_{\min}^2 \twonorm{\a} \mu \rbr{1 - \eta - \mu L \twonorm{\a}}}\calL\rbr{\vct{\theta}^{(\tau)}} \\ 
    & \quad + \rbr{\frac{\mu}{\eta} + \mu^2 L \twonorm{\a} } 4 c_{\max}^2 \twonorm{\a} \delta^2 \rbr{\twonorm{\vct{h}_1^{(\tau)}}^2 + \twonorm{\vct{h}_2^{(\tau)}}^2} \\
     \stackrel{(g)}{\le}& \rbr{1 - \alpha c_{\min}^2 \twonorm{\a} \mu \rbr{1 - \eta - \mu L \twonorm{\a}} + 20 \rbr{\frac{\mu}{\eta} + \mu^2 L \twonorm{\a} } 4 c_{\max}^2 \twonorm{\a} \delta^2} \calL\rbr{\vct{\theta}^{(\tau)}}
\end{align*}
}
where (a) follows from the quadratic upper bound and smoothness of $\calL$, (b) follows from $\inner{\sqrt{\eta} \vct{a}}{\frac{1}{\sqrt{\eta}}\vct{b}} \leq \eta \twonorm{\a}^2 + \frac{1}{\eta} \twonorm{\b}^2$ for any $\eta > 0$, (c) follows from triangle inequality, (d) follows from PL inequality, (e) follows from Lemma \ref{lem:unif_concentration} and gradient identity \eqref{eq:grad_useful_identity}, (f) follows from upper bounds on the norms, and finally (g) follows from applying the population loss lower bound (Lemma \ref{lem:loss_lower_bound}).
Set $\mu = \frac{\bar{\mu}}{\twonorm{\a}}$, $\eta = \frac{1}{4}$. We have
\begin{align*}
    \calL\rbr{\vct{\theta}^{(\tau+1)}} \le& \rbr{1 - \alpha c_{\min}^2 \bar{\mu} \rbr{\frac{3}{4} - \bar{\mu}L} + \rbr{\frac{4\bar{\mu}}{3} + \bar{\mu}^2 L} 80 c_{\max}^2 \delta^2} \calL\rbr{\vct{\theta}^{(\tau)}} \\
    =& \rbr{1 - \rbr{\frac{3\alpha c_{\min}^2}{4} - \frac{320c_{\max}^2 \delta^2}{3}} \bar{\mu} + \rbr{\alpha c_{\min}^2 + 80 c_{\max}^2 \delta^2}L \bar{\mu}^2} \calL\rbr{\vct{\theta}^{(\tau)}}.
\end{align*}
We now choose $\delta$ and $\bar{\mu}$ so that the quadratic factor above yields a strict contraction. First require that the linear coefficient is positive, i.e.
\begin{align*}
    \delta^2 \leq \frac{9\alpha c_{\min}^2}{1280c_{\max}^2}.
\end{align*}
Fix any such $\delta$ (this is ensured by Lemma~\ref{lem:unif_concentration} by taking $n$ sufficiently large). Next, define,
\begin{align*}
    a := \frac{3\alpha c_{\min}^2}{4} - \frac{320c_{\max}^2 \delta^2}{3} \quad \text{and} \quad b:= \rbr{\alpha c_{\min}^2 + 80c_{\max}^2 \delta^2}L.
\end{align*}
If we further choose
\begin{align*}
    \bar{\mu} \le \frac{a}{2b},
\end{align*}
then the quadratic term is dominated by the linear term, and we have
\begin{align*}
    1 - a \bar{\mu} + b \bar{\mu}^2 \le 1 - \frac{a}{2}\bar{\mu}.
\end{align*}
Consequently, for all $\tau \ge T$,
\begin{align*}
    \calL\rbr{\vct{\theta}^{(\tau+1)}} \leq \rbr{1-c\bar{\mu}} \calL\rbr{\vct{\theta}^{(\tau)}},
\end{align*}
where $c := \frac{a}{2} > 0$ is a numerical constant depending only on $\alpha, c_{\min}, c_{\max}$. Iterating this inequality for all $\tau > T$ yields geometric decrease of the population loss
\begin{align*}
    \calL\rbr{\vct{\theta}^{(\tau)}} \leq \rbr{1-c\bar{\mu}}^{\rbr{\tau - T}} \calL\rbr{\vct{\theta}^{(T)}}.
\end{align*}
Finally, we apply the population loss lower bound (Lemma~\ref{lem:loss_lower_bound}) to lower bound the left-hand side:
\begin{align*}
\calL\rbr{\vct{\theta}^{(\tau)}} \ge \frac{1}{\tilde{c}} \fnorm{\text{diag}\rbr{\begin{bmatrix}
        v_1^{(\tau)} \\ v_2^{(\tau)}
    \end{bmatrix}}\W^{\rbr{\tau}} - \W}^2.
\end{align*}
By further upper bounding the right-hand side using the fact that the ReLU activation is 1-Lipschitz, we have $\calL\rbr{\vct{\theta}^{(T)}} \le \fnorm{\text{diag}\rbr{\begin{bmatrix}
        v_1^{(T)} \\ v_2^{(T)}
    \end{bmatrix}}\W^{\rbr{T}} - \W^*}^2$. Combining these yields:
\begin{align*}
\fnorm{\text{diag}\rbr{\begin{bmatrix}
        v_1^{(\tau)} \\ v_2^{(\tau)}
    \end{bmatrix}}\W^{\rbr{\tau}} - \W}^2 \leq \tilde{c}\rbr{1-c\bar{\mu}}^{\rbr{\tau - T}} \fnorm{\text{diag}\rbr{\begin{bmatrix}
        v_1^{(T)} \\ v_2^{(T)}
    \end{bmatrix}}\W^{\rbr{T}} - \W}^2
\end{align*}
for some numerical constant $\tilde{c} > 0$. This shows geometric convergence of the GD iterates, completing the proof of Theorem~\ref{thm:main}.

\section{Additional Experimental Results}
\subsection{Pairing-up Behavior for $r \geq 3$}
\label{apx:multidimension}
In this section, we present additional results on the pairing behavior of $\w_i$ and $\v_i$ for different values of $r$.  Although our theoretical analysis is limited to the scalar output setting, for our experiments we also consider multi-dimensional outputs. We only consider the case where the model is exactly parameterized i.e. $k=2r$. We first show that an interesting pattern arises if both the inner and outer layers of the neural network are initialized sufficiently small.
\begin{figure}[h]
    \centering
    \begin{minipage}[t]{0.48\textwidth}
        \centering
        \includegraphics[width=0.7\linewidth]{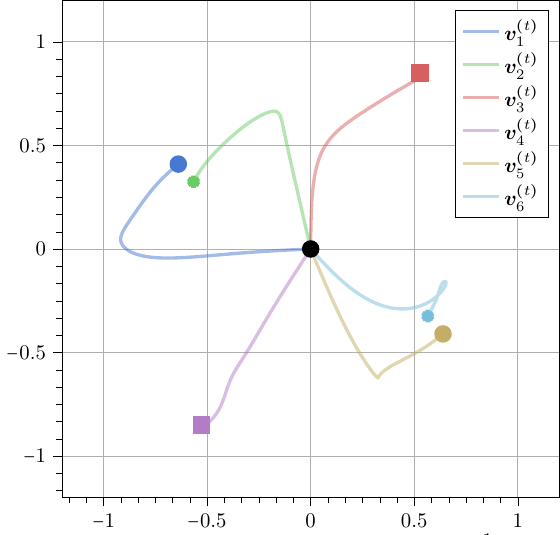}\\
        \textbf{(a)} Trajectory of $\vct{v}_i$'s and their pairing behavior.
    \end{minipage}
    \hfill
    \begin{minipage}[t]{0.48\textwidth}
        \centering
        \includegraphics[width=0.7\linewidth]{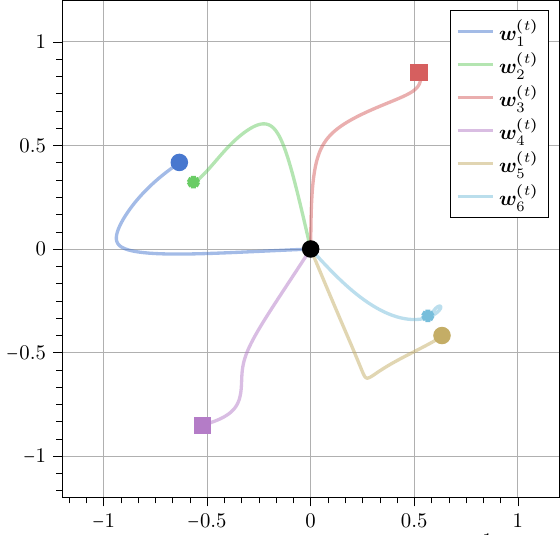}\\
        \textbf{(b)} Trajectory of $\vct{w}_i$'s and their pairing behavior.
    \end{minipage}
    \centering
    \caption{\textbf{Pairing pattern in multi-dimensional setting.} We train the network from small initialization when exactly parameterized ($k=6$ and $r=3$). On left (a), we depict the trajectories of individual weights in the outer layer ($\vct{v}_i$'s) across iterations. We observe that the weights at convergence can be grouped into three pairs such that one of the weights is approximately negative of the other. For instance, we observe that $\vct{v}_3^{\rbr{\infty}} \approx -\vct{v}_4^{\rbr{\infty}}$. Which neurons end up pairing with each other is indicated by the usage of same symbol (square, circle, etc.). A similar pairing is observed for the inner layer weights as well (b). While these vectors all lie in a higher dimensional space, we pick an arbitrary two dimensional axis to plot them in 2D.}
    \label{fig:pairing_pattern_V_W}
\end{figure}

For visualization purposes in Figure \ref{fig:pairing_pattern_V_W}, we pick $r=3$ and $k=6$. As for the target function, we pick $\vct{a}_1$, $\vct{a}_2$, $\vct{a}_3$ to be $\vct{e}_1$, $\vct{e}_2$, $\vct{e}_3$ respectively which correspond to the standard basis vectors in $\R^d$. We plot the trajectory of both the inner and outer layer weights of the network across iterations and observe a peculiar pattern in both $\vct{v}_i$'s and $\vct{w}_i$'s. At convergence, weights can be grouped into pairs such that one of the weights is approximately negative of the other. As a concrete example, in Figure \ref{fig:pairing_pattern_V_W}, we observe that $\vct{v}_3^{\rbr{\infty}} \approx -\vct{v}_4^{\rbr{\infty}}$, $\vct{v}_1^{\rbr{\infty}} \approx -\vct{v}_5^{\rbr{\infty}}$, and $\vct{v}_2^{\rbr{\infty}} \approx -\vct{v}_6^{\rbr{\infty}}$ which also holds similarly for $\vct{w}_i$'s as well. This suggests that after a permutation of the hidden units, we get
\begin{align*}
\mtx{V}^{\rbr{\infty}} \approx \begin{bmatrix}\mtx{I}_r, -\mtx{I}_r\end{bmatrix}^T\widetilde{\mtx{V}}, \quad \mtx{W}^{\rbr{\infty}} \approx \begin{bmatrix}\mtx{I}_r, -\mtx{I}_r\end{bmatrix}^T\widetilde{\mtx{W}}.
\end{align*}
which can be considered as a natural extension to the $\vct{v}_i=\pm 1$ pattern in the single output setting.

Beyond the $r=3$ case, we illustrate the same behavior for $r=5$ in Figure~\ref{fig:apx_pairing_pattern_V_W_dim_5} and for $r=10$ in Figure~\ref{fig:apx_pairing_pattern_V_W_dim_10}. While we also observe the pairing for $r>10$, we omit those results here for visual clarity. In general, we note that the weights at convergence (indicated with \textit{star} symbol in Figures ~\ref{fig:apx_pairing_pattern_V_W_dim_5} and ~\ref{fig:apx_pairing_pattern_V_W_dim_10}) can be grouped into $r$ pairs such that one of the weights is approximately negative of the other. To aid with detecting the pairs visually, we draw the line determined by each pair with dashed lines.

\begin{figure}[h]
    \centering
    \begin{minipage}[t]{0.48\textwidth}
        \centering
        \includegraphics[width=1.0\linewidth]{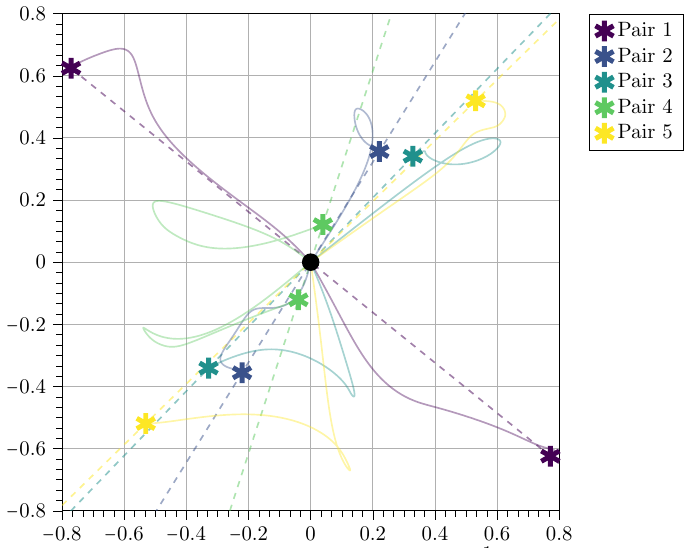}\\
        \textbf{(a)} Trajectory of $\vct{v}_i$'s and their pairing behavior.
    \end{minipage}
    \hfill
    \begin{minipage}[t]{0.48\textwidth}
        \centering
        \includegraphics[width=1.0\linewidth]{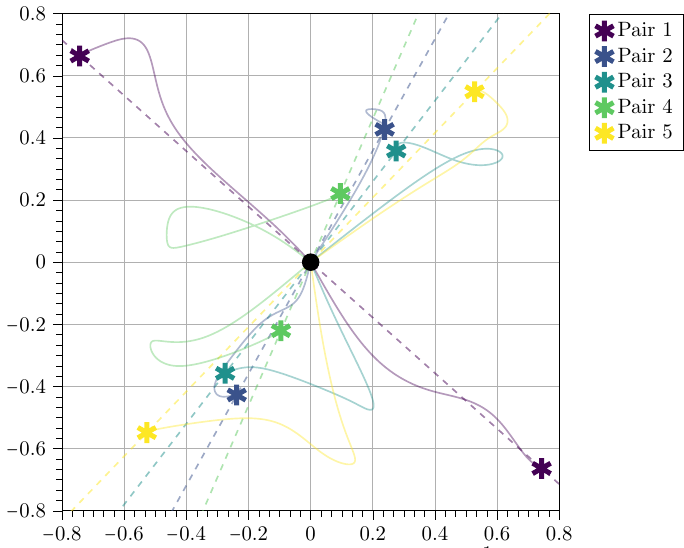}\\
        \textbf{(b)} Trajectory of $\vct{w}_i$'s and their pairing behavior.
    \end{minipage}
    \centering
    \caption{\textbf{Pairing pattern for $r=5$.} We train the network from small initialization when exactly parameterized ($k=10$ and $r=5$). On left (a), we depict the trajectories of individual weights in the outer layer ($\vct{v}_i$'s) across iterations. Each pair is indicated by the same color and the dashed line. A similar pairing is observed for the inner layer weights as well (b). While these vectors all lie in a higher dimensional space, we pick an arbitrary two dimensional axis to plot them in 2D.}
    \label{fig:apx_pairing_pattern_V_W_dim_5}
\end{figure}

\begin{figure}[h]
    \centering
    \begin{minipage}[t]{0.48\textwidth}
        \centering
        \includegraphics[width=1.0\linewidth]{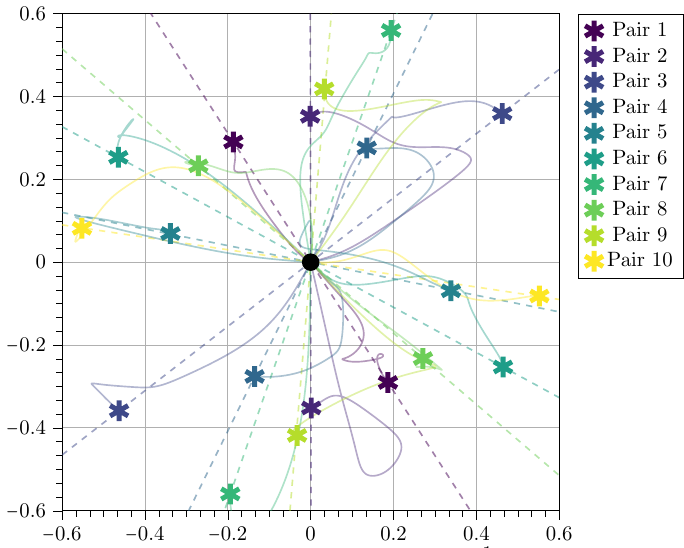}\\
        \textbf{(a)} Trajectory of $\vct{v}_i$'s and their pairing behavior.
    \end{minipage}
    \hfill
    \begin{minipage}[t]{0.48\textwidth}
        \centering
        \includegraphics[width=1.0\linewidth]{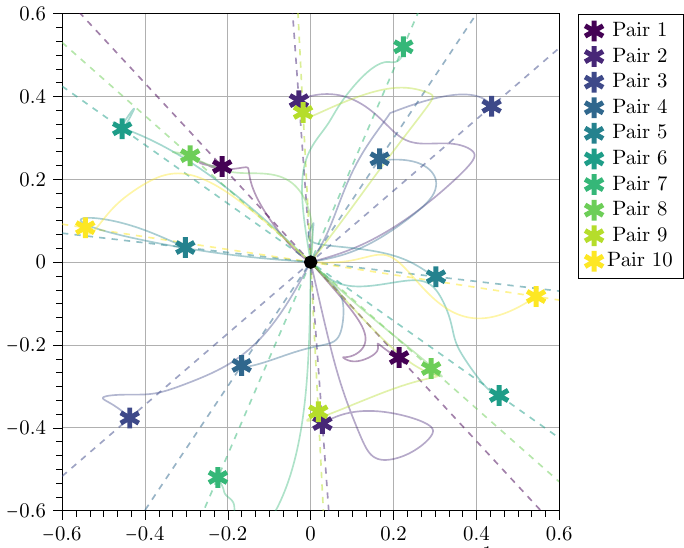}\\
        \textbf{(b)} Trajectory of $\vct{w}_i$'s and their pairing behavior.
    \end{minipage}
    \centering
    \caption{\textbf{Pairing pattern for $r=10$.} We train the network from small initialization when exactly parameterized ($k=20$ and $r=10$). On left (a), we depict the trajectories of individual weights in the outer layer ($\vct{v}_i$'s) across iterations. Each pair is indicated by the same color and the dashed line. A similar pairing is observed for the inner layer weights as well (b). While these vectors all lie in a higher dimensional space, we pick an arbitrary two dimensional axis to plot them in 2D.}
    \label{fig:apx_pairing_pattern_V_W_dim_10}
\end{figure}

%% file: sec/apx/useful_calculations.tex
In this section we provide the derivation of several useful identities.

\subsection{Population Loss} \label{apx:pop_loss_calc}
Let $\vct{a},\vct{b} \in \R^d$ be two arbitrary vectors. Define 
\begin{align}
    f \rbr{\vct{a},\vct{b}} &= \E_{\vct{x}} \bbr{ \bbr{\vct{a^Tx}}_+ \bbr{\vct{b^Tx}}_+} \nonumber \\
    &\stackrel{(a)}{=} \frac{1}{2\pi} \twonorm{\vct{a}} \twonorm{\vct{b}} \rbr{\sin \rbr{\theta_{\vct{a},\vct{b}}}+\rbr{\pi - \theta_{\vct{a},\vct{b}} } \cos \rbr{\theta_{\vct{a},\vct{b}}}} \label{eq:relu_dual}
\end{align}
where $\theta_{\vct{a},\vct{b}} = \cos^{-1} \rbr{ \frac{\vct{a^Tb}}{\twonorm{\vct{a}} \twonorm{\vct{b}}}}$, expectation is over $\vct{x} \sim \mathcal{N}\rbr{\vct{0}, \mtx{I}_d}$ and inequality (a) follows from the Table 1 in \citep{daniely2016toward}.
\\
\\
Using these we calculate the closed form for the population loss (\ref{eq:pop_loss}) as:
\begin{align}
    \calL\rbr{\mtx{\theta}} &= \frac{1}{2}\E_{\vct{x}}\bbr{\twonorm{\sum_{i=1}^k v_i \phi \rbr{\vct{w_i^T x}} - \a^T\x}^2} \nonumber \\
    &= \frac{1}{2} \sum_{i=1}^k \sum_{j=1}^k v_i v_j \E_{\vct{x}}\bbr{\phi \rbr{\vct{w_i^T x}} \phi \rbr{\vct{w_j^T x}}} - \sum_{i=1}^k v_i \a^T \E_{\vct{x}}\bbr{\phi \rbr{\vct{w_i^T x}} \vct{x}} + \frac{1}{2} \E_{\vct{x}}\bbr{\a^T\x \x^T \a} \nonumber \\
    &= \frac{1}{2} \sum_{i=1}^k \sum_{j=1}^k v_i v_j \E_{\vct{x}}\bbr{\phi \rbr{\vct{w_i^T x}} \phi \rbr{\vct{w_j^T x}}} - \sum_{i=1}^k v_i \a^T \E_{\vct{x}}\bbr{\phi \rbr{\vct{w_i^T x}} \vct{x}} + \frac{\twonorm{\a}^2}{2} \nonumber \\
    &\stackrel{(a)}{=} \frac{1}{2} \sum_{i=1}^k \sum_{j=1}^k v_i v_j f \rbr{\vct{w}_i,\vct{w}_j} - \sum_{i=1}^k v_i \a^T \E_{\vct{x}}\bbr{\phi \rbr{\vct{w_i^T x}} \vct{x}} + \frac{\twonorm{\a}^2}{2} \nonumber \\
    &\stackrel{(b)}{=} \frac{1}{2} \sum_{i=1}^k \sum_{j=1}^k v_i v_j f \rbr{\vct{w}_i,\vct{w}_j} - \sum_{i=1}^k v_i \a^T \E_{\vct{x}}\bbr{\nabla_{\vct{x}}\phi \rbr{\vct{w_i^T x}}} + \frac{\twonorm{\a}^2}{2} \nonumber \\
    &\stackrel{}{=} \frac{1}{2} \sum_{i=1}^k \sum_{j=1}^k v_i v_j f \rbr{\vct{w}_i,\vct{w}_j} - \sum_{i=1}^k v_i \a^T \vct{w}_i \E_{\vct{x}}\bbr{\phi' \rbr{\vct{w_i^T x}}} + \frac{\twonorm{\a}^2}{2} \nonumber \\   
    &\stackrel{(c)}{=} \frac{1}{2} \sum_{i=1}^k \sum_{j=1}^k v_i v_j f \rbr{\vct{w}_i,\vct{w}_j} - \frac{1}{2} \sum_{i=1}^k v_i \a^T \vct{w}_i + \frac{\twonorm{\a}^2}{2} \nonumber \\  
    &= \frac{1}{4\pi} \sum_{i=1}^k \sum_{j=1}^k v_i v_j \twonorm{\vct{w}_i} \twonorm{\vct{w}_j} \rbr{\sin\theta_{ij}+\rbr{\pi - \theta_{ij}} \cos \theta_{ij}} - \frac{1}{2} \sum_{i=1}^k v_i \a^T \vct{w}_i + \frac{\twonorm{\a}^2}{2}  \label{eq:pop_loss_closed_form_summation}
\end{align}
where equation (a) follows from the definition of $f\rbr{\vct{a},\vct{b}}$, (b) follows from the Stein's Lemma, and finally (c) follows from the fact that derivative of ReLU activation is the step function and $\vct{w}_i^T \vct{x} > 0$ with probability $\frac{1}{2}$.
\\
\\
We also write this in a more compact matrix form as follows:
\begin{equation*}
    \calL\rbr{\mtx{\theta}} = \frac{1}{4 \pi} \vct{u}^T \rbr{\sin\rbr{\mtx{\Theta}} + \rbr{\pi \mathbbm{1} \mathbbm{1}^T - \mtx{\Theta}} \odot \cos\rbr{\mtx{\Theta}}} \vct{u} - \frac{1}{2} \mtx{a^TW^Tv} + \frac{1}{2} \twonorm{\vct{a}}^2
\end{equation*}
where $\vct{\omega}_i = \twonorm{\vct{w}_i}$, $\vct{u} = diag\rbr{\vct{\omega}} \vct{v}$, and $\mtx{\theta}_{ij}$ is the angle between $\vct{w}_i$ and $\vct{w}_j$.

\subsection{Population Gradient}
\paragraph{Gradient w.r.t. $\W$:} Let us define,
\begin{align}
    g \rbr{\vct{a},\vct{b}} &= \frac{\partial}{\partial \vct{a}} f \rbr{\vct{a},\vct{b}} \nonumber \\
    &= \frac{1}{2 \pi} \rbr{\twonorm{\vct{b}}\sin \rbr{\theta_{\vct{a},\vct{b}}}\vct{\Bar{a}} + \rbr{\pi - \theta_{\vct{a},\vct{b}} }\vct{b} } \quad \rbr{\vct{\Bar{a}} = \frac{\vct{a}}{\twonorm{\vct{a}}}, \quad \vct{\Bar{b}} = \frac{\vct{b}}{\twonorm{\vct{b}}} } \nonumber \\ 
    &= \frac{\twonorm{\vct{b}}}{2 \pi} \rbr{\sin \rbr{\theta_{\vct{a},\vct{b}}}\vct{\Bar{a}} + \rbr{\pi - \theta_{\vct{a},\vct{b}} }\vct{\Bar{b}} }.
\end{align}
Taking the derivative of (\ref{eq:pop_loss_closed_form_summation}) with respect to $\vct{w}_i$, we get
\begin{align*}
    \nabla_{\vct{w_i}}\calL\rbr{\mtx{\theta}} &= \frac{1}{2} v_i^2 \vct{w_i} + \sum_{\substack{j=1 \\ i \neq j}}^k v_i v_j g\rbr{\vct{w_i},\vct{w_j}} - \frac{1}{2}v_i \a \\ 
    &= \sum_{j=1}^k v_i v_j g\rbr{\vct{w_i},\vct{w_j}} - \frac{1}{2} v_i \a \\ 
    &= \frac{1}{2\pi} \sum_{j=1}^k v_i v_j \twonorm{\vct{w_j}} \rbr{\sin\theta_{ij}\vct{\Bar{w_i}} + \rbr{\pi - \theta_{ij} }\vct{\Bar{w_j}}} - \frac{1}{2}v_i \a
\end{align*}
In matrix form:
\begin{align} \label{eq:pop_gradient_w}
    \nabla_{\mtx{W}}\calL\rbr{\mtx{\theta}} =& \frac{1}{2\pi} diag \rbr{\vct{v}} \rbr{\rbr{\pi \mathbbm{1}\mathbbm{1}^T - \mtx{\Theta}} diag\rbr{\vct{u}} + diag \rbr{\sin\rbr{\mtx{\Theta}} \vct{u}}}\mtx{\Bar{W}} - \frac{1}{2}\vct{v}\vct{a}^T
\end{align}
where $\vct{\omega}_i = \twonorm{\vct{w}_i}$, and $\vct{u} = diag\rbr{\vct{\omega}} \vct{v}$. 
\\
\\
\paragraph{Gradient w.r.t. $\v$:} Taking the derivative of (\ref{eq:pop_loss_closed_form_summation}) with respect to $v_i$, we get
\begin{align*}
    \nabla_{v_i}\calL\rbr{\mtx{\theta}} &= v_i f\rbr{\vct{w_i},\vct{w_i}} + \sum_{\substack{j=1 \\ i \neq j}}^k v_j f\rbr{\vct{w_i},\vct{w_j}} - \frac{1}{2}\a^T\vct{w_i} \\
    &= \sum_{j=1}^k v_j f\rbr{\vct{w_i},\vct{w_j}} - \frac{1}{2}\a^T\vct{w_i} \\
    &= \frac{1}{2\pi} \sum_{j=1}^k v_j \twonorm{\vct{w_i}} \twonorm{\vct{w_j}} \rbr{\sin\theta_{ij}+\rbr{\pi - \theta_{ij}} \cos\theta_{ij}} - \frac{1}{2}\a^T\vct{w_i}
\end{align*}
In matrix form:
\begin{align} \label{eq:pop_gradient_v}
    \nabla_{\v}Loss =& \frac{1}{2\pi} diag\rbr{\vct{\omega}}\rbr{\sin\mtx{\Theta} + \cos\mtx{\Theta}\odot\rbr{\pi\1\1^T - \mtx{\Theta}}}diag\rbr{\vct{\omega}} \v - \frac{1}{2}\mtx{W}\a
\end{align}
where $\vct{\omega}_i = \twonorm{\vct{w}_i}$. Finally, we note that the gradient w.r.t $\vct{w}_i$ and $v_i$ are related with the following simple identity:
\begin{align} \label{eq:grad_useful_identity}
    \vct{w}_i^T \nabla_{\vct{w_i}}\calL\rbr{\mtx{\theta}} = v_i \nabla_{v_i}\calL\rbr{\mtx{\theta}}.
\end{align}

\input{sec/apx/hessian_calculation}

%% file: sec/apx/hessian_calculation.tex
\subsection{Population Hessian}
The Hessian consists of four blocks (3 unique) due to interaction of $\v$ and $\W$ terms. We provide these individual blocks below and calculations in the following subsections. Define $\Bar{\vct{w}_l}=\frac{\vct{w}_l}{\twonorm{\vct{w}_l}}$, $\mtx{P}_{\vct{w}_l^\perp} = \rbr{I - \Bar{\vct{w}_l}\Bar{\vct{w}_l}^T}$, and $\vct{w}_{\ell,m^\perp} = \mtx{P}_{\vct{w}_m^\perp} \vct{w}_{\ell}$. Then we have,
\begin{align} \label{eq:pop_hessian}
    \nabla_{v_\ell, v_m}^2 \mathcal{L}\rbr{\vct{\theta}} &= \frac{\twonorm{\w_\ell}\twonorm{\w_m}}{2\pi} \rbr{\rbr{\pi - \theta_{\ell,m}}\cos\theta_{\ell,m} + \sin\theta_{\ell,m}}, \\
    \nabla_{v_\ell, \w_m}^2 \mathcal{L}\rbr{\vct{\theta}} &= \begin{cases}
     \frac{v_\ell \w_\ell^T - \a^T}{2} + \sum_{i=1}^k \frac{v_i \twonorm{\w_i}}{2\pi} \rbr{\rbr{\pi-\theta_{\ell,i}}\bar{\w}_i^T + \sin\theta_{\ell,i} \bar{\w}_\ell^T} & \ell=m\\
      \frac{v_m \twonorm{\w_\ell}}{2\pi} \rbr{\rbr{\pi - \theta_{\ell,m}} \bar{\w}_\ell^T + \sin\theta_{\ell,m} \bar{\w}_m^T} & \ell\neq m \nonumber
    \end{cases}, \\
    \nabla_{\vct{w}_\ell, \vct{w}_m}^2 \mathcal{L}\rbr{\vct{\theta}} &= \begin{cases}
      \frac{\vct{v}_\ell^2}{2}\mtx{I} + \frac{\vct{v}_\ell}{2\pi\twonorm{\vct{w}_\ell}} \sum_{i=1}^k \vct{v}_i \twonorm{\vct{w}_i} \sin\rbr{\theta_{\ell, i}} \rbr{\mtx{P}_{\vct{w}_l^\perp} + \frac{\mtx{P}_{\vct{w}_l^\perp} \bar{\vct{w}}_i \bar{\vct{w}}_i^T \mtx{P}_{\vct{w}_l^\perp}}{\twonorm{\mtx{P}_{\vct{w}_l^\perp} \bar{\vct{w}}_i }^2}} & \ell=m\\
      \frac{\vct{v}_\ell \vct{v}_m}{2\pi} \rbr{\Bar{\vct{w}}_\ell \Bar{\vct{w}}_{m,\ell^\perp}^T + \Bar{\vct{w}}_m \Bar{\vct{w}}_{\ell,m^\perp}^T + \rbr{\pi - \theta_{\ell, m}} \mtx{I}} & \ell\neq m \nonumber
    \end{cases}
\end{align}

\subsubsection{Calculating the $\texorpdfstring{\v,\v}{}$ block}
We have,
\begin{align*}
   \nabla_{v_\ell, v_m}^2 \mathcal{L}\rbr{\vct{\theta}} &= \E_{\x}\bbr{\rbr{f\rbr{\mtx{\theta};\x}-\a^T\x} \nabla_{v_{\ell}, v_m}^2 f\rbr{\mtx{\theta};\x} + \nabla_{v_{\ell}} f\rbr{\mtx{\theta};\x} \nabla_{v_{m}} f\rbr{\mtx{\theta};\x}^T} \\
   &= \E_{\x}\bbr{\nabla_{v_{\ell}} f\rbr{\mtx{\theta};\vct{x}} \nabla_{v_{m}} f\rbr{\mtx{\theta};\vct{x}}^T} \\
   &= \E_{\x} \bbr{\phi\rbr{\w_\ell^T \x} \phi\rbr{\w_m^T \x}} \\
   &= \frac{\twonorm{\w_\ell}\twonorm{\w_m}}{2\pi} \rbr{\rbr{\pi - \theta_{\ell,m}}\cos\theta_{\ell,m} + \sin\theta_{\ell,m}}
\end{align*}
where the last step follows from the Table 1 in \citep{daniely2016toward}.

\subsubsection{Calculating the $\texorpdfstring{\v,\text{vect}\rbr{\W}}{}$ block}
We have,
\begin{align*}
   \nabla_{v_\ell, \vct{w}_m}^2 \mathcal{L}\rbr{\vct{\theta}} &= \E_{\x}\bbr{\rbr{f\rbr{\mtx{\theta};\x}-\a^T\x} \nabla_{v_{\ell}, \w_m}^2 f\rbr{\mtx{\theta};\x} + \nabla_{v_{\ell}} f\rbr{\mtx{\theta};\x} \nabla_{\w_{m}} f\rbr{\mtx{\theta};\x}^T}.
\end{align*}
for calculation of individual terms refer down below.

\noindent \textbf{Calculating the $\E_{\vct{x}}\bbr{\nabla_{v_{\ell}} f\rbr{\mtx{\theta};\x} \nabla_{\w_{m}} f\rbr{\mtx{\theta};\x}^T}$ term:} We have,
\begin{align*}
    \E_{\x}\bbr{\nabla_{v_{\ell}} f\rbr{\mtx{\theta};\x} \nabla_{\w_{m}} f\rbr{\mtx{\theta};\x}^T} &= \E_{\x}\bbr{\phi\rbr{\w_\ell^T \x} v_m \phi'\rbr{\w_m^T \x} \x^T} \\
    &= v_m \twonorm{\w_\ell} \E_{\x}\bbr{\phi\rbr{\bar{\w}_\ell^T \x} \phi'\rbr{\bar{\w}_m^T \x} \x^T} \\
    &\stackrel{(a)}{=} v_m \twonorm{\w_\ell} \E_{\x}\bbr{\nabla_{\x}\rbr{\phi\rbr{\bar{\w}_\ell^T \x} \phi'\rbr{\w_m^T \x}}} \\ 
    &= v_m \twonorm{\w_\ell} \E_{\x}\bbr{\phi'\rbr{\bar{\w}_\ell^T \x} \phi'\rbr{\bar{\w}_m^T \x} \bar{\w}_\ell^T + \phi\rbr{\bar{\w}_\ell^T \x} \delta\rbr{\bar{\w}_m^T \x} \bar{\w}_m^T} \\
    &\stackrel{(b)}{=} v_m \twonorm{\w_\ell} \rbr{\rbr{\frac{\pi - \theta_{\ell, m}}{2\pi}}  \bar{\w}_\ell^T + \E_{\x}\bbr{\phi\rbr{\bar{\w}_\ell^T \x} \delta\rbr{\bar{\w}_m^T \x} \bar{\w}_m^T}}
\end{align*}
where equation (a) follows from Stein's Lemma, and (b) follows from the dual activation of a step function. The handle the remaining expectation, we first define $g = \bar{\w}_m^T\x \sim \mathcal{N}\rbr{0, 1}$. Then,
\begin{align*}
    \E_{\x}\bbr{\phi\rbr{\bar{\w}_\ell^T \x} \delta\rbr{\bar{\w}_m^T \x}} &= \E_{\x}\bbr{\phi\rbr{\bar{\w}_\ell^T \mtx{P}_{\w_m^{\perp}} \x + \bar{\w}_\ell^T \bar{\w}_m g} \delta\rbr{g}} \\
    &= \frac{1}{\sqrt{2\pi}} \E_{\x}\bbr{\phi\rbr{\bar{\w}_\ell^T \mtx{P}_{\w_m^{\perp}} \x }} \quad \rbr{\text{Delta integration}} \\
    &= \frac{\twonorm{\mtx{P}_{\w_m^{\perp}} \bar{\w}_\ell}}{\sqrt{2\pi}} \E_{u}\bbr{\phi\rbr{u}} \quad \rbr{u \sim N\rbr{0,1}} \\
    &= \frac{\sin\theta_{\ell,m}}{2\pi} \quad \rbr{\text{Expectation of rectified Gaussian $f_{\vct{x}}(0)=\frac{1}{\sqrt{2\pi}}$}}.
\end{align*}
Combining everything:
\begin{align*}
    \E_{\x}\bbr{\nabla_{v_{\ell}} f\rbr{\mtx{\theta};\x} \nabla_{\w_{m}} f\rbr{\mtx{\theta};\x}^T} = \frac{v_m \twonorm{\w_\ell}}{2\pi} \rbr{\rbr{\pi - \theta_{\ell, m}} \bar{\w}_\ell^T + \sin\theta_{\ell,m} \bar{\w}_m^T}.
\end{align*}

\noindent \textbf{Calculating the $\E_{\x}\bbr{\rbr{f\rbr{\mtx{\theta};\x}-\a^T\x} \nabla_{v_{\ell}, \w_m}^2 f\rbr{\mtx{\theta};\x}}$ term:} Note that 
\begin{align*}
    \nabla_{v_{\ell}, \w_m}^2 f\rbr{\mtx{\theta};\x} = \1\cbr{\ell = m} \phi'\rbr{\bar{\w}_\ell^T \x} \x^T.
\end{align*}
Hence we focus only on $\ell = m$ case.
\begin{align*}
    \E_{\x}\bbr{\rbr{f\rbr{\mtx{\theta};\x}-\a^T\x} \nabla_{v_{\ell}, \w_m}^2 f\rbr{\mtx{\theta};\x}} &= \E_{\x}\bbr{r\rbr{\x} \phi'\rbr{\bar{\w}_\ell^T \x} \x^T} \\
    &\stackrel{(a)}{=} \E_{\x}\bbr{\nabla_{\x}r\rbr{\x}\phi'\rbr{\bar{\w}_\ell^T\x} + r\rbr{\x}\delta\rbr{\bar{\w}_\ell^T \x} \bar{\w}_\ell^T} \\
    &= \E_{\x}\bbr{\rbr{\sum_{i=1}^k v_i \phi'\rbr{\w_i^T \x}\w_i - \a}^T\phi'\rbr{\bar{\w}_\ell^T\x} + r\rbr{\x}\delta\rbr{\bar{\w}_\ell^T \x} \bar{\w}_\ell^T} \\
    &\stackrel{(b)}{=} \sum_{i=1}^k v_i \w_i^T \rbr{\frac{\pi - \theta_{\ell,i}}{2\pi}} - \frac{\a}{2} + \E_{\x}\bbr{r\rbr{\x}\delta\rbr{\bar{\w}_\ell^T \x}}\bar{\w}_\ell^T,
\end{align*}
where (a) follows from the Stein's identity, and (b) follows from the dual activation of step function. To handle the remaining expectation term, define $g = \bar{\w}_{\ell}^T\x \sim \mathcal{N}\rbr{0, 1}$. Then,
\begin{align*}
    \E_{\x}\bbr{r\rbr{\x}\delta\rbr{\bar{\w}_\ell^T \x}} &= \E_{\x}\bbr{r\rbr{\mtx{P}_{\w_{\ell}^{\perp}} \x + \bar{\w}_\ell g} \delta\rbr{g}} \\
    &= \frac{1}{\sqrt{2\pi}} \E_{\vct{x}}\bbr{r\rbr{\mtx{P}_{\vct{w_l}^\perp}\vct{x}}} \quad \rbr{\text{Delta integration}} \\
    &= \frac{1}{\sqrt{2\pi}} \sum_{i=1}^k \vct{v}_i \E_{\vct{x}}\bbr{\phi\rbr{\vct{w_i}^T\mtx{P}_{\w_\ell^\perp} \vct{x}}} \\
    &= \frac{1}{\sqrt{2\pi}} \sum_{i=1}^k \vct{v}_i \twonorm{\mtx{P}_{\w_\ell^\perp}\w_i} \E_{u}\bbr{\phi\rbr{u}} \quad \rbr{u \sim N\rbr{0,1}} \\
    &= \frac{1}{2\pi}\sum_{i=1}^k \vct{v}_i \twonorm{\vct{w}_i} \sin\rbr{\theta_{\ell, i}} \quad \rbr{\text{Expectation of rectified Gaussian $f_{\vct{x}}(0)=\frac{1}{\sqrt{2\pi}}$}}
\end{align*}
Combining everything,
\begin{align*}
    \E_{\x}\bbr{\rbr{f\rbr{\mtx{\theta};\x}-\a^T\x} \nabla_{v_{\ell}, \w_m}^2 f\rbr{\mtx{\theta};\x}} = \sum_{i=1}^k \frac{v_i \twonorm{\w_i}}{2\pi} \rbr{\rbr{\pi - \theta_{\ell, i}} \bar{\w}_i^T + \sin\theta_{\ell,i}\bar{\w}_\ell^T} - \frac{\a^T}{2},
\end{align*}
when $\ell = m$. Otherwise, this term is $\vct{0}$.

\subsubsection{Calculating the $\texorpdfstring{\text{vect}\rbr{\W},\text{vect}\rbr{\W}}{}$ block}
We have,
\begin{align*}
   \nabla_{\vct{w}_\ell, \vct{w}_m}^2 \mathcal{L}\rbr{\vct{\theta}} &= \E_{\vct{x}}\bbr{\rbr{f\rbr{\mtx{\theta};\vct{x}}-\vct{a}^T\vct{x}} \nabla_{\vct{w}_{\ell}, \vct{w}_m}^2 f\rbr{\mtx{\theta};\vct{x}} + \nabla_{\vct{w}_{\ell}} f\rbr{\mtx{\theta};\vct{x}} \nabla_{\vct{w}_{m}} f\rbr{\mtx{\theta};\vct{x}}^T}.
\end{align*}
for calculation of individual terms refer down below.

\noindent \textbf{Calculating the $\E_{\vct{x}}\bbr{\nabla_{\vct{w}_{\ell}} f\rbr{\mtx{\theta};\vct{x}} \nabla_{\vct{w}_{m}} f\rbr{\mtx{\theta};\vct{x}}^T}$ term:} We have,
\begin{align*}
    \E_{\vct{x}}\bbr{\nabla_{\vct{w}_{\ell}} f\rbr{\mtx{\theta};\vct{x}} \nabla_{\vct{w}_{m}} f\rbr{\mtx{\theta};\vct{x}}^T} &= \E_{\vct{x}}\bbr{\vct{v}_\ell \phi'\rbr{\vct{w}_\ell^T \vct{x}} \vct{x} \vct{x}^T \phi'\rbr{\vct{w}_m^T \vct{x}} \vct{v}_m} \\
    &= \vct{v}_\ell \vct{v}_m \E_{\vct{x}}\bbr{\phi'\rbr{\Bar{\vct{w}}_\ell^T \vct{x}} \phi'\rbr{\Bar{\vct{w}}_m^T \vct{x}} \vct{x} \vct{x}^T }
\end{align*}

To tackle the expectation term, we use second order Stein's Lemma, $\E_{\vct{x}}\bbr{g(\vct{x})\vct{xx^T}} = \E_{\vct{x}}\bbr{\nabla_{\vct{x}}^2g\rbr{\vct{x}}} + \E_{\vct{x}}\bbr{g\rbr{\vct{x}}}\mtx{I}$. 

\begin{align*}
\E_{\vct{x}}\bbr{\phi'\rbr{\Bar{\vct{w}}_\ell^T \vct{x}} \phi'\rbr{\Bar{\vct{w}}_m^T \vct{x}} \vct{x} \vct{x}^T } &= \E_{\vct{x}}\bbr{\nabla_{\vct{x}}^2 \rbr{\phi'\rbr{\Bar{\vct{w}}_\ell^T \vct{x}} \phi'\rbr{\Bar{\vct{w}}_m^T \vct{x}}}} + \E_{\vct{x}}\bbr{\phi'\rbr{\Bar{\vct{w}}_\ell^T \vct{x}} \phi'\rbr{\Bar{\vct{w}}_m^T \vct{x}}} \mtx{I}
\end{align*}

\underline{First term is:}
\begin{align*}
    \E_{\vct{x}}\bbr{\nabla_{\vct{x}}^2 \rbr{\phi'\rbr{\Bar{\vct{w}}_\ell^T \vct{x}} \phi'\rbr{\Bar{\vct{w}}_m^T \vct{x}}}} &= \E_{\vct{x}}\bbr{\delta'\rbr{\Bar{\vct{w}}_\ell^T \vct{x}} \phi'\rbr{\Bar{\vct{w}}_m^T \vct{x}} \Bar{\vct{w}}_\ell \Bar{\vct{w}}_\ell^T} \\
    &\quad + \E_{\vct{x}}\bbr{\delta\rbr{\Bar{\vct{w}}_\ell^T \vct{x}} \delta\rbr{\Bar{\vct{w}}_m^T \vct{x}} \rbr{\Bar{\vct{w}}_\ell \Bar{\vct{w}}_m^T + \Bar{\vct{w}}_m \Bar{\vct{w}}_\ell^T}} \\
    &\quad\quad + \E_{\vct{x}}\bbr{\phi'\rbr{\Bar{\vct{w}}_\ell^T \vct{x}} \delta'\rbr{\Bar{\vct{w}}_m^T \vct{x}} \Bar{\vct{w}}_m \Bar{\vct{w}}_m^T} \\
\end{align*}
These terms can be grouped in two.
\begin{align*}
    \E_{\vct{x}}\bbr{\delta'\rbr{\Bar{\vct{w}}_\ell^T \vct{x}} \phi'\rbr{\Bar{\vct{w}}_m^T \vct{x}}} &= \E_{\vct{x}, g}\bbr{\delta'\rbr{g} \phi'\rbr{\Bar{\vct{w}}_m^T \projw{\ell} \vct{x} + \Bar{\vct{w}}_m^T \Bar{\vct{w}}_\ell g}} \\
    &= -\frac{\cos\rbr{\theta_{\ell,m}}}{\sqrt{2\pi}} \E_{\vct{x}}\bbr{\delta \rbr{\Bar{\vct{w}}_m^T \projw{\ell} \vct{x}}} \quad \rbr{\delta'\rbr{x}f\rbr{x} = -f'\rbr{0}\delta\rbr{x}} \\
    &= -\frac{\cos\rbr{\theta_{\ell,m}}}{2 \pi \twonorm{\projw{\ell} \Bar{\vct{w}}_m}} = -\frac{\cos\rbr{\theta_{\ell,m}}}{2 \pi \sin\rbr{\theta_{\ell,m}}}
\end{align*}
and the other one is
\begin{align*}
    \E_{\vct{x}}\bbr{\delta\rbr{\Bar{\vct{w}}_\ell^T \vct{x}} \delta\rbr{\Bar{\vct{w}}_m^T \vct{x}}} &= \E_{\vct{x}, g}\bbr{\delta\rbr{g} \delta\rbr{\Bar{\vct{w}}_m^T \projw{\ell} \vct{x} + \Bar{\vct{w}}_m^T \Bar{\vct{w}}_\ell g}} \\
    &= \frac{1}{2 \pi \twonorm{\projw{\ell} \Bar{\vct{w}}_m}} =  \frac{1}{2 \pi \sin\rbr{\theta_{\ell, m}}}
\end{align*}
Therefore we get:
\begin{align*}
    \E_{\vct{x}}\bbr{\nabla_{\vct{x}}^2 \rbr{\phi'\rbr{\Bar{\vct{w}}_\ell^T \vct{x}} \phi'\rbr{\Bar{\vct{w}}_m^T \vct{x}}}} &= \frac{\Bar{\vct{w}}_\ell \Bar{\vct{w}}_m^T + \Bar{\vct{w}}_m \Bar{\vct{w}}_\ell^T}{2\pi \sin\rbr{\theta_{\ell,m}}} - \frac{\cos\rbr{\theta_{\ell,m}} \rbr{\Bar{\vct{w}}_\ell \Bar{\vct{w}}_\ell^T + \Bar{\vct{w}}_m \Bar{\vct{w}}_m^T}}{2\pi \sin\rbr{\theta_{\ell,m}}}
\end{align*}
\underline{Second term is:}
\begin{align*}
\E_{\vct{x}}\bbr{\phi'\rbr{\Bar{\vct{w}}_\ell^T \vct{x}} \phi'\rbr{\Bar{\vct{w}}_m^T \vct{x}}} \mtx{I} = \rbr{\frac{\pi - \theta_{\ell, m}}{2\pi}} \mtx{I} \quad \rbr{\text{Dual activation of step function}}
\end{align*}

Combining everything:
\begin{align*}
\E_{\vct{x}}\bbr{\nabla_{\vct{w}_{\ell}} f\rbr{\mtx{\theta};\vct{x}} \nabla_{\vct{w}_{m}} f\rbr{\mtx{\theta};\vct{x}}^T} &= \vct{v}_\ell \vct{v}_m \rbr{\frac{\Bar{\vct{w}}_\ell \Bar{\vct{w}}_m^T + \Bar{\vct{w}}_m \Bar{\vct{w}}_\ell^T - \cos\rbr{\theta_{\ell,m}} \rbr{\Bar{\vct{w}}_\ell \Bar{\vct{w}}_\ell^T + \Bar{\vct{w}}_m \Bar{\vct{w}}_m^T}}{2\pi \sin\rbr{\theta_{\ell,m}}} + \rbr{\frac{\pi - \theta_{\ell, m}}{2\pi}} \mtx{I}}
\end{align*}

or alternatively (by substituting $\cos\rbr{\theta_{i,j}} = \Bar{w}_i^T\Bar{w}_j$): 
\begin{align*}
\E_{\vct{x}}\bbr{\nabla_{\vct{w}_{\ell}} f\rbr{\mtx{\theta};\vct{x}} \nabla_{\vct{w}_{m}} f\rbr{\mtx{\theta};\vct{x}}^T} &= \frac{\vct{v}_\ell \vct{v}_m}{2\pi} \rbr{\Bar{\vct{w}}_\ell \Bar{\vct{w}}_{m,\ell^\perp}^T + \Bar{\vct{w}}_m \Bar{\vct{w}}_{\ell,m^\perp}^T + \rbr{\pi - \theta_{\ell, m}} \mtx{I}}
\end{align*}

\noindent \textbf{Calculating the $\E_{\vct{x}}\bbr{\rbr{f\rbr{\mtx{\theta};\vct{x}}-\vct{a}^T\vct{x}} \nabla_{\vct{w}_{\ell}, \vct{w}_m}^2 f\rbr{\mtx{\theta};\vct{x}} }$ term:} Note that 
\begin{align*}
    \nabla_{\vct{w}_{\ell}, \vct{w}_m}^2 f\rbr{\mtx{\theta};\vct{x}} = diag\rbr{\vct{v} \odot \phi''\rbr{\mtx{Wx}}}_{\ell, m} \mtx{x x^T}.
\end{align*}
This expectation is $\mtx{0}$ when $\ell \neq m$. Define $g = \Bar{\vct{w_l}}^T\vct{x} \sim \mathcal{N}\rbr{0, 1}$.
\begin{align*}
    \E_{\vct{x}}\bbr{\rbr{f\rbr{\mtx{\theta};\vct{x}}-\vct{a}^T\vct{x}} \nabla_{\vct{w}_{\ell}, \vct{w}_\ell}^2 f\rbr{\mtx{\theta};\vct{x}} } &= \E_{\vct{x}}\bbr{r\rbr{\vct{x}} \vct{v}_\ell \delta\rbr{\vct{w}_\ell^T\vct{x}}\mtx{xx^T}} \\
    &=  \frac{\vct{v}_\ell}{\twonorm{\vct{w}_\ell}}  \E_{\vct{x}}\bbr{r\rbr{\vct{x}} \delta\rbr{\Bar{\vct{w}_{\ell}}^T\vct{x}}\mtx{xx^T}} \\
    &= \frac{\vct{v}_\ell}{\twonorm{\vct{w}_\ell}}  \E_{\vct{x}, g}\bbr{r\rbr{\mtx{P}_{\vct{w_l}^\perp}\vct{x} + \Bar{\vct{w_l}}g} \delta\rbr{g}\rbr{\mtx{P}_{\vct{w_l}^\perp}\vct{x} + \Bar{\vct{w_l}}g}\rbr{\mtx{P}_{\vct{w_l}^\perp}\vct{x} + \Bar{\vct{w_l}}g}^T} \\
    &= \frac{\vct{v}_\ell}{\sqrt{2 \pi}\twonorm{\vct{w}_\ell}}  \mtx{P}_{\vct{w_l}^\perp} \E_{\vct{x}}\bbr{r\rbr{\mtx{P}_{\vct{w_l}^\perp}\vct{x}}\mtx{xx^T}} \mtx{P}_{\vct{w_l}^\perp}\\
\end{align*}
To tackle the expectation term, we use second order Stein's Lemma, $\E_{\vct{x}}\bbr{g(\vct{x})\vct{xx^T}} = \E_{\vct{x}}\bbr{\nabla_{\vct{x}}^2g\rbr{\vct{x}}} + \E_{\vct{x}}\bbr{g\rbr{\vct{x}}}\mtx{I}$. 

\begin{align*}
\E_{\vct{x}}\bbr{r\rbr{\mtx{P}_{\vct{w_l}^\perp}\vct{x}}\mtx{xx^T}} &= \E_{\vct{x}}\bbr{\nabla_{\vct{x}}^2 r\rbr{\mtx{P}_{\vct{w_l}^\perp}\vct{x}}} + \E_{\vct{x}}\bbr{r\rbr{\mtx{P}_{\vct{w_l}^\perp}\vct{x}}}\mtx{I}
\end{align*}

First term is:
\begin{align*}
    \E_{\vct{x}}\bbr{\nabla_{\vct{x}}^2 r\rbr{\mtx{P}_{\vct{w_l}^\perp}\vct{x}}} &= \E_{\vct{x}}\bbr{\nabla_{\vct{x}}^2 \rbr{\sum_{i=1}^k \vct{v}_i \phi\rbr{\vct{w}_i^T\mtx{P}_{\vct{w_l}^\perp}\vct{x}}}} \quad \rbr{\vct{a}^T\vct{x} \ \text{vanishes.}} \\
    &= \sum_{i=1}^k \vct{v}_i \E_{\vct{x}}\bbr{\nabla_{\vct{x}}^2\phi\rbr{\vct{w}_i^T\mtx{P}_{\vct{w_l}^\perp}\vct{x}}} \\
    &= \sum_{i=1}^k \vct{v}_i \E_{\vct{x}}\bbr{\delta\rbr{\vct{w}_i^T\mtx{P}_{\vct{w_l}^\perp}\vct{x}}\mtx{P}_{\vct{w_l}^\perp} \vct{w}_i \vct{w}_i^T \mtx{P}_{\vct{w_l}^\perp}} \\
    &= \sum_{i=1}^k \frac{\vct{v}_i}{\twonorm{\mtx{P}_{\vct{w_l}^\perp}\vct{w_i}}} \E_{u}\bbr{\delta\rbr{u}} \mtx{P}_{\vct{w_l}^\perp} \vct{w}_i \vct{w}_i^T \mtx{P}_{\vct{w_l}^\perp} \\
    &= \frac{1}{\sqrt{2\pi}} \sum_{i=1}^k \frac{\vct{v}_i}{\twonorm{\vct{w}_i} \sin\rbr{\theta_{\ell, i}}} \mtx{P}_{\vct{w_l}^\perp} \vct{w}_i \vct{w}_i^T \mtx{P}_{\vct{w_l}^\perp} \quad \rbr{\text{Delta integration}} \\
    &= \frac{1}{\sqrt{2\pi}} \sum_{i=1}^k \frac{\vct{v}_i \twonorm{\vct{w}_i}}{\sin\rbr{\theta_{\ell, i}}} \mtx{P}_{\vct{w_l}^\perp} \bar{\vct{w}}_i \bar{\vct{w}}_i^T \mtx{P}_{\vct{w_l}^\perp}
\end{align*}

Second term is:
\begin{align*}
    \E_{\vct{x}}\bbr{r\rbr{\mtx{P}_{\vct{w_l}^\perp}\vct{x}}} &= \sum_{i=1}^k \vct{v}_i \E_{\vct{x}}\bbr{\phi\rbr{\vct{w_i}^T\mtx{P}_{\vct{w_l}^\perp} \vct{x}}} \\
    &= \sum_{i=1}^k \vct{v}_i \twonorm{\mtx{P}_{\vct{w_l}^\perp}\vct{w_i}} \E_{u}\bbr{\phi\rbr{u}} \quad \rbr{u \sim N\rbr{0,1}} \\
    &= \frac{1}{\sqrt{2\pi}}\sum_{i=1}^k \vct{v}_i \twonorm{\vct{w}_i} \sin\rbr{\theta_{\ell, i}} \quad \rbr{\text{Expectation of rectified Gaussian $f_{\vct{x}}(0)=\frac{1}{\sqrt{2\pi}}$}}
\end{align*}
Combining both terms we get
\begin{align*}
\E_{\vct{x}}\bbr{r\rbr{\mtx{P}_{\vct{w_l}^\perp}\vct{x}}\mtx{xx^T}} &= \frac{1}{\sqrt{2\pi}}\sum_{i=1}^k \vct{v}_i \twonorm{\vct{w}_i} \rbr{\sin\rbr{\theta_{\ell, i}} \mtx{I} + \frac{\mtx{P}_{\vct{w_l}^\perp} \bar{\vct{w}}_i \bar{\vct{w}}_i^T \mtx{P}_{\vct{w_l}^\perp}}{ \sin\rbr{\theta_{\ell, i}}}}.
\end{align*}

Finally we plug this back to get:
\begin{align*}
    \E_{\vct{x}}\bbr{r\rbr{\vct{x}} \nabla_{\vct{w}_{\ell}, \vct{w}_\ell}^2 f\rbr{\mtx{\theta};\vct{x}} } &= \frac{\vct{v}_\ell}{2 \pi\twonorm{\vct{w}_\ell}}  \mtx{P}_{\vct{w_l}^\perp} \rbr{\sum_{i=1}^k \vct{v}_i \twonorm{\vct{w}_i} \rbr{\sin\rbr{\theta_{\ell, i}} \mtx{I} + \frac{\mtx{P}_{\vct{w_l}^\perp} \bar{\vct{w}}_i \bar{\vct{w}}_i^T \mtx{P}_{\vct{w_l}^\perp}}{ \sin\rbr{\theta_{\ell, i}}}}} \mtx{P}_{\vct{w_l}^\perp} \\
    &= \frac{\vct{v}_\ell}{2\pi\twonorm{\vct{w}_\ell}} \sum_{i=1}^k \vct{v}_i \twonorm{\vct{w}_i} \rbr{\sin\rbr{\theta_{\ell, i}}  \mtx{P}_{\vct{w_l}^\perp} + \frac{\mtx{P}_{\vct{w_l}^\perp} \bar{\vct{w}}_i \bar{\vct{w}}_i^T \mtx{P}_{\vct{w_l}^\perp}}{\sin\rbr{\theta_{\ell, i}}}} \\
    &= \frac{\vct{v}_\ell}{2\pi\twonorm{\vct{w}_\ell}} \sum_{i=1}^k \vct{v}_i \twonorm{\vct{w}_i} \sin\rbr{\theta_{\ell, i}} \rbr{\mtx{P}_{\vct{w_l}^\perp} + \frac{\mtx{P}_{\vct{w_l}^\perp} \bar{\vct{w}}_i \bar{\vct{w}}_i^T \mtx{P}_{\vct{w_l}^\perp}}{\twonorm{\mtx{P}_{\vct{w_l}^\perp} \bar{\vct{w}}_i }^2}}.
\end{align*}

%% file: sec/apx/old/proof_pop_directional_smoothness.tex
To establish the imbalance bound (Lemma~\ref{lemma:imbalance_bound}), we first introduce a key lemma that characterizes the gradient smoothness toward the global optima in the population case. This result relates the norm of the population gradient to the relative distance between the current parameters and the global optima:

\begin{lemma} \label{lemma:popgrad_smooth}
Under the constraint $v_1=v_2=1$, the following inequality holds for all $\w_1,\w_2\in\R^{d}$:
\begin{align*}
    \twonorm{\nabla_{\w_1}\mathcal{L}}^2 + \twonorm{\nabla_{\w_2}\mathcal{L}}^2 \leq \frac{5}{2} \rbr{\twonorm{\w_1-\a}^2 + \twonorm{\w_2+\a}^2}.
\end{align*}
\end{lemma}

\begin{proof}

We begin by demonstrating that:
\begin{align} \label{eq:bound_on_part_of_grad_w1}
\twonorm{\sin\theta\twonorm{\w_2}\Bar{\w_1} + \rbr{\pi-\theta}\w_2} \le \pi \twonorm{\w_1 + \w_2}.
\end{align}

Note that $0\le \theta \le \pi$ since it is the angle between $\w_1$ and $\w_2$. We proceed by case analysis on the value of $\theta$. When $0\le\theta< \frac{\pi}{2}$, we have 
\begin{align*}
\twonorm{\sin\theta\twonorm{\w_2}\Bar{\w_1} + \rbr{\pi-\theta}\w_2} \stackrel{(a)}{\le} & \twonorm{\sin\theta\twonorm{\w_2}\Bar{\w_1}} + \twonorm{\rbr{\pi-\theta}\w_2} \\
    = & \sin\theta\twonorm{\w_2} + \rbr{\pi-\theta}\twonorm{\w_2} \\
    \stackrel{(b)}{\le} & \pi \twonorm{\w_2} \\
    \stackrel{(c)}{\le} & \pi \twonorm{\w_1 + \w_2}.
\end{align*}
In Inequality (a) we use the triangle inequality. Inequality (b) follows from the fact that $\sin\theta\le \theta$ when $\theta\ge 0$. Inequality (c) follows from the fact that $\theta\le \frac{\pi}{2}$.

When $\theta\ge \frac{\pi}{2}$, we observe that
\begin{align*}
\twonorm{\sin\theta\twonorm{\w_2}\Bar{\w_1} + \rbr{\pi-\theta}\w_2} \stackrel{(a)}{\le} & \twonorm{\sin\theta\twonorm{\w_2}\Bar{\w_1}} + \twonorm{\rbr{\pi-\theta}\w_2} \\
    = & \sin\theta\twonorm{\w_2} + \rbr{\pi-\theta}\twonorm{\w_2} \\
    = & \rbr{1 + \frac{\pi-\theta}{\sin\theta}} \sin\theta \twonorm{\w_2} \\
    \stackrel{(b)}{\le} & \rbr{1 + \frac{\pi-\theta}{\sin\theta}} \twonorm{\w_1 + \w_2} \\
    \stackrel{(c)}{\le} & \rbr{1 + \frac{\pi}{2}} \twonorm{\w_1 + \w_2} \\
    \stackrel{(d)}{\le} & \pi \twonorm{\w_1 + \w_2}.
\end{align*}
In Inequality (a) we use the triangle inequality. Inequality (b) follows from the fact that $\w_1 + \w_2$ has a component with magnitude $\sin\theta \twonorm{\w_2}$ perpendicular to $\w_1$. Inequality (c) follows since $\frac{\pi-\theta}{\sin\theta}$ attains its maximum at $\theta=\frac{\pi}{2}$ when restricted to the range $\theta\geq \frac{\pi}{2}$. Finally, (d) follows because $1 \leq \frac{\pi}{2}$. This finishes the proof of Ineq. \ref{eq:bound_on_part_of_grad_w1}. Note that due to symmetry we get the following as a corollary:
\begin{equation} \label{eq:bound_on_part_of_grad_w2}
\twonorm{\sin\theta\twonorm{\w_1}\Bar{\w_2} + \rbr{\pi-\theta}\w_1} \le \pi \twonorm{\w_1 + \w_2}.
\end{equation}
Under the constraint $v_1=v_2=1$, the partial gradients with respect to $\w_1$ and $\w_2$ are given separately by:
\begin{align*}
    &\nabla_{\w_1}\mathcal{L} = -\frac{\a}{2}+\frac{1}{2\pi} \rbr{\pi \w_1 - \sin\theta\twonorm{\w_2}\Bar{\w_1} - \rbr{\pi-\theta}\w_2} \\
    &\nabla_{\w_2}\mathcal{L} = \frac{\a}{2}-\frac{1}{2\pi} \rbr{\rbr{\pi-\theta}\w_1 + \sin\theta\twonorm{\w_1}\Bar{\w_2} - \pi \w_2}.
\end{align*}
Using Ineq. \ref{eq:bound_on_part_of_grad_w1}, we can write
\begin{align*}
    \twonorm{\nabla_{\w_1}\mathcal{L}}^2 &= \twonorm{-\frac{\a}{2}+\frac{1}{2\pi} \rbr{\pi \w_1 - \sin\theta\twonorm{\w_2}\Bar{\w_1} - \rbr{\pi-\theta}\w_2}}^2 \\
    &= \twonorm{\frac{\w_1-\a}{2} - \frac{1}{2\pi} \rbr{\sin\theta\twonorm{\w_2}\Bar{\w_1} + \rbr{\pi-\theta}\w_2}}^2 \\
    &\leq 2\twonorm{\frac{\w_1-\a}{2}}^2 + 2\twonorm{\frac{1}{2\pi} \rbr{\sin\theta\twonorm{\w_2}\Bar{\w_1} + \rbr{\pi-\theta}\w_2}}^2 \\
    &\leq \frac{1}{2}\twonorm{\w_1-\a}^2 + \frac{1}{2}\twonorm{\w_1+\w_2}^2 \\
    &= \frac{1}{2}\twonorm{\w_1-\a}^2 + \frac{1}{2}\twonorm{\w_1-\a + \a + \w_2}^2 \\
    &\leq \frac{3}{2}\twonorm{\w_1-\a}^2 + \twonorm{\w_2+\a}^2.
\end{align*}
Similarly, using Eq. \ref{eq:bound_on_part_of_grad_w2} on the gradient for $\w_2$, we get
\begin{align*}
    \twonorm{\nabla_{\w_2}\mathcal{L}}^2 \leq \frac{3}{2} \twonorm{\w_2+\a}^2 + \twonorm{\w_1-\a}^2.
\end{align*}
Combining these, we obtain
\begin{align*}
    \twonorm{\nabla_{\w_1}\mathcal{L}}^2 + \twonorm{\nabla_{\w_2}\mathcal{L}}^2 \leq \frac{5}{2} \rbr{\twonorm{\w_1-\a}^2 + \twonorm{\w_2+\a}^2}.
\end{align*}
This completes the proof of Lemma \ref{lemma:popgrad_smooth}.
\end{proof}

%% file: sec/apx/proof_pop_pl_ineq.tex
First, we show it is sufficient to analyze $v_1=v_2=1$. For $v_1,v_2 >0$, we define $\tilde{\w}_i = v_i \w_i$. Then,
\begin{align*}
    \calL\rbr{\mtx{\theta}} = \E_{\x}\bbr{\rbr{v_1 \text{ReLU}\rbr{\w_1^T \x} - v_2 \text{ReLU} \rbr{\w_2^T \x} - \a^T \x}^2} = \E_{\x}\bbr{\rbr{\text{ReLU}\rbr{\tilde{\w}_1^T \x} - \text{ReLU} \rbr{\tilde{\w}_2^T \x} - \a^T \x}^2}.
\end{align*}

Let us focus on squared gradient norms:
\begin{align*}
    \twonorm{\nabla_{\w_1} \calL }^2 + \twonorm{\nabla_{\w_2} \calL }^2 =& \twonorm{v_1 \nabla_{\tilde{\w}_1} \calL }^2 + \twonorm{v_2 \nabla_{\tilde{\w}_2} \calL }^2 \\
    \geq& \min\rbr{v_1^2, v_2^2} \rbr{\twonorm{\nabla_{\tilde{\w}_1} \calL }^2 + \twonorm{\nabla_{\tilde{\w}_2} \calL }^2}
\end{align*}
This suggests that proving $\twonorm{\nabla_{\w_1} \calL }^2 + \twonorm{\nabla_{\w_2} \calL }^2 \geq \alpha \calL$ when $v_1=v_2=1$ implies that
\begin{align*}
    \twonorm{\nabla_{\w_1} \calL }^2 + \twonorm{\nabla_{\w_2} \calL }^2 \geq \min\rbr{v_1^2, v_2^2} \alpha \calL
\end{align*}
for arbitrary $v_1, v_2 > 0$. Now, we assume $v_1=v_2=1$. We define
\begin{align*}
    h\rbr{\w_1, \w_2, \a} &= \twonorm{\nabla_{\w_1}\calL}^2 + \twonorm{\nabla_{\w_2}\calL}^2 - \alpha\calL.
\end{align*}

Using the gradient calculations in \eqref{eq:pop_gradient_w}, we can write it equivalently as
\begin{align*}
    h\rbr{\w_1, \w_2, \a} &= \frac{1}{4}\rbr{1-\alpha + \frac{\rbr{\pi-\theta}^2 + 2\rbr{\pi-\theta}\sin\theta\cos\theta +\rbr{\sin\theta}^2}{\pi^2}}\rbr{\twonorm{\w_1}^2 + \twonorm{\w_2}^2} \\
    &\quad - \rbr{1-\frac{\alpha}{2}}\frac{\rbr{\pi-\theta}\cos\theta + \sin\theta}{\pi} \twonorm{\w_1}\twonorm{\w_2} \\
    &\quad - \frac{1}{2}\rbr{\rbr{\frac{\pi-\theta}{\pi} + 1 - \alpha} \twonorm{\w_1} - \frac{\sin\theta}{\pi}\twonorm{\w_2}} \Bar{\w}_1^T \a \\
    &\quad + \frac{1}{2}\rbr{\rbr{\frac{\pi-\theta}{\pi} + 1 - \alpha} \twonorm{\w_2} - \frac{\sin\theta}{\pi}\twonorm{\w_1}} \Bar{\w}_2^T \a \\
    &\quad + \rbr{\frac{1-\alpha}{2}} \twonorm{\a}^2 \\
    &= \frac{1-\alpha}{2} \twonorm{\a}^2 + \vct{b}^T \a + c 
\end{align*}
where $\vct{b}$ and $c$ are defined by the following terms for brevity,
\begin{align*}
    \alpha_1 &= - \frac{1}{2}\rbr{\rbr{\frac{\pi-\theta}{\pi} + 1 - \alpha} \twonorm{\w_1} - \frac{\sin\theta}{\pi}\twonorm{\w_2}} \\
    \alpha_2 &= + \frac{1}{2}\rbr{\rbr{\frac{\pi-\theta}{\pi} + 1 - \alpha} \twonorm{\w_2} - \frac{\sin\theta}{\pi}\twonorm{\w_1}} \\
    c &= \frac{1}{4}\rbr{1-\alpha + \frac{\rbr{\pi-\theta}^2 + 2\rbr{\pi-\theta}\sin\theta\cos\theta +\rbr{\sin\theta}^2}{\pi^2}}\rbr{\twonorm{\w_1}^2 + \twonorm{\w_2}^2} \\
    &\quad - \rbr{1-\frac{\alpha}{2}}\frac{\rbr{\pi-\theta}\cos\theta + \sin\theta}{\pi} \twonorm{\w_1}\twonorm{\w_2}  \\
    \vct{b} &= \alpha_1 \Bar{\w}_1 + \alpha_2 \Bar{\w}_2 
\end{align*}
Noting that the expression above is quadratic in $\a$, we compute $\Tilde{h}\rbr{\w_1,\w_2} = \min\limits_{\a}  h\rbr{\w_1,\w_2,\a}$. The choice of $\a$ that minimizes the expression is $\a = -\frac{\vct{b}}{1-\alpha}$. Plugging this in back we get,
\begin{align*}
    \Tilde{h}\rbr{\w_1,\w_2} = c - \frac{\twonorm{\vct{b}}^2}{2 \rbr{1-\alpha}} = c - \frac{\alpha_1^2 + 2 \alpha_1 \alpha_2 \cos\theta + \alpha_2^2}{2 \rbr{1-\alpha}}
\end{align*}

\paragraph{Taking the norm out:} Note that we are only interested in the positivity of $\Tilde{h}$, therefore dividing it by $\twonorm{\w_2}^2$ does not change the sign. Denote $\frac{\twonorm{\w_1}}{\twonorm{\w_2}} = r$. Then we still have
\begin{align*}
    \frac{\Tilde{h}\rbr{\w_1,\w_2}}{\twonorm{\w_2}^2} = c - \frac{\alpha_1^2 + 2 \alpha_1 \alpha_2 \cos\theta + \alpha_2^2}{2 \rbr{1-\alpha}}
\end{align*}
but the variables are modified as
\begin{align*}
    \alpha_1 &= - \frac{1}{2}\rbr{\rbr{\frac{\pi-\theta}{\pi} + 1 - \alpha} r - \frac{\sin\theta}{\pi}} \\
    \alpha_2 &= + \frac{1}{2}\rbr{\rbr{\frac{\pi-\theta}{\pi} + 1 - \alpha} - \frac{\sin\theta}{\pi}r} \\
    c &= \frac{1}{4}\rbr{1-\alpha + \frac{\rbr{\pi-\theta}^2 + 2\rbr{\pi-\theta}\sin\theta\cos\theta +\rbr{\sin\theta}^2}{\pi^2}}\rbr{r^2 + 1} - \rbr{1-\frac{\alpha}{2}}\frac{\rbr{\pi-\theta}\cos\theta + \sin\theta}{\pi} r
\end{align*}
We note that the expression is of the form $C_1 \rbr{r^2+1} + C_2 r$. Without changing the sign, we can take out $r$ outside. Then we notice that the minima is achieved at $r=1$. Therefore, it is sufficient for us to check the positivity of the expression at $r=1$. That is, we draw:

\begin{align*}
    \frac{\Tilde{h}\rbr{\w_1,\w_2}}{\twonorm{\w_2}^2} \Bigg|_{r=1} = c - \frac{\rbr{1 - \cos\theta}}{\rbr{1-\alpha}} \tilde{\alpha}^2 
\end{align*}
where
\begin{align*}
    \tilde{\alpha} &= \frac{1}{2}\rbr{\frac{\pi-\theta -\sin\theta}{\pi} + 1 - \alpha}, \\
    c &= \frac{1}{2}\rbr{1-\alpha + \frac{\rbr{\pi-\theta}^2 + 2\rbr{\pi-\theta}\sin\theta\cos\theta +\rbr{\sin\theta}^2}{\pi^2}} - \rbr{1-\frac{\alpha}{2}}\frac{\rbr{\pi-\theta}\cos\theta + \sin\theta}{\pi}.
\end{align*}

\begin{figure}[h]
    \centering
    \includegraphics[width=0.9\linewidth]{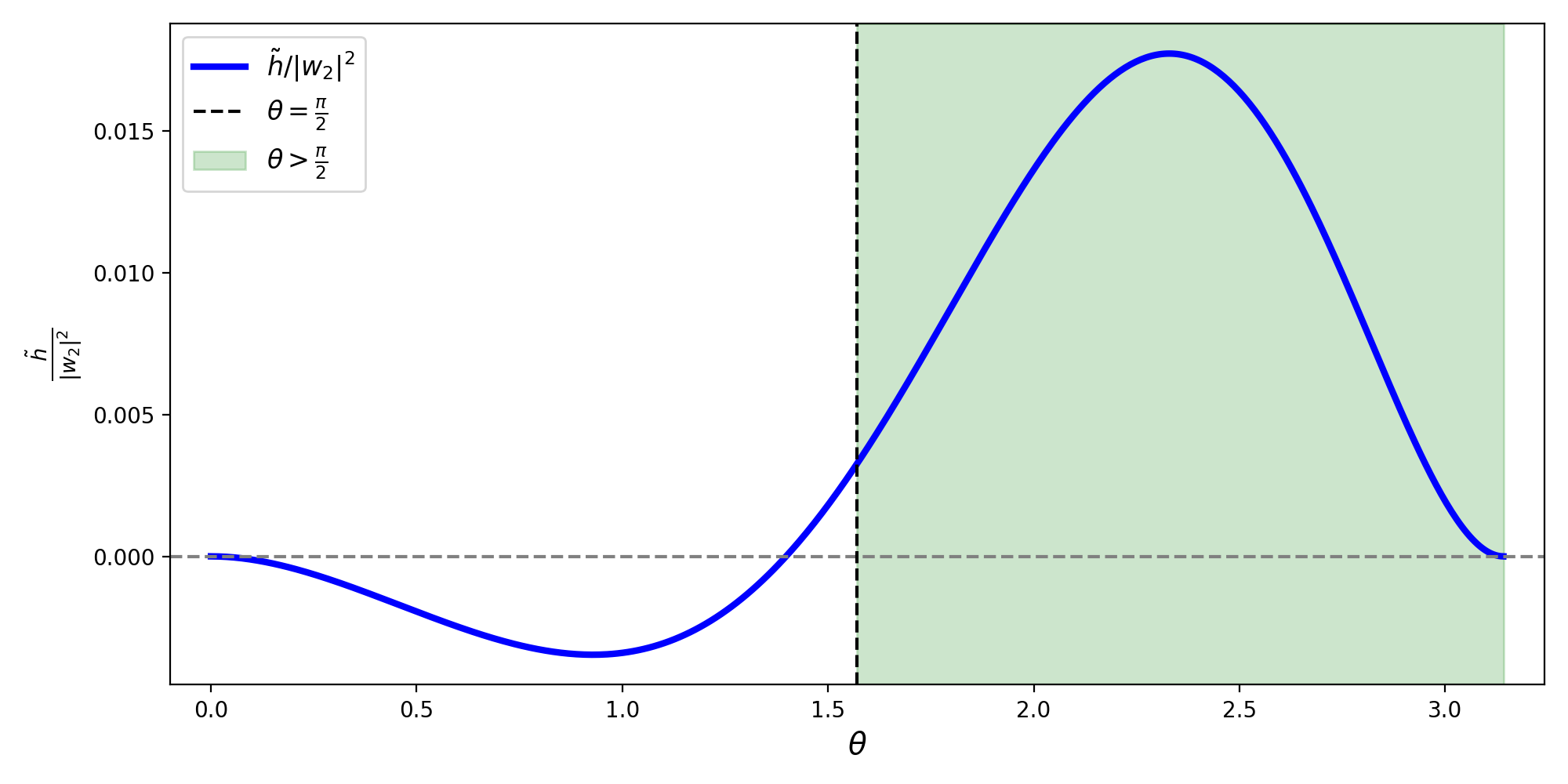}
    \caption{$\twonorm{\nabla_{\w_1}\calL}^2 + \twonorm{\nabla_{\w_2}\calL}^2 - \alpha\calL$ \textbf{is non-negative}. We set $\alpha=0.05$ and draw $\frac{1}{\twonorm{\w_2}}\Tilde{h}\rbr{\w_1,\w_2}$ for $\theta \in \bbr{0,\pi}$. We show that $\frac{1}{\twonorm{\w_2}}\Tilde{h}\rbr{\w_1,\w_2}$ is non-negative inside the shaded region ($\theta \ge \frac{\pi}{2}$).}
    \label{fig:apx_proof_pl_ineq}
\end{figure}

To complete the proof, in Figure \ref{fig:apx_proof_pl_ineq}, we set $\alpha = 0.05$ and draw $\frac{\Tilde{h}\rbr{\w_1,\w_2}}{\twonorm{\w_2}^2} \Bigg|_{r=1}$ as a 1D plot for $\theta \in \bbr{0,\pi}$. The plot demonstrates that $\Tilde{h}$ is non-negative for $\theta > \frac{\pi}{2}$. This finishes the proof.

%% file: sec/apx/proof_pop_smoothness.tex
We bound the population Hessian $\nabla^2 \calL\rbr{\v,\W}$ in the local refinement phase. That is, we assume $c_1 \sqrt{\twonorm{\a}} \leq v_1, v_2, \twonorm{\w_1}, \twonorm{\w_2} \leq c_2 \sqrt{\twonorm{\a}}$. By the sub-additivity properties of the spectral norm, we have
\begin{align*}
    \twonorm{\nabla^2 \calL\rbr{\v, \W}}_2 \le & \twonorm{\nabla_{\v,\v}^2 \calL\rbr{\v, \W}}_2 + 2\twonorm{\nabla_{\v,\text{vect}\rbr{\W}}^2 \calL\rbr{\v, \W}}_2 + \twonorm{\nabla_{\text{vect}\rbr{\W},\text{vect}\rbr{\W}}^2 \calL\rbr{\v, \W}}_2.
\end{align*}
We bound each term separately below.

\paragraph{$\nabla_{\v, \v}^2 \calL\rbr{\mtx{\theta}}$ term:} We have
\begin{align*}
    \twonorm{\nabla_{\v, \v}^2 \calL\rbr{\mtx{\theta}}}_2 \le & \sum_{\ell =1}^2\sum_{m =1}^2\left|\nabla_{v_{\ell},v_m}^2 \calL\rbr{\mtx{\theta}}\right|.
\end{align*}
where,
\begin{align*}
    \left|\nabla_{v_{\ell},v_m}^2 \calL\rbr{\mtx{\theta}}\right| = & \frac{\twonorm{\w_\ell}\twonorm{\w_m}}{2\pi} \rbr{\rbr{\pi - \theta_{\ell,m}}\cos\theta_{\ell,m} + \sin\theta_{\ell,m}} \\
    \le & \frac{c_2^2 \twonorm{\a}}{2\pi} \rbr{\rbr{\pi - \theta_{\ell,m}}\cos\theta_{\ell,m} + \sin\theta_{\ell,m}} \\
    \le & \frac{c_2^2 \twonorm{\a}}{2\pi} \rbr{\pi + 1} \\
    \le & c_2^2 \twonorm{\a} \frac{1+\pi}{2\pi}.
\end{align*}
Then,
\begin{align*}
    \twonorm{\nabla_{\v, \v}^2 \calL\rbr{\mtx{\theta}}}_2 \le & \sum_{\ell =1}^2\sum_{m =1}^2\left|\nabla_{v_{\ell},v_m}^2 \calL\rbr{\mtx{\theta}}\right| \\
    \le & 4 c_2^2 \twonorm{\a} \frac{1+\pi}{2\pi} = \rbr{2 + \frac{2}{\pi}} c_2^2 \twonorm{\a}.
\end{align*}

\paragraph{$\nabla_{\v, \text{vect}\rbr{\mtx{W}}}^2 \calL\rbr{\mtx{\theta}}$ term:} We have
\begin{align*}
    \twonorm{\nabla_{\v, \text{vect}\rbr{\mtx{W}}}^2 \calL\rbr{\mtx{\theta}}}_2 \le & \sum_{\ell =1}^2\sum_{m =1}^2\twonorm{\nabla_{v_{\ell},\w_m}^2 \calL\rbr{\mtx{\theta}}}.
\end{align*}
For $\ell \neq m$, we have
\begin{align*}
    \twonorm{\nabla_{v_{\ell},\w_m}^2 \calL\rbr{\mtx{\theta}}} = & \twonorm{ \frac{v_m \twonorm{\w_\ell}}{2\pi} \rbr{\rbr{\pi - \theta_{\ell,m}} \bar{\w}_\ell^T + \sin\theta_{\ell,m} \bar{\w}_m^T}} \\
    = & \frac{v_m\twonorm{\w_\ell}}{2\pi}\twonorm{\rbr{\rbr{\pi - \theta_{\ell,m}} \bar{\w}_\ell + \sin\theta_{\ell,m} \bar{\w}_m}} \\
    \le &\frac{v_m\twonorm{\w_\ell}}{2\pi}\rbr{\twonorm{\rbr{\pi - \theta_{\ell,m}} \bar{\w}_\ell} + \twonorm{\sin\theta_{\ell,m} \bar{\w}_m}} \\
    \le & \frac{v_m \twonorm{\w_\ell}}{2\pi}\rbr{\pi + 1} \\
    \le & c_2^2 \twonorm{\a} \frac{1+\pi}{2\pi}.
\end{align*}
For $\ell = m$, we have
\begin{align*}
    \twonorm{\nabla_{v_{\ell},\w_m}^2 \calL\rbr{\mtx{\theta}}} = & \twonorm{\frac{v_\ell \w_\ell^T - \a^T}{2} + \sum_{i=1}^2 \frac{v_i \twonorm{\w_i}}{2\pi} \rbr{\rbr{\pi-\theta_{\ell,i}}\bar{\w}_i^T + \sin\theta_{\ell,i} \bar{\w}_\ell^T}} \\
    \le & \frac{\twonorm{v_\ell \w_\ell}}{2} + \frac{\twonorm{\a}}{2} + \sum_{i=1}^2 \twonorm{\frac{v_i \twonorm{\w_i}}{2\pi} \rbr{\rbr{\pi-\theta_{\ell,i}}\bar{\w}_i^T + \sin\theta_{\ell,i} \bar{\w}_\ell^T}} \\
    \le & \frac{\twonorm{v_\ell \w_\ell}}{2} + \frac{\twonorm{\a}}{2} + \sum_{i=1}^2 \frac{v_i \twonorm{\w_i}}{2\pi} \rbr{\twonorm{\rbr{\pi-\theta_{\ell,i}}\bar{\w}_i} + \twonorm{\sin\theta_{\ell,i} \bar{\w}_\ell}} \\
    \le & \frac{c_2^2 \twonorm{\a}}{2} + \frac{\twonorm{\a}}{2} + \sum_{i=1}^2 \frac{c_2^2 \twonorm{\a}}{2\pi} \rbr{\pi + 1} = \rbr{\frac{1}{2} + \rbr{\frac{3}{2} + \frac{1}{\pi}}c_2^2}\twonorm{\a}.
\end{align*}
Combining both inequalities, we have 
\begin{align*}
    \twonorm{\nabla_{\v, \text{vect}\rbr{\mtx{W}}}^2 \calL\rbr{\mtx{\theta}}}_2 \le & \sum_{\ell =1}^2\sum_{m =1}^2\twonorm{\nabla_{v_{\ell},\w_m}^2 \calL\rbr{\mtx{\theta}}} \\
    \le & 2 \rbr{c_2^2 \twonorm{\a} \frac{1+\pi}{2\pi} + \rbr{\frac{1}{2} + \rbr{\frac{3}{2} + \frac{1}{\pi}}c_2^2}\twonorm{\a}} \\
    = & \rbr{1 + \rbr{4 + \frac{3}{\pi}}c_2^2}\twonorm{\a}.
\end{align*}

\paragraph{$\nabla_{\text{vect}\rbr{\mtx{W}}, \text{vect}\rbr{\mtx{W}}}^2 \calL\rbr{\mtx{\theta}}$ term:} We have
\begin{align*}
    \twonorm{\nabla_{\text{vect}\rbr{\mtx{W}}, \text{vect}\rbr{\mtx{W}}}^2 \calL\rbr{\mtx{\theta}}}_2 \le & \sum_{\ell =1}^2\sum_{m =1}^2\twonorm{\nabla_{\w_{\ell},\w_m}^2 \calL\rbr{\mtx{\theta}}}.
\end{align*}
For $\ell \neq m$, we have
\begin{align*}
    \twonorm{\nabla_{\w_{\ell},\w_m}^2 \calL\rbr{\mtx{\theta}}}_2 = & \twonorm{\frac{v_\ell v_m}{2\pi} \rbr{\Bar{\vct{w}}_\ell \Bar{\vct{w}}_{m,\ell^\perp}^T + \Bar{\vct{w}}_m \Bar{\vct{w}}_{\ell,m^\perp}^T + \rbr{\pi - \theta_{\ell, m}} \mtx{I}}}_2 \\
    = & \frac{v_\ell v_m}{2\pi}\twonorm{\Bar{\vct{w}}_\ell \Bar{\vct{w}}_{m,\ell^\perp}^T + \Bar{\vct{w}}_m \Bar{\vct{w}}_{\ell,m^\perp}^T + \rbr{\pi - \theta_{\ell, m}} \mtx{I}}_2 \\
    \le & \frac{v_\ell v_m}{2\pi}\rbr{\twonorm{\Bar{\vct{w}}_\ell \Bar{\vct{w}}_{m,\ell^\perp}^T}_2 + \twonorm{\Bar{\vct{w}}_m \Bar{\vct{w}}_{\ell,m^\perp}^T}_2 + \twonorm{\rbr{\pi - \theta_{\ell, m}} \mtx{I}}_2} \\
    \le & \frac{v_\ell v_m}{2\pi}\rbr{1 + 1 + \pi} \\
    = & v_\ell v_m \frac{2+\pi}{2\pi} \\
    \le & c_2^2 \twonorm{\a} \frac{2+\pi}{2\pi}.
\end{align*}
For $\ell = m$, we have
\begin{align*}
    \twonorm{\nabla_{\w_{\ell},\w_m}^2 \calL\rbr{\mtx{\theta}}}_2 = & \twonorm{\frac{v_\ell^2}{2}\mtx{I} + \frac{v_\ell}{2\pi\twonorm{\vct{w}_\ell}} \sum_{i=1}^2 v_i \twonorm{\vct{w}_i} \sin\rbr{\theta_{\ell, i}} \rbr{\mtx{P}_{\vct{w}_l^\perp} + \frac{\mtx{P}_{\vct{w}_l^\perp} \bar{\vct{w}}_i \bar{\vct{w}}_i^T \mtx{P}_{\vct{w}_l^\perp}}{\twonorm{\mtx{P}_{\vct{w}_l^\perp} \bar{\vct{w}}_i }^2}}}_2 \\
    \le & \twonorm{\frac{v_\ell^2}{2}\mtx{I}}_2 + \frac{v_\ell}{2\pi\twonorm{\vct{w}_\ell}} \sum_{i=1}^2 \twonorm{v_i \twonorm{\vct{w}_i} \sin\rbr{\theta_{\ell, i}} \rbr{\mtx{P}_{\vct{w}_l^\perp} + \frac{\mtx{P}_{\vct{w}_l^\perp} \bar{\vct{w}}_i \bar{\vct{w}}_i^T \mtx{P}_{\vct{w}_l^\perp}}{\twonorm{\mtx{P}_{\vct{w}_l^\perp} \bar{\vct{w}}_i }^2}}}_2 \\
    \le & \twonorm{\frac{v_\ell^2}{2}\mtx{I}}_2 + \frac{v_\ell}{2\pi\twonorm{\vct{w}_\ell}} \sum_{i=1}^2 v_i \twonorm{\vct{w}_i} \sin\rbr{\theta_{\ell, i}} \rbr{\twonorm{\mtx{P}_{\vct{w}_l^\perp}} + \twonorm{\frac{\mtx{P}_{\vct{w}_l^\perp} \bar{\vct{w}}_i \bar{\vct{w}}_i^T \mtx{P}_{\vct{w}_l^\perp}}{\twonorm{\mtx{P}_{\vct{w}_l^\perp} \bar{\vct{w}}_i }^2}}_2} \\
    \le & \frac{v_\ell^2}{2} + \frac{v_\ell}{2\pi\twonorm{\w_{\ell}}}\sum_{i=1}^2 v_i \twonorm{\w_i}(1 + 1) \\
    \le & \frac{c_2^2 \twonorm{\a}}{2} + \frac{c_2}{\pi c_1}\rbr{2 c_2^2 \twonorm{\a}} = c_2^2 \twonorm{\a} \rbr{\frac{1}{2} + \frac{2c_2}{\pi c_1}}.
\end{align*}
Combining both inequalities, we have 
\begin{align*}
    \twonorm{\nabla_{\text{vect}\rbr{\mtx{W}}, \text{vect}\rbr{\mtx{W}}}^2 \calL\rbr{\mtx{\theta}}}_2 \le & \sum_{\ell =1}^2\sum_{m =1}^2\twonorm{\nabla_{\w_{\ell},\w_m}^2 \calL\rbr{\mtx{\theta}}} \\
    \le & 2 \rbr{c_2^2 \twonorm{\a} \frac{2+\pi}{2\pi} + c_2^2 \twonorm{\a} \rbr{\frac{1}{2} + \frac{2c_2}{\pi c_1}}} \\
    = & \rbr{2 + \frac{2}{\pi} + \frac{4c_2}{\pi c_1}} c_2^2 \twonorm{\a}.
\end{align*}

\paragraph{Combining the terms:} Putting everything together,
\begin{align*}
    \twonorm{\nabla^2 \calL\rbr{\v, \W}}_2 \le & \twonorm{\nabla_{\v,\v}^2 \calL\rbr{\v, \W}}_2 + 2\twonorm{\nabla_{\v,\text{vect}\rbr{\W}}^2 \calL\rbr{\v, \W}}_2 + \twonorm{\nabla_{\text{vect}\rbr{\W},\text{vect}\rbr{\W}}^2 \calL\rbr{\v, \W}}_2 \\
    \le & \rbr{2 + \frac{2}{\pi}} c_2^2 \twonorm{\a} + 2 \rbr{1 + \rbr{4 + \frac{3}{\pi}}c_2^2}\twonorm{\a} + \rbr{2 + \frac{2}{\pi} + \frac{4c_2}{\pi c_1}} c_2^2 \twonorm{\a} \\
    = & \rbr{2 + \rbr{12 + \frac{10}{\pi} + \frac{4c_2}{\pi c_1}}c_2^2}\twonorm{\a} \\
    := & L \twonorm{\a}.
\end{align*}
This completes the proof of Lemma \ref{lemma:grad_smooth}.

%% file: sec/apx/proof_pop_lower_bound_lemma.tex
\begin{lemma}[Population Loss Lower Bound] \label{lem:loss_lower_bound}
For $v_1, v_2 > 0$ and $\theta > \frac{\pi}{2}$. We have 
\begin{align*}
    \twonorm{v_1\w_1 - \a}^2 + \twonorm{v_2\w_2 + \a}^2 \le 20 \calL\rbr{\v,\W}.
\end{align*}
\end{lemma}

\begin{proof}
Define $\tilde{\w}_i = v_i \w_i$. For $v_1,v_2 >0$, both $\twonorm{v_1\w_1 - \a}^2 + \twonorm{v_2\w_2 + \a}^2$ and $\calL\rbr{\v,\W}$ are only functions of $\tilde{\w}_1, \tilde{\w}_2$. Next, we define
\begin{align*}
    h\rbr{\tilde{\w}_1, \tilde{\w}_2, \a} &= \calL\rbr{\tilde{\w}_1, \tilde{\w}_2} - \tilde{\alpha} \rbr{\twonorm{\tilde{\w}_1 - \a}^2 + \twonorm{\tilde{\w}_2 + \a}^2}.
\end{align*}
We can write it equivalently as
\begin{align*}
    h\rbr{\tilde{\w}_1, \tilde{\w}_2, \a} &= \rbr{\frac{1}{4}-\tilde{\alpha}} \rbr{\twonorm{\tilde{\w}_1}^2 + \twonorm{\tilde{\w}_2}^2} - \frac{\rbr{\pi-\theta}\cos\theta + \sin\theta}{2\pi} \twonorm{\tilde{\w}_1}\twonorm{\tilde{\w}_2}\\
    &\quad -  \rbr{\frac{1}{2}-2\tilde{\alpha}}\a^T \rbr{\tilde{\w}_1 - \tilde{\w}_2} + \rbr{\frac{1}{2} - 2\tilde{\alpha}} \twonorm{\a}^2
\end{align*}
Noting that the expression above is quadratic in $\a$, we compute $\Tilde{h}\rbr{\tilde{\w}_1,\tilde{\w}_2} = \min\limits_{\a}  h\rbr{\tilde{\w}_1,\tilde{\w}_2,\a}$. The choice of $\a$ that minimizes the expression is $\a = \frac{\tilde{\w}_1-\tilde{\w}_2}{2}$. Plugging this in back we get,
\begin{align*}
    \Tilde{h}\rbr{\w_1,\w_2} &= \rbr{\frac{1}{4}-\tilde{\alpha}} \rbr{\twonorm{\tilde{\w}_1}^2 + \twonorm{\tilde{\w}_2}^2} - \frac{\rbr{\pi-\theta}\cos\theta + \sin\theta}{2\pi} \twonorm{\tilde{\w}_1}\twonorm{\tilde{\w}_2}\\
    &\quad -  \rbr{\frac{1}{8}-\frac{\tilde{\alpha}}{2}}\twonorm{\tilde{\w}_1 - \tilde{\w}_2}^2 \\
    &= \rbr{\frac{1}{8}-\frac{\tilde{\alpha}}{2}} \rbr{\twonorm{\tilde{\w}_1}^2 + \twonorm{\tilde{\w}_2}^2} - \frac{\rbr{\pi\rbr{\frac{1}{2} + 2\tilde{\alpha}}-\theta}\cos\theta + \sin\theta}{2\pi} \twonorm{\tilde{\w}_1}\twonorm{\tilde{\w}_2}
\end{align*}
\paragraph{Taking the norm out:} Note that we are only interested in the positivity of $\Tilde{h}$, therefore dividing it by $\twonorm{\tilde{\w}_2}^2$ does not change the sign. Denote $\frac{\twonorm{\tilde{\w}_1}}{\twonorm{\tilde{\w}_2}} = r$. Then,
\begin{align*}
    \frac{\Tilde{h}\rbr{\tilde{\w}_1,\tilde{\w}_2}}{\twonorm{\tilde{\w}_2}^2} = \rbr{\frac{1}{8}-\frac{\tilde{\alpha}}{2}} \rbr{r^2+1} - \frac{\rbr{\pi\rbr{\frac{1}{2} + 2\tilde{\alpha}}-\theta}\cos\theta + \sin\theta}{2\pi} r
\end{align*}
We note that the expression is of the form $C_1 \rbr{r^2+1} + C_2 r$. Note that the minima of this expression is achieved at $r=1$. Therefore, it is sufficient for us to check the positivity of the expression at $r=1$. To this aim we draw
\begin{align*}
    \frac{\Tilde{h}\rbr{\tilde{\w}_1,\tilde{\w}_2}}{\twonorm{\tilde{\w}_2}^2} \Bigg|_{r=1} = \rbr{\frac{1}{4}-\tilde{\alpha}} - \frac{\rbr{\pi\rbr{\frac{1}{2} + 2\tilde{\alpha}}-\theta}\cos\theta + \sin\theta}{2\pi}
\end{align*}
\begin{figure}[h]
    \centering
    \includegraphics[width=0.9\linewidth]{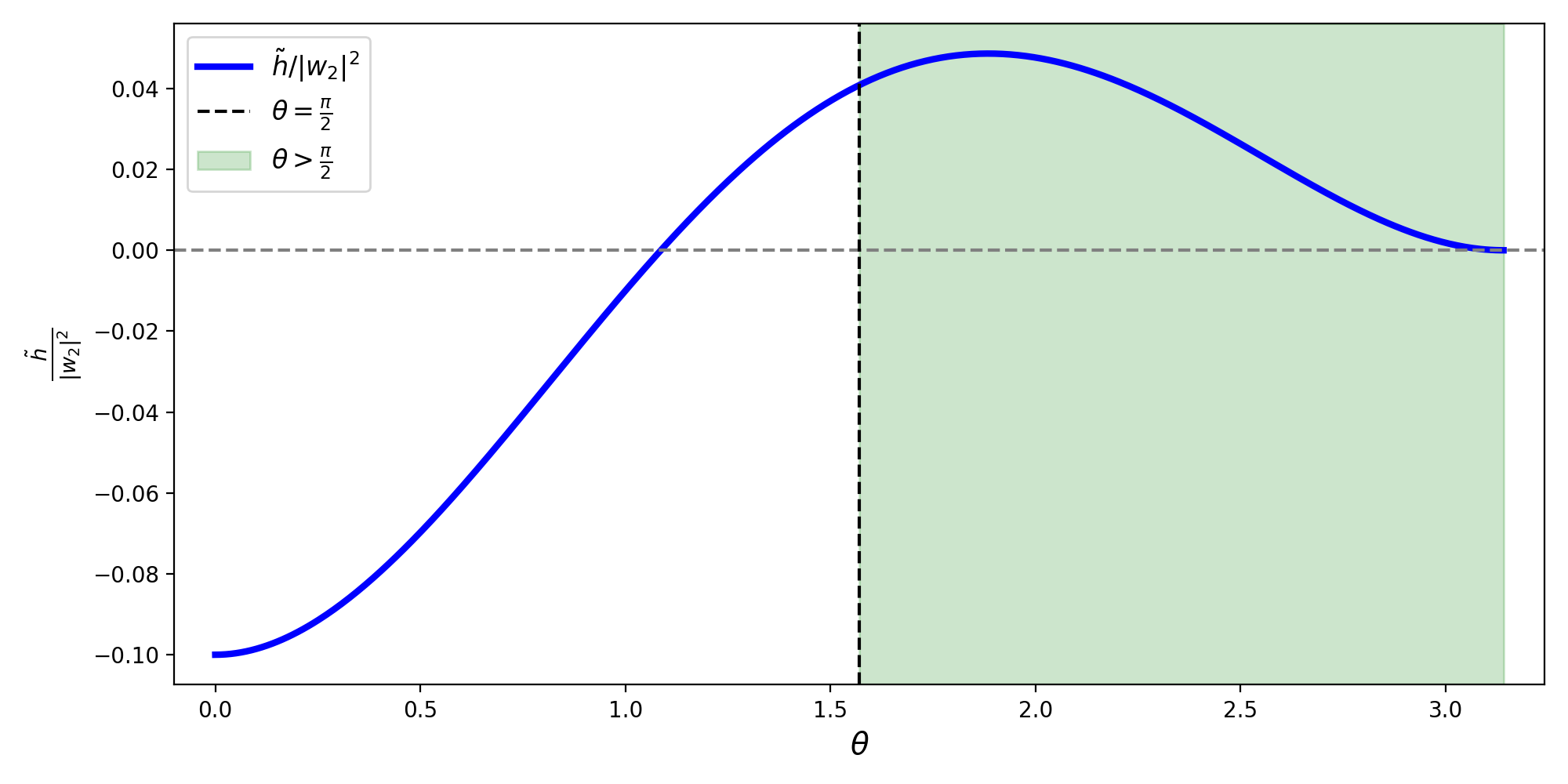}
    \caption{$\calL\rbr{\v,\W} - \tilde{\alpha} \rbr{\twonorm{v_1 \w_1 - \a}^2 + \twonorm{v_2 \w_2 + \a}^2}$ \textbf{is non-negative}. We set $\tilde{\alpha}=0.05$ and draw $\frac{1}{\twonorm{\tilde{\w}_2}^2}\Tilde{h}\rbr{\tilde{\w}_1,\tilde{\w}_2}$ for $\theta \in \bbr{0,\pi}$. We show that $\frac{1}{\twonorm{\tilde{\w}_2}^2}\Tilde{h}\rbr{\tilde{\w}_1,\tilde{\w}_2}$ is non-negative inside the shaded region ($\theta \ge \frac{\pi}{2}$).}
    \label{fig:apx_proof_pop_lower_bound}
\end{figure}
To complete the proof, in Figure \ref{fig:apx_proof_pop_lower_bound}, we set $\tilde{\alpha} = 0.05$ and draw $\frac{\Tilde{h}\rbr{\tilde{\w}_1,\tilde{\w}_2}}{\twonorm{\tilde{\w}_2}^2} \Bigg|_{r=1}$ as a 1D plot for $\theta \in \bbr{0,\pi}$. The plot demonstrates that $\Tilde{h}$ is non-negative for $\theta > \frac{\pi}{2}$. This finishes the proof of Lemma \ref{lem:loss_lower_bound}.
\end{proof}

%% file: sec/apx/imbalance_bound.tex
By symmetry, it suffices to prove the bound for $b_{1}^{(\tau+1)} - b_{1}^{(\tau)}$. We first evaluate the per-step change in the imbalance term $b_1^{(\tau)}$. By the update rule of gradient descent, we have
\begin{align*}
    b_{1}^{(\tau+1)} = & \twonorm{\w_1^{(\tau+1)}}^2 - \left(v_1^{(\tau+1)}\right)^2 \\
    = & \twonorm{\w_1 - \mu \nabla_{\w_1}\widehat{\mathcal{L}}}^2 - \left(v_1 - \mu\nabla_{v_1}\widehat{\mathcal{L}}\right)^2 \\
    = & \twonorm{\w_1}^2 - 2\mu \w_1^T\nabla_{\w_1}\widehat{\mathcal{L}} + \mu^2\twonorm{\nabla_{\w_1}\widehat{\mathcal{L}}}^2 - \left(v_1^2 -2 \mu v_1\nabla_{v_1}\widehat{\mathcal{L}} + \mu^2\left(\nabla_{v_1}\widehat{\mathcal{L}}\right)^2\right) \\
    \stackrel{(a)}{=} & \twonorm{\w_1}^2 - v_1^2 + \mu^2\twonorm{\nabla_{\w_1}\widehat{\mathcal{L}}}^2 - \mu^2\left(\nabla_{v_1}\widehat{\mathcal{L}}\right)^2 \\
    = & b_{1}^{(\tau)} + \mu^2\rbr{\twonorm{\nabla_{\w_1}\widehat{\mathcal{L}}}^2 - \left(\nabla_{v_1}\widehat{\mathcal{L}}\right)^2},
\end{align*}
where (a) follows from Eq. \ref{eq:grad_useful_identity}. It follows that
\begin{align}\label{ineq:imbalance_bound}
    \left|b_{1}^{(\tau+1)} - b_{1}^{(\tau)}\right| = & \mu^2\left|\twonorm{\nabla_{\w_1}\widehat{\mathcal{L}}}^2 - \left(\nabla_{v_1}\widehat{\mathcal{L}}\right)^2\right| \nonumber\\
    \le & \mu^2\rbr{\twonorm{\nabla_{\w_1}\widehat{\mathcal{L}}}^2 + \left(\nabla_{v_1}\widehat{\mathcal{L}}\right)^2}.
\end{align}

To bound the drift, we decompose the empirical gradients into their population counterparts and the associated estimation errors:
\begin{align}\label{ineq:imbalance_bound0}
    & \rbr{\nabla_{v_1}\widehat{\mathcal{L}}}^2 + \twonorm{\nabla_{\w_1}\widehat{\mathcal{L}}}^2 \nonumber\\
    = & \rbr{\nabla_{v_1}\mathcal{L} + \nabla_{v_1}\widehat{\mathcal{L}} - \nabla_{v_1}\mathcal{L}}^2 + \twonorm{\nabla_{\w_1}\mathcal{L} + \nabla_{w_1}\widehat{\mathcal{L}} - \nabla_{w_1}\mathcal{L}}^2 \nonumber\\
    \le & 2\rbr{\nabla_{v_1}\mathcal{L}}^2 + 2\rbr{\nabla_{v_1}\widehat{\mathcal{L}} - \nabla_{v_1}\mathcal{L}}^2 + 2\twonorm{\nabla_{\w_1}\mathcal{L}}^2 + 2\twonorm{\nabla_{w_1}\widehat{\mathcal{L}} - \nabla_{w_1}\mathcal{L}}^2 \nonumber\\
    = & 2\rbr{\rbr{\nabla_{v_1}\mathcal{L}}^2 + \twonorm{\nabla_{\w_1}\mathcal{L}}^2} + 2\rbr{\nabla_{v_1}\widehat{\mathcal{L}} - \nabla_{v_1}\mathcal{L}}^2 + 2\twonorm{\nabla_{w_1}\widehat{\mathcal{L}} - \nabla_{w_1}\mathcal{L}}^2 \nonumber\\
    = & 2\rbr{\rbr{\nabla_{v_1}\mathcal{L}}^2 + \twonorm{\nabla_{\w_1}\mathcal{L}}^2} + 2\rbr{\frac{\w_1^T}{v_1}\nabla_{\w_1}\widehat{\mathcal{L}} - \frac{\w_1^T}{v_1}\nabla_{\w_1}\mathcal{L}}^2 + 2\twonorm{\nabla_{w_1}\widehat{\mathcal{L}} - \nabla_{w_1}\mathcal{L}}^2 \nonumber\\
    \le & 2\rbr{\rbr{\nabla_{v_1}\mathcal{L}}^2 + \twonorm{\nabla_{\w_1}\mathcal{L}}^2} + 2\frac{\twonorm{\w_1}^2\twonorm{\nabla_{\w_1}\widehat{\mathcal{L}} - \nabla_{\w_1}\mathcal{L}}^2}{v_1^2} + 2\twonorm{\nabla_{w_1}\widehat{\mathcal{L}} - \nabla_{w_1}\mathcal{L}}^2 \nonumber\\
    = & 2\rbr{\rbr{\nabla_{v_1}\mathcal{L}}^2 + \twonorm{\nabla_{\w_1}\mathcal{L}}^2} + 2\rbr{1 + \frac{\twonorm{\w_1}^2}{v_1^2}}\twonorm{\nabla_{\w_1}\widehat{\mathcal{L}} - \nabla_{\w_1}\mathcal{L}}^2.
\end{align}

Regarding the population component, note that by considering a reparameterized set of weights $\tilde{v}_1 = \tilde{v}_2 =1, \tilde{\w}_1 = v_1\w_1, \tilde{\w}_2 = v_2\w_2$, we can leverage the smoothness properties of the population loss (Lemma~\ref{lemma:popgrad_smooth}): 
\begin{align*}
    \twonorm{\nabla_{\tilde{\w}_1}\mathcal{L}}^2 + \twonorm{\nabla_{\tilde{\w}_2}\mathcal{L}}^2 \le \frac{5}{2}\rbr{\twonorm{\tilde{\w}_1 -\a}^2 + \twonorm{\tilde{\w}_2 +\a}^2}.
\end{align*}
Given $v_1, v_2 >0$, it holds that $\nabla_{\tilde{\w}_1}\mathcal{L} = \frac{1}{v_1}\nabla_{\w_1}\mathcal{L}$. It then follows that: 
\begin{align}\label{ineq:imbalance_bound1}
    \rbr{\nabla_{v_1}\mathcal{L}}^2 + \twonorm{\nabla_{\w_1}\mathcal{L}}^2 = & \rbr{\frac{\w_1^T}{v_1}\nabla_{\w_1}\mathcal{L}}^2 + \twonorm{\nabla_{\w_1}\mathcal{L}}^2 \nonumber\\
    \le & \frac{\twonorm{\w_1}^2\twonorm{\nabla_{\w_1}\mathcal{L}}^2}{v_1^2} + \twonorm{\nabla_{\w_1}\mathcal{L}}^2 \nonumber\\
    = & \rbr{ \twonorm{\w_1}^2 + v_1^2}\twonorm{\nabla_{\tilde{\w}_1}\mathcal{L}}^2 \nonumber\\
    \le & \rbr{\twonorm{\w_1}^2 + v_1^2}\cdot\frac{5}{2}\rbr{\twonorm{\tilde{\w}_1 -\a}^2 + \twonorm{\tilde{\w}_2 +\a}^2} \nonumber\\
    = & \frac{5}{2}\rbr{\twonorm{\w_1}^2 + v_1^2}\rbr{\twonorm{v_1\w_1 -\a}^2 + \twonorm{v_2\w_2 +\a}^2}.
\end{align}
As for the second term, Lemma \ref{lem:component_deviation} provides the following concentration bound:
\begin{align*}
    \twonorm{\nabla_{w_1}\widehat{\mathcal{L}} - \nabla_{w_1}\mathcal{L}} \le v_1 \delta \rbr{\twonorm{v_1\w_1 -\a} + \twonorm{v_2\w_2 +\a}},
\end{align*}
where $\delta\le \frac{1}{2}$ is a constant. Squaring both sides and applying Jensen's inequality leads to:
\begin{align}\label{ineq:imbalance_bound2}
    \twonorm{\nabla_{w_1}\widehat{\mathcal{L}} - \nabla_{w_1}\mathcal{L}}^2 \le & v_1^2 \delta^2 \rbr{\twonorm{v_1\w_1 -\a} + \twonorm{v_2\w_2 +\a}}^2 \nonumber\\
    \le & 2 v_1^2 \delta^2 \rbr{\twonorm{v_1\w_1 -\a}^2 + \twonorm{v_2\w_2 +\a}^2}.
\end{align}

Substituting the bounds from \eqref{ineq:imbalance_bound1} and \eqref{ineq:imbalance_bound2} into \eqref{ineq:imbalance_bound0}, we obtain that
\begin{align*}
    & \rbr{\nabla_{v_1}\widehat{\mathcal{L}}}^2 + \twonorm{\nabla_{\w_1}\widehat{\mathcal{L}}}^2 \\
    \le & 2\rbr{\rbr{\nabla_{v_1}\mathcal{L}}^2 + \twonorm{\nabla_{\w_1}\mathcal{L}}^2} + 2\rbr{1 + \frac{\twonorm{\w_1}^2}{v_1^2}}\twonorm{\nabla_{\w_1}\widehat{\mathcal{L}} - \nabla_{\w_1}\mathcal{L}}^2 \\
    \le & 2\cdot \frac{5}{2}\rbr{\twonorm{\w_1}^2 + v_1^2}\rbr{\twonorm{v_1\w_1 -\a}^2 + \twonorm{v_2\w_2 +\a}^2}  \\
    & + 2\rbr{1 + \frac{\twonorm{\w_1}^2}{v_1^2}}\cdot 2 v_1^2 \delta^2 \rbr{\twonorm{v_1\w_1 -\a}^2 + \twonorm{v_2\w_2 +\a}^2} \\
    = & \rbr{5 + 4\delta^2}\rbr{\twonorm{\w_1}^2 + v_1^2}\rbr{\twonorm{v_1\w_1 -\a}^2 + \twonorm{v_2\w_2 +\a}^2}.
\end{align*}

Putting the above inequality into \eqref{ineq:imbalance_bound}, we complete the proof of the lemma with the constant $c_6 = 6$.

%% file: sec/apx/angle_stay_small.tex
We prove the lemma with the following constants: $c_0 \le \frac{1}{2}, c_2 = 2, c_5 = \frac{1}{50}, c_4 = \frac{\pi}{20}, \gamma = \frac{1}{4}$.

By symmetry, it suffices to prove the bound for $\theta_{1}^{(\tau+1)}$. By the update rule of gradient descent, we have
\begin{align*}
    & \w_1^{(\tau+1)} \\
    = &\w_1 - \mu\nabla_{\w_1} \widehat{\mathcal{L}} \\
    = & \w_1 - \mu\cdot\frac{v_1}{2}\left(v_1\w_1 - \left(1 - \frac{\theta}{\pi}\right)v_2\w_2 - \frac{\sin\theta}{\pi}v_2 \twonorm{\w_2}\dirw_1 -\a + \Delta\mathcal{G}_1\right) \\
    = & \left(\left(1 - \frac{\mu v_1^2}{2}\right)\twonorm{\w_1} +\frac{\mu v_1}{2}\cdot \frac{\sin\theta}{\pi}v_2 \twonorm{\w_2}\right)\dirw_1 + \frac{\mu v_1}{2} \left(\left(1 - \frac{\theta}{\pi}\right)v_2\w_2 +\a -\Delta\mathcal{G}_1\right).
\end{align*}

Given that $v_1,v_2 > 0$ and $1 - \frac{\mu v_1^2}{2}\ge 1 - \frac{c_0 c_2^2}{2} \ge 0$, the update vector $\w_1 - \mu\nabla_{\w_1} \widehat{\mathcal{L}}$ is a nonnegative linear combination of $\dirw_1$ and the vector
\[
    \vct{q} := \left(1 - \frac{\theta}{\pi}\right)v_2\w_2 +\a -\Delta\mathcal{G}_1.
\]
Recall that the angle of a positive linear combination with a reference vector $\a$ is bounded by the maximum angle of its components. Since we assume $\angle(\dirw_1,\a) \le c_4$, it suffices to show that $\angle(\vct{q},\a) \le c_4$ to conclude the proof.

Geometrically, to ensure this angle constraint on $\vct{q}$, we need to bound the perturbation magnitude $\twonorm{\Delta\mathcal{G}_1}$ by the Euclidean distance to the cone boundary. Specifically, under the assumption that $0 < v_2 \le c_2\sqrt{\twonorm{\a}}$ and $\left|b_2^{(\tau)}\right| \le \gamma\twonorm{\a}$, it follows that $\twonorm{v_2\w_2} \le c_2\sqrt{c_2^2+\gamma}\twonorm{\a}$. This implies that the condition:
\[
    \twonorm{\Delta\mathcal{G}_1} \le c_5 \twonorm{\a}\le \left(\frac{\twonorm{\a}}{2\cos c_4} - \frac{2c_4c_2\sqrt{c_2^2+\gamma}}{\pi}\twonorm{\a}\right)\sin(2c_4)
\]
is sufficient to guarantee $\angle(\vct{q},\a) \le c_4$. Numerical verification confirms that the chosen constants $c_2 = 2$, $c_4 = \frac{\pi}{20}$, $c_5 = \frac{1}{50}$, and $\gamma = \frac{1}{4}$ satisfy the required inequality. Consequently, we have $\angle(\w_1^{(\tau+1)},\a) \le c_4$, which completes the proof.

%% file: sec/apx/norm_control.tex
We prove the lemma with the following constants: $c_2 = 2, c_0 \le \frac{4}{25}, c_4 \le \frac{\pi}{10}, c_5 \le \frac{1}{3}, \gamma \le \frac{1}{2}$.

By symmetry, it suffices to prove the bound for $v_1^{(\tau+1)}$. We first show that $v_1^{(\tau+1)} \le 2\sqrt{\twonorm{\a}}$. Applying the gradient descent update rule, the partial derivative with respect to $v_1$ is bounded as follows:
\begin{align*}
    \nabla_{v_1}\widehat{\mathcal{L}} &= \frac{1}{2}\rbr{v_1\twonorm{\w_1}^2 - \frac{\rbr{\pi - \theta}\cos\theta + \sin\theta}{\pi}v_2 \twonorm{\w_1} \twonorm{\w_2} -\w_1^T\a + \w_1^T\Delta\mathcal{G}_1}  \\
    &\ge \frac{1}{2}\rbr{v_1\twonorm{\w_1}^2 - \frac{\rbr{\pi - \theta}\cos\theta + \sin\theta}{\pi} v_2 \twonorm{\w_1} \twonorm{\w_2} -\twonorm{\w_1} \twonorm{\a} - \twonorm{\w_1} \twonorm{\Delta\mathcal{G}_1}}  \\
    &= \frac{\twonorm{\w_1}}{2}\rbr{v_1\twonorm{\w_1} - \frac{\rbr{\pi - \theta}\cos\theta + \sin\theta}{\pi} v_2 \twonorm{\w_2} -\twonorm{\a} - \twonorm{\Delta\mathcal{G}_1}}  \\
    &\ge \frac{\twonorm{\w_1}}{2}\rbr{v_1\twonorm{\w_1} -\frac{1}{3}\twonorm{\a} - \twonorm{\a} -\frac{1}{3}\twonorm{\a}} \\
    &= \frac{\twonorm{\w_1}}{2}\rbr{v_1\twonorm{\w_1} -\frac{5}{3}\twonorm{\a}}.
\end{align*}
In the penultimate line, we use the assumption $\twonorm{\Delta\mathcal{G}_1} \le c_5\twonorm{\a} \le \frac{1}{3}\twonorm{\a}$ and the fact that
\begin{align*}
    \frac{\rbr{\pi - \theta}\cos\theta + \sin\theta}{\pi}v_2\twonorm{\w_2}\le \frac{\rbr{2c_4}\cos(\pi - 2c_4) + \sin(\pi - 2c_4))}{\pi}c_2\sqrt{c_2^2+\gamma}\twonorm{\a}\le \frac{1}{3}\twonorm{\a}.
\end{align*}
To establish the bound, we consider the following two cases based on the magnitude of $v_1$:
\begin{itemize}
    \item \textbf{Case 1:} $v_1 \ge \frac{5}{3}\sqrt{\twonorm{\a}}$. \\
    Since $\left|\twonorm{\w_1}^2 - v_1^2\right| =  \left|b_1\right| \le \gamma\twonorm{\a} \le \twonorm{\a}$, we have $\twonorm{\w_1}\ge \sqrt{\twonorm{\a}}$. Thus, we have
    \begin{align*}
        \nabla_{v_1}\widehat{\mathcal{L}} & \ge \frac{\twonorm{\w_1}}{2}\rbr{v_1\twonorm{\w_1} -\frac{5}{3}\twonorm{\a}} \\
        & \ge \frac{\twonorm{\w_1}}{2}\rbr{\frac{5}{3}\sqrt{\twonorm{\a}}\cdot\sqrt{\twonorm{\a}} -\frac{5}{3}\twonorm{\a}} \\
        & \ge 0,
    \end{align*}
    which means $v_1^{(\tau+1)} = v_1 - \mu \nabla_{v_1}\widehat{\mathcal{L}} \le v_1 \le 2\sqrt{\twonorm{\a}}$.
    \item \textbf{Case 2:} $v_1 < \frac{5}{3}\sqrt{\twonorm{\a}}$. \\
    We have
    \begin{align*}
        \nabla_{v_1}\widehat{\mathcal{L}} & \ge \frac{\twonorm{\w_1}}{2}\rbr{v_1\twonorm{\w_1} -\frac{5}{3}\twonorm{\a}} \\
        & \ge -\frac{5}{6}\twonorm{\w_1}\twonorm{\a},
    \end{align*}
    which means 
    \begin{align*}
        v_1^{(\tau+1)} = & v_1 - \mu \nabla_{v_1}\widehat{\mathcal{L}} \\
        \le & \frac{5}{3}\sqrt{\twonorm{\a}} + \frac{5}{6}\mu\twonorm{\w_1}\twonorm{\a} \\ 
        \le & 2\sqrt{\twonorm{\a}}.
    \end{align*}
    Here we use the fact that $\mu\twonorm{\w_1}\sqrt{\twonorm{\a}}\le c_0\sqrt{c_2^2 + \gamma} \le \frac{4}{25} \cdot \sqrt{5}\le\frac{2}{5}$.
\end{itemize}

Next, we establish the lower bound $v_1^{(\tau+1)} > \beta\sqrt{\twonorm{\a}}$. By the update rule of gradient descent, we have
\begin{align*}
    \nabla_{v_1}\widehat{\mathcal{L}} = &\frac{1}{2}\rbr{v_1\twonorm{\w_1}^2 - \frac{\rbr{\pi - \theta}\cos\theta + \sin\theta}{\pi}v_2 \twonorm{\w_1} \twonorm{\w_2} -\w_1^T\a + \w_1^T\Delta\mathcal{G}_1}  \\
    \le & \frac{1}{2}\rbr{v_1\twonorm{\w_1}^2 -\w_1^T\a + \twonorm{\w_1}\twonorm{\Delta\mathcal{G}_1}} \\
    = & \frac{\twonorm{\w_1}}{2}\rbr{v_1\twonorm{\w_1} -\cos\theta_1\twonorm{\a} + \twonorm{\Delta\mathcal{G}_1}} \\
    \le & \frac{\twonorm{\w_1}}{2}\rbr{v_1\twonorm{\w_1} -\frac{5}{6}\twonorm{\a} + \frac{1}{3}\twonorm{\a}} \\
    = & \frac{\twonorm{\w_1}}{2}\rbr{v_1\twonorm{\w_1} -\frac{1}{2}\twonorm{\a}}.
\end{align*}
In the penultimate line, we use the assumption $\twonorm{\Delta\mathcal{G}_1} \le c_5\twonorm{\a}\le\frac{1}{3}\twonorm{\a}$ and  $\theta_1 \le c_4\le \frac{\pi}{10} \le\arccos\rbr{\frac{5}{6}}$.

To establish the bound, we consider the following two cases based on the magnitude of $v_1\twonorm{\w_1}$:
\begin{itemize}
    \item \textbf{Case 1:} $v_1\twonorm{\w_1} \le \frac{1}{2}\twonorm{\a}$. \\
    we have
    \begin{align*}
        \nabla_{v_1}\widehat{\mathcal{L}} \le & \frac{\twonorm{\w_1}}{2}\rbr{v_1\twonorm{\w_1} -\frac{1}{2}\twonorm{\a}} \\
        \le & 0,
    \end{align*}
    which means $v_1^{(\tau+1)} = v_1 - \mu \nabla_{v_1}\widehat{\mathcal{L}} \ge \beta \sqrt{\twonorm{\a}}$.
    \item \textbf{Case 2:} $v_1\twonorm{\w_1} > \frac{1}{2}\twonorm{\a}$. \\
    Since $\left|\twonorm{\w_1}^2 - v_1^2\right| = \left|b_1\right| \le  \gamma\twonorm{\a} \le \frac{1}{2}\twonorm{\a}$, we have $v_1 > \frac{1}{2}\sqrt{\twonorm{\a}}$ and $\twonorm{\w_1}< \sqrt{\twonorm{\a}}$. Thus, we have
    \begin{align*}
        \nabla_{v_1}\widehat{\mathcal{L}} \le & \frac{\twonorm{\w_1}}{2}\rbr{v_1\twonorm{\w_1} -\frac{1}{2}\twonorm{\a}} \\
        \le & \frac{v_1\twonorm{\w_1}^2}{2} \\
        < & \frac{v_1\twonorm{\a}}{2}.
    \end{align*}
    It follows that $v_1^{(\tau+1)} = v_1 - \mu \nabla_{v_1}\widehat{\mathcal{L}} \ge v_1 - \frac{\mu v_1\twonorm{\a}}{2} \ge \frac{1}{2}v_1 > \frac{1}{4}\sqrt{\twonorm{\a}}\ge \beta \sqrt{\twonorm{\a}}$. Here we use the assumption $\mu\le \frac{c_0}{\twonorm{\a}} \le \frac{1}{\twonorm{\a}}$.
\end{itemize}
This concludes the proof of Lemma \ref{lemma:norm_control}.

%% file: sec/apx/angle_alignment.tex
We prove the lemma with the following constants: $c_0\le 1, c_4\le \frac{\pi}{4}, c_9 = \frac{64}{\tan c_4}, c_2 = 2, c_3 = \frac{1}{4}, c_{10} = \frac{1}{4e^{2\alpha}}, \sigma_0 = \frac{1}{8e^{2\alpha}}$, where $\alpha = \frac{65}{\tan c_4}$.

For notational simplicity, we assume without loss of generality that $\twonorm{\a} = 1$. Recall that our initialization scheme is given by
\[
\w_1^{(0)},\w_2^{(0)}\sim\mathcal N\!\Bigl(0,\tfrac{\sigma^2}{d}\mtx{I}_d\Bigr),
\qquad
v_1^{(0)},v_2^{(0)}\sim \frac{\sigma}{\sqrt{d}}\xi,\quad
\xi^2\sim\chi_d^2,
\]
where $\sigma\le\sigma_0 \sqrt{\twonorm{\a}}$ and $\chi_d^2$ denotes the chi-squared distribution with $d$ degrees of freedom.

By definition, both the squared scalar $\left(v_i^{(0)}\right)^2$ and the squared vector norm $\left\|\boldsymbol{w}_i^{(0)}\right\|^2$ follow the same Chi-squared distribution. By standard concentration inequalities, both variables concentrate sharply around the value $\sigma^2$ with probability at least $1 - O(e^{-cd})$. Furthermore, the projection $\boldsymbol{a}^T \boldsymbol{w}_i^{(0)}$ follows a Gaussian distribution $\mathcal{N}(0, \frac{\sigma^2}{d}\|\boldsymbol{a}\|^2)$, which concentrates around 0 with magnitude $O(1/\sqrt{d})$. 

Specifically, applying the Laurent-Massart concentration bounds for the Chi-squared distribution and standard Gaussian tail bounds, we have that with probability at least $1 - 8\exp\rbr{-\min\rbr{\frac{1}{16}, \frac{\hat{c}_4^2}{4}}d}$, the following inequalities hold simultaneously:
\begin{align*}
    & v_1^{(0)} \ge \frac{1}{2}\twonorm{\w_1^{(0)}}, v_2^{(0)} \ge \frac{1}{2} \twonorm{\w_2^{(0)}} \\
    & \frac{1}{2} \sigma \le v_1^{(0)}, v_2^{(0)} \le 2 \sigma \\
    & \a^T\w_1^{(0)} \ge -\hat{c}_4 v_1^{(0)}, \a^T\w_2^{(0)} \ge -\hat{c}_4 v_2^{(0)}.
\end{align*}
where $\hat{c}_4$ is a fixed positive constant.

Assume that $c_0 \le 1, c_4\le \frac{\pi}{4}$ and $T_1 = \left\lceil\frac{64}{\mu \tan c_4}\right\rceil$. We aim to establish the following properties for all iterations $\tau \le T$ via induction:
\begin{align}
    & v_1^{(\tau)} \ge \frac{1}{2}v_1^{(0)}, v_2^{(\tau)}\ge \frac{1}{2}v_2^{(0)} \label{phase1ineq1}\\
    & \a^T\w_1^{(\tau)} \ge -\hat{c}_4 v_1^{(0)}, \a^T\w_2^{(\tau)} \ge -\hat{c}_4 v_2^{(0)} \label{phase1ineq2}\\
    & v_1^{(\tau)},\twonorm{\w_1^{(\tau)}} \le 2(1+\mu)^{\tau}v_1^{(0)}\le \hat{c}_3 v_1^{(0)}\le \frac{1}{\hat{c}_5} \label{phase1ineq3}\\
    & v_2^{(\tau)},\twonorm{\w_2^{(\tau)}} \le 2(1+\mu)^{\tau}v_2^{(0)}\le \hat{c}_3 v_2^{(0)}\le \frac{1}{\hat{c}_5} \nonumber
\end{align}
with constant $\hat{c}_4 = \frac{1}{2\alpha}, \hat{c}_3 = \hat{c}_5 = 2e^{\alpha}, \sigma_0 = \frac{1}{8e^{2\alpha}}$
where $\alpha = \frac{65}{\tan c_4}$.

We prove \eqref{phase1ineq1}, \eqref{phase1ineq2}, \eqref{phase1ineq3} by induction. At initialization it is true with probability at least $1 - \tilde{C} e ^{-\tilde{c}d}$ as explained above. Assuming these hypotheses hold for some $\tau <T$, we proceed to show they remain valid for iteration $\tau+1$. By symmetry, we focus on $v_1$ and $\w_1$. For the sake of notation simplicity, we suppress the superscript $(\tau)$ where the context is clear.

We start with bound for $\a^T\w_1$. Specifically, the update rule for the alignment term yields
\begin{align*}
    & \a^T\w_1^{(\tau+1)} \\
    = & \a^T\w_1 - \mu\a^T\nabla_{\w_1}\widehat{\mathcal{L}} \\
    = & \a^T\w_1 + \frac{\mu}{2} v_1- \frac{\mu v_1}{2}\left(v_1 \a^T\w_1 - \left(1 - \frac{\theta}{\pi}\right)v_2 \a^T\w_2 - \frac{\sin\theta}{\pi}v_2 \twonorm{\w_2} \a^T\dirw_1 + \a^T\Delta\mathcal{G}_1\right).
\end{align*}

By Lemma \ref{lem:component_deviation}, with probability at least $1-3e^{-cd}$, we have $\twonorm{\mathcal{G}_1} \le \frac{1}{\hat{c}_5^2}\twonorm{\a}$. It follows that
\begin{align*}
    & \left|v_1 \a^T\w_1 - \left(1 - \frac{\theta}{\pi}\right)v_2 \a^T\w_2 - \frac{\sin\theta}{\pi}v_2 \twonorm{\w_2} \a^T\dirw_1 + \a^T\Delta\mathcal{G}_1\right| \\
    \le & v_1 \twonorm{\w_1} + v_2 \twonorm{\w_2} + v_2 \twonorm{\w_2} + \twonorm{\a^T\Delta\mathcal{G}_1} \\
    \le &\frac{4}{\hat{c}_5^2} \le \frac{1}{2}.
\end{align*}
Combining the above bounds, the term inside the parenthesis is bounded by $\frac{1}{2}$ in absolute value. Consequently, the update satisfies
\begin{align}\label{phase1ineq4}
    \a^T\w_1 + \frac{\mu}{4} v_1 \le \a^T\w_1^{(\tau+1)} \le \a^T\w_1 + \frac{3\mu}{4} v_1.
\end{align}
It follows that $\a^T\w_1^{(\tau + 1)} \ge \a^T\w_1$, By the inductive hypothesis $\a^T\w_1 \ge -\hat{c}_4 v_1^{(0)}$, we conclude that \eqref{phase1ineq2} holds for iteration $\tau+1$.

Next, we bound the orthogonal component $\twonorm{(\mtx{I} - \a\a^T)\w_1^{(\tau+1)}}$. Observe that:
\begin{align*}
    & (\mtx{I} - \a\a^T)\w_1^{(\tau+1)} \\
    = & (\mtx{I} - \a\a^T)\w_1 - \mu(\mtx{I} - \a\a^T)\nabla_{\w_1}Loss \\
    = & (\mtx{I} - \a\a^T)\w_1 - \mu\frac{v_1}{2}(\mtx{I} - \a\a^T)\left(v_1\w_1 - \left(1 - \frac{\theta}{\pi}\right)v_2\w_2 - \frac{\sin\theta}{\pi}v_2 \twonorm{\w_2}\dirw_1 -\a + \Delta\mathcal{G}_1\right).
\end{align*}

Since $\rbr{\mtx{I} - \a\a^T}\a = 0$, the update simplifies to:
\begin{align*}
    & (\mtx{I} - \a\a^T)\w_1^{(\tau+1)} \\
    = & (\mtx{I} - \a\a^T)\w_1 - \mu\frac{v_1}{2}(\mtx{I} - \a\a^T)\left(v_1\w_1 - \left(1 - \frac{\theta}{\pi}\right)v_2\w_2 - \frac{\sin\theta}{\pi}v_2 \twonorm{\w_2}\dirw_1 + \Delta\mathcal{G}_1\right).
\end{align*}
Applying the triangle inequality and substituting the bounds for $\twonorm{\w_1}$, $\twonorm{\w_2}$, and $\twonorm{\Delta\mathcal{G}_1}$, we have
\begin{align}
    & \twonorm{(\mtx{I} - \a\a^T)\w_1^{(\tau+1)}} \nonumber\\
    \le &\twonorm{(\mtx{I} - \a\a^T)\w_1} + \mu\twonorm{\frac{v_1}{2}(\mtx{I} - \a\a^T)\left(v_1\w_1 - \left(1 - \frac{\theta}{\pi}\right)v_2\w_2 - \frac{\sin\theta}{\pi}v_2 \twonorm{\w_2}\dirw_1 + \Delta\mathcal{G}_1\right)} \nonumber\\
    \le &\twonorm{(\mtx{I} - \a\a^T)\w_1} + \frac{\mu v_1}{2}(v_1 \twonorm{\w_1} + 2v_2 \twonorm{\w_2} +  \twonorm{\Delta\mathcal{G}_1}
    ) \nonumber\\
    \le &\twonorm{(\mtx{I} - \a\a^T)\w_1} + \frac{2\mu v_1}{\hat{c}_5^2}. \label{phase1ineq5}
\end{align}

Combining the upper bound for $\a^T\w_1^{(\tau+1)}$ (from inequality \eqref{phase1ineq4}) and the bound in \eqref{phase1ineq5}, we derive the upper bound for $\twonorm{\w_1^{(\tau+1)}}$ as follows:
\begin{align*}
    \twonorm{\w_1^{(\tau+1)}}^2 = &\left(\a^T\w_1^{(\tau+1)}\right)^2 + \twonorm{(\mtx{I} - \a\a^T)\w_1^{(\tau+1)}}^2 \\
    \le & \left(\a^T\w_1 + \frac{3\mu}{4} v_1\right)^2 + \left(\twonorm{(\mtx{I} - \a\a^T)\w_1} + \frac{2}{c_5^2}\mu v_1^{(0)}\right)^2 \\
    \le & \twonorm{\w_1}^2 + 2\left(\frac{3\mu}{4} v_1 + \frac{2}{c_5^2}\mu v_1^{(0)}\right)\twonorm{\w_1} + \left(\frac{3\mu}{4} v_1\right)^2+ \left(\frac{2}{c_5^2}\mu v_1\right)^2 \\
    \le & \left(\twonorm{\w_1} + \frac{3\mu}{4} v_1 + \frac{2}{c_5^2}\mu v_1^{(0)}\right)^2,
\end{align*}
Taking the square root and using the inductive hypothesis $v_1^{(\tau)},\twonorm{\w_1^{(\tau)}} \le 2(1+\mu)^{\tau}v_1^{(0)}$, along with the fact that $\frac{2}{\hat{c}_5^2}\le \frac{1}{4}$, we conclude $\twonorm{\w_1^{(\tau+1)}} \le 2(1+\mu)^{\tau+1}v_1^{(0)}$, which establishes the upper bound for $\twonorm{\w_1}$ in \eqref{phase1ineq3} for iteration $\tau+1$.

We now turn to the evolution of $v_1$. the update for $v_1$ is given by:
\begin{align*}
    &v_1^{\rbr{\tau+1}}  \\
    = &v_1 - \mu \nabla_{v_1}Loss \\
    = &v_1 - \mu \frac{\w_1^T}{v_1}\nabla_{\w_1}Loss \\
    = &v_1 - \frac{\mu}{2} \w_1^T\left(v_1\w_1 - \left(1 - \frac{\theta}{\pi}\right)v_2\w_2 - \frac{\sin\theta}{\pi}v_2 \twonorm{\w_2}\dirw_1 - \a + \Delta\mathcal{G}_1\right) \\
    = &v_1 + \frac{\mu}{2} \a^T\w_1 - \frac{\mu}{2} \w_1^T\w_1 v_1 - \frac{\mu}{2}\w_1^T\Delta\mathcal{G}_1 + \frac{\mu}{2} v_2\left(\left(1 - \frac{\theta}{\pi}\right)\w_1^T\w_2 + \frac{\sin\theta}{\pi}\twonorm{\w_1}\twonorm{\w_2}\right).
\end{align*}

By triangle inequalities, we have
\begin{align*}
    & \left| - \frac{\mu}{2} \w_1^T\w_1 v_1 - \frac{\mu}{2}\w_1^T\Delta\mathcal{G}_1 + \frac{\mu}{2} v_2\left(\left(1 - \frac{\theta}{\pi}\right)\w_1^T\w_2 + \frac{\sin\theta}{\pi}\twonorm{\w_1}\twonorm{\w_2}\right)\right| \\
    \le & \frac{\mu}{2}\twonorm{\w_1}^2v_1 + \frac{\mu}{2}\twonorm{\w_1}\twonorm{\mathcal{G}_1} + \frac{\mu}{2}v_2\twonorm{\w_1}\twonorm{\w_2} + \frac{\mu}{2}v_2\twonorm{\w_1}\twonorm{\w_2} \\
    \le & \frac{\mu}{2}\cdot 4\frac{\hat{c}_3}{\hat{c}_5^2}v_1^{(0)}.
\end{align*}

For the upper bound, we have
\begin{align*}
    v_1^{\rbr{\tau+1}} & \le v_1 + \frac{\mu}{2} \a^T\w_1 +  \frac{\mu}{2}\cdot 4\frac{\hat{c}_3}{\hat{c}_5^2}v_1^{(0)} \\
    & \le v_1 + \frac{\mu}{2} \a^T\w_1 + \frac{\mu}{2} v_1^{(0)} \\
    & \le \left(1+\frac{\mu}{2}+\frac{\mu}{2}\right)\cdot2(1+\mu)^{\tau}v_1^{(0)} \\
    & = 2(1+\mu)^{\tau +1}v_1^{(0)},
\end{align*}
which shows the bound of $v_1$ in \eqref{phase1ineq3} for iteration $\tau+1$.

For the lower bound, we have
\begin{align*}
    v_1^{\rbr{\tau+1}} & \ge v_1 + \frac{\mu}{2} \left(-\hat{c}_4 v_1^{(0)}\right)-  \frac{\mu}{2}\cdot 4\frac{\hat{c}_3}{\hat{c}_5^2}v_1^{(0)} \\
    & \ge v_1 - \frac{1}{2T_1} v_1^{(0)}.
\end{align*}
This implies that $v_1$ decrease at most one half in the first $T_1$ iterations, which shows \eqref{phase1ineq1} for iteration $\tau+1$. Here we use the fact that $\hat{c}_4 + 4\frac{\hat{c}_3}{\hat{c}_5^2} \le \frac{1}{\mu T_1} \le 1$.

Resuming the proof of Lemma \ref{lemma:angle_alignment}. Inequalities  \eqref{phase1ineq1} and \eqref{phase1ineq3} directly imply that at iteration $T_1$: 
\begin{align*}
    v_1^{(T_1)} \ge \frac{1}{2}v_1^{(0)} \ge \frac{1}{4} \sigma,
\end{align*}
and 
\begin{align*}
    v_1^{(T_1)} \le \frac{1}{\hat{c}_5} \le 2 \sqrt{\twonorm{\a}}.
\end{align*}
Additionally, \eqref{phase1ineq3} yields an upper bound on the norm of the imbalance term $\left|b_1^{(T_1)}\right|$:
\begin{align*}
    \left|b_1^{(T_1)}\right| = \left|\rbr{v_1^{(T_1)}}^2-\twonorm{\w_1^{(T_1)}}^2 \right| \le \frac{1}{\hat{c}_5^2}
\end{align*}
Next, we estimate the alignment angle $\theta_1^{(T_1)}$. Summing the updates in \eqref{phase1ineq4} and \eqref{phase1ineq5} over $T_1$ iterations, we have 
\begin{align*}
    \a^T\w_1^{(T_1)} & \ge \a^T\w_1^{(0)} + \frac{\mu}{8} v_1^{(0)}\cdot T_1 \\
    & \ge -\hat{c}_4 v_1^{(0)} + \frac{\mu}{8}T_1 v_1^{(0)}
\end{align*}
and 
\begin{align*}
    \twonorm{(\mtx{I} - \a\a^T)\w_1^{(T)}} & \le \twonorm{(\mtx{I} - \a\a^T)\w_1^{(0)}} + \frac{2c_3}{c_5^2}\mu v_1^{(0)} \cdot T_1 \\
    & \le 2v_1^{(0)} + \frac{2\hat{c}_3}{\hat{c}_5^2}\mu T_1 v_1^{(0)}.
\end{align*}
It follows that $\a^T\w_1^{(T)} \tan c_4 \ge \twonorm{(\mtx{I} - \a\a^T)\w_1^{(T)}}$, which confirms
\begin{align*}
    \theta_1^{(T_1)} = \angle(\w_1^{(T_1)},\a) \le c_4.
\end{align*}
Here we use the fact that $\mu T\left(\frac{\tan c_4}{8} - \frac{2\hat{c}_3}{\hat{c}_5^2}\right) \ge \hat{c}_4\tan c_4 +2$. By symmetry, identical bounds hold for $v_2^{(T_1)}, b_2^{(T_1)}, \theta_2^{(T_1)}$. This completes the proof of Lemma \ref{lemma:angle_alignment} with constants $c_0\le 1, c_4\le \frac{\pi}{4}, c_9 = \frac{64}{\tan c_4}, c_2 = 2, c_3 = \frac{1}{4}, c_{10} = \frac{1}{4e^{2\alpha}}, \sigma_0 = \frac{1}{8e^{2\alpha}}$, where $\alpha = \frac{65}{\tan c_4}$.

%% file: sec/apx/norm_growth.tex
We prove the lemma with the following constants: $c_0 \le \frac{1}{10^8}, c_7= \frac{1}{4}, c_8 = 32, c_{11} = \frac{1}{50}.$

We first show that for any iteration $\tau \in [T_1, T_1+T_2]$, the following bounds hold:
\begin{align}
    & 0 < v_1^{(\tau)}, v_2^{(\tau)} \le 2\sqrt{\twonorm{\a}} \label{normgrowth_normineq}\\
    & \theta_1^{(\tau)}, \theta_2^{(\tau)} \le c_4\le \frac{\pi}{10} \label{normgrowth_angleineq} \\
    & \left|b_1^{(\tau)}\right|, \left|b_2^{(\tau)}\right| \le c_{10}\twonorm{\a} + (\tau - T_1)\mu \twonorm{\a}^2 \le c_{11}\twonorm{\a}\label{normgrowth_imbalanceineq}
\end{align}
We proceed by induction. According to Lemma~\ref{lemma:angle_alignment}, equations
\eqref{normgrowth_normineq}, \eqref{normgrowth_angleineq} and \eqref{normgrowth_imbalanceineq} hold for $\tau = T_1$ with probability at least $1- Ce^{-cd}$. Now assume that we have \eqref{normgrowth_normineq}, \eqref{normgrowth_angleineq} and \eqref{normgrowth_imbalanceineq} for $\tau = T_1 + t$ with $t \in [0, T_2 - 1]$. For $\tau = T_1 + t + 1$, observe that
\begin{align*}
    \left|b_1^{(\tau)}\right|, \left|b_2^{(\tau)}\right| \le c_{10}\twonorm{\a} + t\mu \twonorm{\a}^2 \le \gamma \twonorm{\a}.
\end{align*}
By invoking Lemma~\ref{lemma:angle_stay_small}, Lemma~\ref{lemma:norm_control} with $\beta = 0$, we establish that \eqref{normgrowth_normineq} and \eqref{normgrowth_angleineq} hold for $\tau = T_1 + t + 1$. To prove \eqref{normgrowth_imbalanceineq}, we apply the Lemma~\ref{lemma:imbalance_bound} with the constant $c_6 = 6$, which yields:
\begin{align*}
    \left|b_{1}^{(\tau+1)} - b_{1}^{(\tau)}\right| & \le c_6\mu^2\rbr{\rbr{v_1^{(\tau)}}^2 + \twonorm{\w_1^{(\tau)}}^2}\rbr{\twonorm{v_1^{(\tau)}\w_1^{(\tau)} -\a}^2 + \twonorm{v_2^{(\tau)}\w_2^{(\tau)} +\a}^2} \\
    & \le c_6\mu^2\rbr{4\twonorm{\a} + 5\twonorm{\a}}\rbr{\rbr{2\sqrt{5}+1}^2\twonorm{\a}^2 + \rbr{2\sqrt{5}+1}^2\twonorm{\a}^2} \\
    & \le 3300\mu^2\twonorm{\a}^3.
\end{align*}
Here, we use the assumptions $v_1^{(\tau)}, v_2^{(\tau)} \le 2\sqrt{\twonorm{\a}}$ and the fact that $\twonorm{\w_1^{(\tau)}}^2 \le \rbr{v_1^{(\tau)}}^2 + b_1^{(\tau)} \le (4 + c_{11})\twonorm{\a} \le 5\twonorm{\a}$, $\twonorm{\w_2^{(\tau)}}^2 \le \rbr{v_2^{(\tau)}}^2 + b_2^{(\tau)} \le (4 + c_{11})\twonorm{\a} \le 5\twonorm{\a}$. Substituting these into the recursive relation for $b_1$ and using the Inequality \eqref{normgrowth_imbalanceineq} for $\tau = T_1 + t$ we have:
\begin{align*}
    \left|b_1^{(\tau + 1)}\right| \le \left|b_1^{(\tau)}\right| + \left|b_1^{(\tau + 1)} - b_1^{(\tau)}\right| \le c_{10}\twonorm{\a} + (\tau + 1 - T_1)\cdot 3300\mu^2\twonorm{\a}^3 \le c_{11}\twonorm{\a}.
\end{align*}
Here we use the assumption that $\tau + 1 - T_1 \le T_2 = \lceil\frac{c_8}{\mu \twonorm{\a}}\ln\rbr{\frac{\sqrt{\twonorm{\a}}}{\sigma}}\rceil$ and $\mu \le \frac{c_0}{\twonorm{\a}\ln\rbr{\frac{\sqrt{\twonorm{\a}}}{\sigma}}}$. By substituting the constants $c_8 = 32$, $c_{10} \leq \frac{1}{100}$, $c_{11} = \frac{1}{50}$, and $c_0 \leq \frac{1}{10^8}$, one can verify that the requirement for $c_{11}$ is satisfied. 
By symmetry, we also have $\left|b_2^{(\tau + 1)}\right| \le c_{11}\twonorm{\a}$. Thus, we have \eqref{normgrowth_imbalanceineq} for $\tau = T_1 + t + 1$.

By symmetry, we only need to focus on $v_1, \w_1$. We first establish a lower bound on $\twonorm{\w_1^{(\tau+1)}}$. Since $\w_1^T\nabla_{\w_1}\widehat{\mathcal{L}} = v_1\nabla_{v_1}\widehat{\mathcal{L}}$ (due to Eq. \ref{eq:grad_useful_identity}), we have
\begin{align*}
    \twonorm{\w_1^{(\tau+1)}}^2 = & \twonorm{\w_1 - \mu\nabla_{\w_1}\widehat{\mathcal{L}}}^2\\
    = & \twonorm{\w_1}^2 - 2\mu \w_1^T\nabla_{\w_1}\widehat{\mathcal{L}} + \mu^2\twonorm{\nabla_{\w_1}\widehat{\mathcal{L}}}^2\\
    = & \twonorm{\w_1}^2 - 2\mu v_1\nabla_{v_1}\widehat{\mathcal{L}} + \mu^2\twonorm{\nabla_{\w_1}\widehat{\mathcal{L}}}^2,
\end{align*}
and
\begin{align*}
    \twonorm{\nabla_{\w_1}\widehat{\mathcal{L}}} &\ge \frac{|\w_1^T\nabla_{\w_1}\widehat{\mathcal{L}}|}{\twonorm{\w_1}} \\
    &= \frac{|v_1\nabla_{v_1}\widehat{\mathcal{L}}|}{\twonorm{\w_1}}.
\end{align*}
Thus, we have
\begin{align*}
    \twonorm{\w_1^{(\tau+1)}}^2 & = \twonorm{\w_1}^2 - 2\mu v_1\nabla_{v_1}\widehat{\mathcal{L}} + \mu^2\twonorm{\nabla_{\w_1}\widehat{\mathcal{L}}}^2\\
    & \ge \twonorm{\w_1}^2 - 2\mu v_1\nabla_{v_1}\widehat{\mathcal{L}} + \mu^2\left(\frac{|v_1\nabla_{v_1}\widehat{\mathcal{L}}|}{\twonorm{\w_1}}\right)^2 \\
    & = \left(\twonorm{\w_1} - \mu\frac{v_1\nabla_{v_1}\widehat{\mathcal{L}}}{\twonorm{\w_1}}\right)^2.
\end{align*}
It follows that
\begin{align}\label{ineqnormgrowth1}
    \twonorm{\w_1^{(\tau+1)}} & \ge \twonorm{\w_1} - \mu\frac{v_1\nabla_{v_1}\widehat{\mathcal{L}}}{\twonorm{\w_1}}.
\end{align}
We continue by estimating $\nabla_{v_1}\widehat{\mathcal{L}}$. Note that
\begin{align*}
    \nabla_{v_1}\widehat{\mathcal{L}} &= \frac{1}{2}\rbr{v_1\twonorm{\w_1}^2 - v_2 \frac{\rbr{\pi - \theta}\cos\theta + \sin\theta}{\pi} \twonorm{\w_1} \twonorm{\w_2} -\w_1^T\a + \w_1^T\Delta\mathcal{G}_1}  \\
    &\le \frac{1}{2}\rbr{v_1\twonorm{\w_1}^2 -\w_1^T\a + \w_1^T\Delta\mathcal{G}_1} \\
    &\le \frac{1}{2}\rbr{v_1\twonorm{\w_1}^2 -\cos c_4\twonorm{\w_1}\twonorm{\a} + \twonorm{\w_1}\twonorm{\Delta\mathcal{G}_1}} \\
    &\le \frac{1}{2}\rbr{v_1\twonorm{\w_1}^2 -\frac{5}{6}\twonorm{\w_1}\twonorm{\a} + \frac{1}{3}\twonorm{\w_1}\twonorm{\a}} \\
    &\le \frac{1}{2}\rbr{v_1\twonorm{\w_1}^2 -\frac{1}{2}\twonorm{\w_1}\twonorm{\a}}.
\end{align*}
In the penultimate line, we use the assumptions $c_4 = \frac{\pi}{10} \le \arccos \left(\frac{5}{6}\right)$ and $\twonorm{\Delta\mathcal{G}_1} \le \frac{1}{3} \twonorm{\a}$.
To establish the lower bound for $v_1^{(T_1+T_2)}$, we will show that in phase 2, $v_1\twonorm{\w_1}$ increases when it is small and will never decrease too much when it is large. Specifically, we consider the following three cases based on the magnitude of $v_1\twonorm{\w_1}$:
\begin{itemize}
    \item \textbf{Case 1:} $v_1\twonorm{\w_1}\le \frac{1}{4}\twonorm{\a}$.\\
    We have
    \begin{align*}
    \nabla_{v_1}\widehat{\mathcal{L}} \le \frac{1}{2}\rbr{v_1\twonorm{\w_1}^2 -\frac{1}{2}\twonorm{\w_1}\twonorm{\a}}\le -\frac{1}{8}\twonorm{\a}\twonorm{\w_1}.
    \end{align*}
    By Inequality \eqref{ineqnormgrowth1}, we have
        \begin{align*}
        \twonorm{\w_1^{(\tau+1)}} \ge & \twonorm{\w_1} - \mu\frac{v_1\nabla_{v_1}\widehat{\mathcal{L}}}{\twonorm{\w_1}} \\
        \ge & \twonorm{\w_1} - \mu\frac{v_1(-\frac{1}{8}\twonorm{\w_1}\twonorm{\a})}{\twonorm{\w_1}} \\
        = & \twonorm{\w_1} + \frac{1}{8}\mu \twonorm{\a}v_1.
        \end{align*}

    Summing the two inequalities we conclude that
    \begin{align*}
        v_1^{(\tau+1)} + \twonorm{\w_1^{(\tau+1)}} \ge \left(1+\frac{1}{8}\mu\twonorm{\a}\right)\left(v_1^{(\tau)} + \twonorm{\w_1^{(\tau)}}\right).
    \end{align*}
    \item \textbf{Case 2:} $\frac{1}{4}\twonorm{\a} <v_1\twonorm{\w_1}\le \frac{1}{2}\twonorm{\a}$. \\
    We have
    \begin{align*}
        \nabla_{v_1}\widehat{\mathcal{L}} \le & \frac{1}{2}\rbr{v_1\twonorm{\w_1}^2 -\frac{1}{2}\twonorm{\w_1}\twonorm{\a}} \\
        \le & 0,
    \end{align*}
    which means $v_1^{(\tau+1)} = v_1 - \mu \nabla_{v_1}\widehat{\mathcal{L}} \ge v_1$. By Inequality \eqref{ineqnormgrowth1}, we have
    \begin{align*}
        \twonorm{\w_1^{(\tau+1)}} \ge & \twonorm{\w_1} - \mu\frac{v_1\nabla_{v_1}\widehat{\mathcal{L}}}{\twonorm{\w_1}} \\
    \ge & \twonorm{\w_1}.
    \end{align*}
    Thus, we have $v_1^{(\tau+1)}\twonorm{\w_1}^{(\tau+1)} \ge v_1\twonorm{\w_1} \ge \frac{1}{4}\twonorm{\a}.$

    \item \textbf{Case 3:} $v_1\twonorm{\w_1}> \frac{1}{2}\twonorm{\a}$. \\
    We have
\begin{align*}
    \nabla_{v_1}\widehat{\mathcal{L}} \le & \frac{1}{2}\rbr{v_1\twonorm{\w_1}^2 -\frac{1}{2}\twonorm{\w_1}\twonorm{\a}} \\
    \le & \frac{1}{2}v_1\twonorm{\w_1}^2,
\end{align*}
which means $v_1^{(\tau+1)} = v_1 - \mu \nabla_{v_1}\widehat{\mathcal{L}} \ge v_1 - \frac{\mu v_1}{2}\twonorm{\w_1}^2 \ge v_1 - \frac{1}{4}v_1 = \frac{3}{4}v_1$. Here we use the fact that $\mu\twonorm{\w_1}^2 \le c_0 (c_2^2+\gamma)\le \frac{1}{2}$. Using Inequality \eqref{ineqnormgrowth1}, we have
\begin{align*}
    \twonorm{\w_1^{(\tau+1)}} \ge & \twonorm{\w_1} - \mu\frac{v_1\nabla_{v_1}\widehat{\mathcal{L}}}{\twonorm{\w_1}} \\
    \ge & \twonorm{\w_1} - \frac{\mu v_1^2}{2}\twonorm{\w_1} \\
    \ge & \twonorm{\w_1} - \frac{1}{4}\twonorm{\w_1} \\
    = & \frac{3}{4}\twonorm{\w_1}.
\end{align*}
Here we use the fact that $\mu v_1^2 \le c_0 c_2^2\le \frac{1}{2}$. Thus, we have $v_1^{(\tau+1)}\twonorm{\w_1}^{(\tau+1)} \ge \frac{3}{4}v_1\cdot\frac{3}{4}\twonorm{\w_1} \ge \frac{1}{2}v_1\twonorm{\w_1} \ge \frac{1}{4}\twonorm{\a}.$
\end{itemize}
As a result, the sum $v_1+\twonorm{\w_1}$ will increase by a factor $\left(1+\frac{1}{8}\mu\twonorm{\a}\right)$ as long as $v_1\twonorm{\w_1} < \frac{1}{4}\twonorm{\a}$. Once we have $v_1\twonorm{\w_1} \ge \frac{1}{4}\twonorm{\a}$ at some iteration, it remains bounded below by $\frac{1}{4}\twonorm{\a}$. By Lemma \ref{lemma:angle_alignment}, $v_1^{(T_1)}+\twonorm{\w_1^{(T_1)}} \ge c_3\sigma \ge\frac{1}{4}\sigma$. Consequently, after $T_2 = \lceil\frac{c_8}{\mu \twonorm{\a}}\ln\rbr{\frac{\sqrt{\twonorm{\a}}}{\sigma}}\rceil$ iterations, where we have used the constant $c_8 =32$, we have $v_1^{(T_1+T_2)}\twonorm{\w_1^{(T_1+T_2)}} \ge \frac{1}{4}\twonorm{\a}$. Combining this with the imbalance bound $\left|\rbr{v_1^{(T_1+T_2)}}^2 - \twonorm{\w_1^{(T_1+T_2)}}^2 \right| \le c_{11}\twonorm{\a} \le \frac{1}{2}\twonorm{\a}$, we have $v_1^{(T_1+T_2)} \ge \frac{1}{4}\sqrt{\twonorm{\a}}$. This completes the proof with the constant $c_{11}= \frac{1}{2}$.

%% file: sec/apx/uniform_concentration.tex
We prove the result in 5 steps.

\paragraph{Step 1: Standard net reduction for operator norm.}

\begin{lemma}[Operator norm on a net]\label{lem:net}
Let $A\in\R^{d\times d}$ be symmetric and let $U\subset \Sph^{d-1}$ be an $\varepsilon$-net with $\varepsilon\in(0,1/2)$.
Then
\[
\opnorm{A}\ \le\ \frac{1}{1-2\varepsilon}\max_{u\in \mathcal{U}}|u^\top A u|.
\]
Moreover, there exists a $1/4$-net $U$ with $|U|\le 9^d$.
\end{lemma}

Hence it suffices to control, uniformly over $\w,\w^*$,
\[
\max_{u\in \mathcal{U}}\left|u^\top\Big(M(\w,\w^*)-\E M(\w,\w^*)\Big)u\right|.
\]

\paragraph{Step 2: Truncation decomposition.}

Fix $\tau\ge 1$ and define truncation and remainder for $t\ge 0$:
\[
T_\tau(t):=\min\{t,\tau\},\qquad R_\tau(t):=(t-\tau)_+.
\]
Fix $u\in\Sph^{d-1}$ and define $W_u(x):=\ip{x}{u}^2$ and
\[
W_{u}^{(\tau)}(x):=T_\tau(W_u(x)),\qquad G_{u}^{(\tau)}(x):=R_\tau(W_u(x)).
\]
For any $(\w,\w^*)$ and any $x$,
\begin{equation}\label{eq:trunc_pointwise}
W_u(x)\,\1_{\{\ip{x}{\w}\ge0\}}\1_{\{\ip{x}{\w^*}\ge0\}}
\ \le\
W_u^{(\tau)}(x)\,\1_{\{\ip{x}{\w}\ge0\}}\1_{\{\ip{x}{\w^*}\ge0\}}
\ +\ G_u^{(\tau)}(x),
\end{equation}
since $W_u = W_u^{(\tau)}+G_u^{(\tau)}$ and indicators are $\le 1$.

For fixed $u$, define the \emph{bounded} function class
\[
\mathcal{F}_{u,\tau}
:=
\Big\{
f_{\w,\w^*}(x)
:=
W_u^{(\tau)}(x)\,\1_{\{\ip{x}{\w}\ge0\}}\1_{\{\ip{x}{\w^*}\ge0\}}
\ :\ \w,\w^*\in\Sph^{d-1}
\Big\}.
\]
Then every $f\in\mathcal{F}_{u,\tau}$ satisfies $0\le f\le \tau$ pointwise.

\paragraph{Step 3: Tail remainder bound}

\begin{lemma}[Chi-square tail moments]\label{lem:chisq_tail}
Let $Z\sim\mathcal N(0,1)$ and $Y=(Z^2-\tau)_+$.
There exist universal constants $c,C>0$ such that for all $\tau\ge 1$,
\[
\E Y \le C e^{-c\tau},
\qquad
\E Y^2 \le C e^{-c\tau}.
\]
\end{lemma}

\begin{proof}
Standard: $\Pbb(Z^2>t)\le 2e^{-t/2}$ and $\E(Z^2-\tau)_+=\int_\tau^\infty \Pbb(Z^2>t)\,dt$,
$\E(Z^2-\tau)_+^2 = 2\int_\tau^\infty (t-\tau)\Pbb(Z^2>t)\,dt$.
\end{proof}

\begin{lemma}[Uniform control of the truncation tail over a $1/4$-net]\label{lem:tail_net}
Let $U$ be a $1/4$-net with $|U|\le 9^d$. There exist universal $c,C>0$ such that for all $\tau\ge 1$,
with probability at least $1-2e^{-cd}$,
\[
\max_{u\in \mathcal{U}}\left|\frac1n\sum_{i=1}^n G_u^{(\tau)}(x_i)-\E G_u^{(\tau)}(X)\right|
\ \le\ C\sqrt{\frac{e^{-c\tau}d}{n}} + C\frac{d}{n}.
\]
Moreover, $\max_{u\in \mathcal{U}}\E G_u^{(\tau)}(X)\le C e^{-c\tau}$.
\end{lemma}

\begin{proof}
Fix $u\in \mathcal{U}$. Then $G_u^{(\tau)}(X)\stackrel{d}{=}(Z^2-\tau)_+$ with $Z\sim\mathcal N(0,1)$.
By Lemma \ref{lem:chisq_tail}, $\E G_u^{(\tau)}\lesssim e^{-c\tau}$ and $\E(G_u^{(\tau)})^2\lesssim e^{-c\tau}$,
and $G_u^{(\tau)}$ is sub-exponential (dominated by $Z^2$).
More specifically, it is a $(\sigma^2,b)$ subexponential with $\sigma^2\le Ce^{-c\tau}$ and $b$ a constant.
\begin{definition}[Sub-exponential random variable with parameters $(\sigma^2,b)$]
A real-valued random variable $X$ is said to be \emph{sub-exponential} with parameters
$(\sigma^2,b)$ if, for all $|\lambda|\le \frac{1}{b}$,
\begin{equation}\label{eq:subexp-mgf}
\mathbb{E}\exp\!\big(\lambda (X-\mathbb{E}X)\big)
\;\le\;
\exp\!\left(\frac{\lambda^2\sigma^2}{2}\right).
\end{equation}
\end{definition}
For such subexponential random variables we have the following standard refined Bernstein-type inequality.
\begin{theorem}[Bernstein's-type inequality for sub-exponential sums]\label{thm:bernstein-subexp}
Let $X_1,\dots,X_n$ be independent mean-zero random variables, where each $X_i$ are
sub-exponential with parameters $(\sigma^2,b)$ in the sense of
\eqref{eq:subexp-mgf}. Then, for every $s\ge 0$,
\begin{equation}\label{eq:bernstein-subexp}
\mathbb{P}\!\left(\frac{1}{n}\sum_{i=1}^n X_i \ge s\right)
\;\le\;
2\exp\!\left(
-\frac{1}{2}n\cdot\min\left\{
\frac{s^2}{\sigma^2},\;
\frac{s}{b}
\right\}
\right).
\end{equation}
\end{theorem}
Thus using using this Bernstein's inequality for sub-exponential variables above with $s=C\sqrt{\frac{e^{-c\tau}t}{n}}+C\frac{t}{n}$ yields for all $t\ge 1$,
\[
\Pbb\!\left(\frac1n\sum_{i=1}^n G_u^{(\tau)}(x_i) - \E G_u^{(\tau)}(x) \ge
C\sqrt{\frac{e^{-c\tau}t}{n}}+C\frac{t}{n}\right)\le e^{-t}.
\]
Set $t=c_1 d$ and union bound over $|\mathcal{U}|\le 9^d$; choosing $c_1$ large enough makes the union-bound failure probability $\le e^{-cd}$.
\end{proof}

\paragraph{Step 4: Uniform control of the truncated term.} We will use the following result on  covering numbers for VC-subgraph classes.

\begin{theorem}[Theorem 2.6.7 in \cite{vanderVaart1996}]
\label{lem:VCentropy}
Let $\mathcal H$ be a VC-subgraph class of real-valued functions on $\R^d$ with VC-subgraph dimension at most $V$ and envelope bound $|h|\le B$ pointwise.
Then there exist absolute constants $A,C>0$ such that for every probability measure $Q$ and every $0<\eta\le B$,
\[
\log N(\eta,\mathcal H,L_2(Q))\ \le\ C V \log\!\Bigl(\frac{AB}{\eta}\Bigr).
\]
\end{theorem}

We will also use the following Dudley-type bound for Rademacher averages. 

\begin{theorem}[Equation (5.48) in \cite{Wainwright_book}]
\label{lem:dudley}
There exists an absolute constant $C>0$ such that, for any function class $\mathcal F$,
conditionally on $x_1,\dots,x_n$,
\[
\E_\varepsilon\Big[\sup_{f\in\mathcal F} \Big|\frac1n\sum_{i=1}^n \varepsilon_i f(x_i)\Big|\Big]
\le
\frac{C}{\sqrt n}\int_{0}^{2r}\sqrt{\log N(\eta,\mathcal F, L_2(\mathbb{P}_n))}\,d\eta,
\]
where $\mathbb{P}_n$ is the empirical measure of $x_1, \ldots, x_n$ and $r:=\sup_{f\in\mathcal F}\|f\|_{L_2(\mathbb{P}_n)}$.
\end{theorem}

\begin{lemma}[Expected supremum for the truncated class]\label{lem:Esup}
There exists a universal constant $C>0$ such that for each fixed $u$,
\[
\E\Big[\sup_{f\in\mathcal{F}_{u,\tau}} |(\mathbb{P}_n-\mathbb{P})f|\Big]
\ \le\
C\tau\sqrt{\frac{ d}{n}}.
\]
\end{lemma}

\begin{proof}
By symmetrization,
\[
\E\sup_{f\in\mathcal{F}_{u,\tau}} |(\mathbb{P}_n-\mathbb{P})f|
\le 2\E\sup_{f\in\mathcal{F}_{u,\tau}}\Big|\frac1n\sum_{i=1}^n \varepsilon_i f(x_i)\Big|.
\]
Let $r:=\sup_{f\in\mathcal{F}_{u,\tau}}\|f\|_{L_2(\mathbb{P}_n)}$. Since $0\le f\le \tau$ we have $r\le \tau$. Applying Dudley's entropy integral bound for Rademacher averages per Theorem \ref{lem:dudley} and the covering bound from Theorem \ref{lem:VCentropy} with $Q=\mathbb{P}_n$ yields
\[
\E\Big[\sup_{f\in \mathcal{F}}\Big|\frac1n\sum_{i=1}^n \varepsilon_i f(X_i)\Big|\Big]
\ \le\
\frac{C_3}{\sqrt n}\int_0^{2r}\sqrt{\log N\big(\eta,\mathcal{F},L_2(\mathbb{P}_n)\big)}\,d\eta
\ \le\
\frac{C_3}{\sqrt n}\int_0^{2r}\sqrt{V\log\Big(\frac{A\tau}{\eta}\Big)}\,d\eta.
\]
Using the change of variables $\eta=re^{-s}$ and the elementary bound
\[
\int_0^{2r}\sqrt{\log\Big(\frac{A\tau}{\eta}\Big)}\,d\eta
\ \le\
C_4\, r\,\sqrt{\log\Big(\frac{eA\tau}{r}\Big)},
\]
we obtain
\[
\E\Big[\sup_{f\in F}\Big|\frac1n\sum_{i=1}^n \varepsilon_i f(X_i)\Big|\Big]
\ \le\
C_5\,\frac{r}{\sqrt n}\,\sqrt{V}\,
\sqrt{\log\Big(\frac{eA\tau}{r}\Big)}.
\]
Substituting $r\le \tau$ and $V\le c_2 d$, and absorbing constants (including $A$ and $C_0$) into a single absolute constant $C$, gives
\[
\E\Big[\sup_{f\in F}\Big|\frac1n\sum_{i=1}^n \varepsilon_i f(X_i)\Big|\Big]
\ \le\
C\tau\,\sqrt{\frac{ d}{n}}.
\]
Finally, multiplying by $2$ from symmetrization proves the claim.
\end{proof}
Now we focus on establishing a high-probability bound. For this we will use Bousquet's concentration inequality for suprema of bounded empirical processes. 

\begin{theorem}[Theorem 2.3 in \cite{Bousquet2002}]
\label{lem:bousquet}
Let $\mathcal{F}$ be a class of measurable functions with $0\le f\le b$.
Let
\[
Z:=\sup_{f\in\mathcal{F}} \sum_{i=1}^n\big(f(X_i)-\E f(X)\big),
\qquad
\sigma^2 := \sup_{f\in\mathcal{F}} \mathrm{Var}(f(X)).
\]
Then for all $t\ge 0$, with probability at least $1-e^{-t}$,
\begin{equation}\label{eq:Bous1}
Z \le \E Z + \sqrt{2t(n\sigma^2+2b\E Z)} + \frac{b\,t}{3}.
\end{equation}
\end{theorem}
In particular, we will further upper bound the RHS of \eqref{eq:Bous1} as 
\begin{equation}\label{eq:Bous2}
\E Z + \sqrt{2t(n\sigma^2+2b\E Z)} + \frac{b\,t}{3}\le 2\E Z + \sqrt{2tn\sigma^2} + \frac{4b\,t}{3},
\end{equation}
where we have used that $\sqrt{a+b}\le \sqrt{a}+\sqrt{b}$.

\begin{lemma}[Uniform high-probability deviation for $\mathcal{F}_{u,\tau}$]\label{lem:dev_trunc}
There exist universal constants $C,c>0$ such that for each fixed $u$ and each $t\ge 1$,
with probability at least $1-e^{-t}$,
\[
\sup_{f\in\mathcal{F}_{u,\tau}} |(\mathbb{P}_n-\mathbb{P})f|
\le
C\tau\sqrt{\frac{ d}{n}}
+ C\sqrt{\frac{\tau t}{n}}
+ C\frac{\tau t}{n}.
\]
\end{lemma}

\begin{proof}
Apply \eqref{eq:Bous2} to $\mathcal{F}=\mathcal{F}_{u,\tau}$ with $b=\tau$.
We have $\sigma^2\le \sup_f \E f^2 \le \tau/2$ as above.
Also, $\E Z = n\cdot \E\sup_{f\in\mathcal{F}_{u,\tau}} |(\mathbb{P}_n-\mathbb{P})f|$ which is bounded by Lemma \ref{lem:Esup}.
Divide by $n$ to conclude.
\end{proof}

\paragraph{Step 5: Complete the operator-norm bound} Fix $\tau>0$ and recall the decomposition
\[
W_u(x) \;=\; W_u^{(\tau)}(x) + G_u^{(\tau)}(x),
\qquad
W_u^{(\tau)}(x):=\min\{\langle x,u\rangle^2,\tau\},
\qquad
G_u^{(\tau)}(x):=(\langle x,u\rangle^2-\tau)_+.
\]
For $(\w,\w^\star)$ define the ReLU sign indicator
\[
\1_{\w,\w^\star}(x)
:=\1\{\langle x,\w\rangle\ge 0\}\,\1\{\langle x,\w^\star\rangle\ge 0\}\in\{0,1\}.
\]
Then for every $x$,
\[
W_u(x)\, \1_{\w,\w^\star}(x)
=
W_u^{(\tau)}(x)\, \1_{\w,\w^\star}(x)
+
G_u^{(\tau)}(x)\, \1_{\w,\w^\star}(x),
\]
and hence, by the triangle inequality,
\begin{equation}
\label{eq:step34-triangle}
\sup_{\w,\w^\star}\big|(\mathbb{P}_n-\mathbb{P})\big(W_u \1_{\w,\w^\star}\big)\big|
\;\le\;
\sup_{\w,\w^\star}\big|(\mathbb{P}_n-\mathbb{P})\big(W_u^{(\tau)} \1_{\w,\w^\star}\big)\big|
\;+\;
\sup_{\w,\w^\star}\big|(\mathbb{P}_n-\mathbb{P})\big(G_u^{(\tau)} \1_{\w,\w^\star}\big)\big|.
\end{equation}

We control the second term in \eqref{eq:step34-triangle} uniformly over $(\w,\w^\star)$ without any VC/covering argument.
Using the identity
\[
G_u^{(\tau)} \1_{\w,\w^\star}
=
G_u^{(\tau)} - G_u^{(\tau)}(1-\1_{\w,\w^\star}),
\]
we have for every $(\w,\w^\star)$,
\[
(\mathbb{P}_n-\mathbb{P})\!\big(G_u^{(\tau)} \1_{\w,\w^\star}\big)
=
(\mathbb{P}_n-\mathbb{P})\!\big(G_u^{(\tau)}\big)
-
(\mathbb{P}_n-\mathbb{P})\!\big(G_u^{(\tau)}(1-\1_{\w,\w^\star})\big),
\]
so
\begin{equation}
\label{eq:GI-split}
\big|(\mathbb{P}_n-\mathbb{P})\!\big(G_u^{(\tau)} \1_{\w,\w^\star}\big)\big|
\;\le\;
\big|(\mathbb{P}_n-\mathbb{P})\!\big(G_u^{(\tau)}\big)\big|
+
\big|(\mathbb{P}_n-\mathbb{P})\!\big(G_u^{(\tau)}(1-\1_{\w,\w^\star})\big)\big|.
\end{equation}
Since $G_u^{(\tau)}\ge 0$ and $0\le 1-\1_{\w,\w^\star}\le 1$, we have pointwise
\[
0 \le G_u^{(\tau)}(1-\1_{\w,\w^\star}) \le G_u^{(\tau)},
\]
and therefore
\[
\mathbb{P}_n\!\big(G_u^{(\tau)}(1-\1_{\w,\w^\star})\big)\le \mathbb{P}_n\!\big(G_u^{(\tau)}\big),
\qquad
\mathbb{P}\!\big(G_u^{(\tau)}(1-\1_{\w,\w^\star})\big)\le \mathbb{P}\!\big(G_u^{(\tau)}\big).
\]
Consequently,
\[
\big|(\mathbb{P}_n-\mathbb{P})\!\big(G_u^{(\tau)}(1-\1_{\w,\w^\star})\big)\big|
\le
\mathbb{P}_n\!\big(G_u^{(\tau)}\big)+\mathbb{P}\!\big(G_u^{(\tau)}\big)
=
(\mathbb{P}_n-\mathbb{P})\!\big(G_u^{(\tau)}\big)+2\,\mathbb{P}\!\big(G_u^{(\tau)}\big),
\]
and hence
\begin{equation}
\label{eq:GI-uniform-tail}
\sup_{\w,\w^\star}\big|(\mathbb{P}_n-\mathbb{P})\!\big(G_u^{(\tau)} \1_{\w,\w^\star}\big)\big|
\;\le\;
2\,\big|(\mathbb{P}_n-\mathbb{P})\!\big(G_u^{(\tau)}\big)\big|
\;+\;
2\,\mathbb{P}\!\big(G_u^{(\tau)}\big).
\end{equation}

\medskip
\noindent Plugging \eqref{eq:GI-uniform-tail} into \eqref{eq:step34-triangle} yields the corrected completion bound:
\begin{equation}
\label{eq:step34-final}
\sup_{\w,\w^\star}\big|(\mathbb{P}_n-\mathbb{P})\big(W_u \1_{\w,\w^\star}\big)\big|
\;\le\;
\sup_{\w,\w^\star}\big|(\mathbb{P}_n-\mathbb{P})\big(W_u^{(\tau)} \1_{\w,\w^\star}\big)\big|
\;+\;
2\,\big|(\mathbb{P}_n-\mathbb{P})\!\big(G_u^{(\tau)}\big)\big|
\;+\;
2\,\mathbb{P}\!\big(G_u^{(\tau)}\big).
\end{equation}
Now apply Lemma \ref{lem:dev_trunc} and Lemma \ref{lem:tail_net} and union bound over $u\in \mathcal{U}$.
Taking $t=c_0 d$ with $c_0$ large enough absorbs the $9^d$ factor, giving with probability $\ge 1-3e^{-cd}$:
\[
\max_{u\in \mathcal{U}}\sup_{\w,\w^*}\Big|u^\top (M-\E M)u\Big|
\le
C\tau\sqrt{\frac{ d}{n}}
+C\sqrt{\frac{\tau d}{n}}+ C\frac{\tau d}{n}
+ C\sqrt{\frac{e^{-c\tau}d}{n}} +Ce^{-c\tau}+ C\frac{d}{n}.
\]
Multiplying by $2$ from the net lemma gives the same bound for
$\sup_{\w,\w^*}\opnorm{M-\E M}$.

Finally choose $\tau=C\log(e/\delta)$ so that $e^{-c\tau}\ll \delta$.
If $n\ge C d\,\frac{\log^2(1/\delta)}{\delta^2}$, then each term on the right is $\le \delta$ (after increasing the universal constants),
which proves Theorem \ref{lem:unif_concentration}.
\qed

%% file: sec/apx/grad_concentration.tex
For the proof, we utilize the variational characterization of the Euclidean norm: $\twonorm{\vct{z}} = \sup_{\twonorm{\mathbf{u}}=1} \inner{\mathbf{u}}{\vct{z}}$.

\paragraph{Bound for $\w_1$:}
Let $\mathbf{u} \in \mathbb{S}^{d-1}$ be an arbitrary unit vector. Using the Fundamental Theorem of Calculus, the definition of the ReLU gradient, and the residual $r(\x_i) = v_1\phi(\w_1^\top \x_i) - v_2\phi(\w_2^\top \x_i) - (\phi(\a^\top \x_i) - \phi(-\a^\top \x_i))$, we observe:
\begin{align*}
    \inner{\mathbf{u}}{\nabla_{\w_1} \widehat{\mathcal{L}} - \nabla_{\w_1} \mathcal{L}} &= v_1 \int_0^1 \mathbf{u}^\top \left( M(t(v_1\w_1) + (1-t)\a, \w_1) - \mathbb{E}M(\dots) \right) \mathbf{h}_1 \, dt \\
    &\quad + v_1 \int_0^1 \mathbf{u}^\top \left( M(-t\a + (1-t)(v_2\w_2), \w_1) - \mathbb{E}M(\dots) \right) (-\mathbf{h}_2) \, dt,
\end{align*}
where $M(\tilde{\mathbf{u}}, \tilde{\mathbf{v}})$ is defined as in Eq. \ref{eq:defM}. The term $v_1$ appears due to the chain rule derivative with respect to $\w_1$. By Lemma \ref{lem:unif_concentration}, given the sample complexity $n \ge C d \frac{\log(1/\delta)^2}{\delta^2}$, the following spectral deviation bound holds with probability at least $1-3e^{-cd}$:
\begin{equation}
    \sup_{\tilde{\mathbf{u}}, \tilde{\mathbf{v}} \in \mathbb{S}^{d-1}} \twonorm{M(\tilde{\mathbf{u}}, \tilde{\mathbf{v}}) - \mathbb{E}\bbr{M(\tilde{\mathbf{u}}, \tilde{\mathbf{v}})}} \le \delta.
\end{equation}
Since the indicator functions $\1_{\{\langle \x, \mathbf{u} \rangle \ge 0\}}$ are scale-invariant, this uniform bound applies to all directions $t\mathbf{u} + (1-t)\mathbf{v}$ appearing in the integrals. Then,
\begin{align*}
    \left|\inner{\mathbf{u}}{\nabla_{\w_1} \widehat{\mathcal{L}} - \nabla_{\w_1} \mathcal{L}} \right| &\le v_1 \delta \twonorm{\mathbf{u}} \twonorm{\mathbf{h}_1} + v_1 \delta \twonorm{\mathbf{u}} \twonorm{\mathbf{h}_2} \\
    &= v_1 \delta (\twonorm{\mathbf{h}_1} + \twonorm{\mathbf{h}_2}).
\end{align*}
Taking the supremum over $\mathbf{u}$ proves the first bound.

\paragraph{Bound for $\w_2$:}
Similarly, for an arbitrary unit vector $\mathbf{v} \in \mathbb{S}^{d-1}$:
\begin{align*}
    \inner{\mathbf{v}}{\nabla_{\w_2} \widehat{\mathcal{L}} - \nabla_{\w_2} \mathcal{L}} &= v_2 \int_0^1 \mathbf{v}^\top \left( M(t(v_2\w_2) - (1-t)\a, \w_2) - \mathbb{E}M(\dots) \right) \mathbf{h}_2 \, dt \\
    &\quad + v_2 \int_0^1 \mathbf{v}^\top \left( M(t\a + (1-t)(v_1\w_1), \w_2) - \mathbb{E}M(\dots) \right) (-\mathbf{h}_1) \, dt.
\end{align*}
Applying the same uniform spectral bound yields:
\begin{align*}
    \left|\inner{\mathbf{v}}{\nabla_{\w_2} \widehat{\mathcal{L}} - \nabla_{\w_2} \mathcal{L}}\right| &\le v_2 \delta (\twonorm{\mathbf{h}_1} + \twonorm{\mathbf{h}_2}).
\end{align*}
This completes the proof of Lemma \ref{lem:component_deviation}.

\begin{corollary}[Bounds in terms of $\|\a\|$]
Further assume that $\twonorm{v_1 \w_1} \le C \twonorm{\a}$ and $\twonorm{v_2 \w_2} \le C \twonorm{\a}$. Then:
\begin{align}
    \twonorm{\nabla_{\w_1} \widehat{\mathcal{L}} - \nabla_{\w_1} \mathcal{L}} &\le \tilde{c} v_1 \twonorm{\a}, \\
    \twonorm{\nabla_{\w_2} \widehat{\mathcal{L}} - \nabla_{\w_2} \mathcal{L}} &\le \tilde{c} v_2 \twonorm{\a},
\end{align}
where $\tilde{c} = 2\delta(C+1)$.
\end{corollary}

\begin{proof}
We bound the norms of the error vectors $\mathbf{h}_1$ and $\mathbf{h}_2$ using the triangle inequality:
\begin{align*}
    \twonorm{\mathbf{h}_1} &= \twonorm{v_1 \w_1 - \a} \le \twonorm{v_1 \w_1} + \twonorm{\a} \le (C+1)\twonorm{\a}, \\
    \twonorm{\mathbf{h}_2} &= \twonorm{v_2 \w_2 + \a} \le \twonorm{v_2 \w_2} + \twonorm{\a} \le (C+1)\twonorm{\a}.
\end{align*}
Substituting these into the theorem's bounds:
\begin{align*}
    \twonorm{\nabla_{\w_1} \widehat{\mathcal{L}} - \nabla_{\w_1} \mathcal{L}} &\le v_1 \delta \rbr{2(C+1)\twonorm{\a}} = 2\delta(C+1) v_1 \twonorm{\a}, \\
    \twonorm{\nabla_{\w_2} \widehat{\mathcal{L}} - \nabla_{\w_2} \mathcal{L}} &\le v_2 \delta \rbr{2(C+1)\twonorm{\a}} = 2\delta(C+1) v_2 \twonorm{\a}.
\end{align*}
\end{proof}

%% file: sec/apx/proof_landscape.tex
To prove this theorem, we first show that $v_1 \w_1 = \a$, $\v_2 \w_2 = -\a$ is the global optima. Since $v_1, v_2 > 0$,
\begin{align*}
    v_1 \phi\rbr{\w_1^T \x} - v_2 \phi \rbr{\w_2^T \x} = \phi\rbr{v_1 \w_1^T \x} - \phi\rbr{v_2 \w_2^T \x} = \phi\rbr{\a^T \x} - \phi\rbr{-\a^T \x} = \a^T \x.
\end{align*}
Hence, the given weights implement the planted model exactly. Next, we verify that all $v_1, v_2>0$, and $\w_1, \w_2$ that satisfy
\begin{align*}
v_1 \w_1 - v_2 \w_2 = \a, \quad \text{and} \quad \theta = 0
\end{align*}
are indeed non-strict saddle points of our optimization problem when $k=2$. We first show that the gradient vanishes. Plugging such $v_1, v_2, \w_1, \w_2$ into (\ref{eq:pop_loss}):
\begin{align*} 
    \nabla_{\mtx{W}}\calL\rbr{\mtx{\theta}} =& \frac{1}{2\pi} diag \rbr{\vct{v}} \rbr{\rbr{\pi \mathbbm{1}\mathbbm{1}^T - \mtx{\Theta}} diag\rbr{\vct{u}} + diag \rbr{\sin\rbr{\mtx{\Theta}} \vct{u}}}\mtx{\Bar{W}} - \frac{1}{2}\vct{v}\vct{a}^T \\
    &\stackrel{(a)}{=} \frac{1}{2} diag \rbr{\vct{v}} \mathbbm{1}\mathbbm{1}^T diag\rbr{\vct{v}}\mtx{\W} - \frac{1}{2}\vct{v}\vct{a}^T \\
    &= \frac{1}{2} \vct{v} \rbr{\mtx{W}^T\vct{v} - \vct{a}}^T \\
    &= \frac{1}{2} \vct{v} \rbr{v_1\vct{w}_1 - v_2\vct{w}_2 - \vct{a}}^T = \frac{1}{2} \vct{v} \rbr{\vct{a} - \vct{a}}^T = \mtx{0}.
\end{align*}
where (a) follows from the fact that $\mtx{\Theta} = \mtx{0}$ at these points. Furthermore, due to \eqref{eq:grad_useful_identity}, $\nabla_{\v} \calL\rbr{\mtx{\theta}}$ is also $\vct{0}$. Next we show that the Hessian at these points are PSD. Plugging the values into (\ref{eq:pop_hessian}) we get:
\begin{align*} 
   \nabla^2 \mathcal{L}\rbr{\vct{\theta}} = \frac{1}{2} \begin{bmatrix}
       \twonorm{\w_1}^2 & -\twonorm{\w_1}\twonorm{\w_2} & v_1 \w_1^T & -v_2 \w_1^T \\
       -\twonorm{\w_1}\twonorm{\w_2} & \twonorm{\w_2}^2 & -v_1 \w_2^T & v_2 \w_2^T \\
       v_1 \w_1 & -v_1 \w_2 & v_1^2 \mtx{I} & -v_1 v_2 \mtx{I} \\
       -v_2 \w_1 & v_2 \w_2 & -v_1 v_2 \mtx{I} & v_2^2 \mtx{I}  \end{bmatrix}
\end{align*}
which follows from the fact that $\theta_{\ell,i} = 0$ and $\Bar{\vct{w}}_{m,\ell^\perp} = \Bar{\vct{w}}_{\ell,m^\perp} = \vct{0}$ for any choice of $\ell, m, i \in \bbr{2}$. This $\rbr{2d + 2} \times \rbr{2d+2}$ matrix has eigenvalues $0$, $\frac{v_1^2 + v_2^2}{2}$, and $\frac{\twonorm{\w_1}^2 + \twonorm{\w_2}^2 + v_1^2 + v_2^2}{2}$ (all non-negative) with multiplicities $d+2, d-1$, and $1$ respectively. Therefore, all the stationary points are in fact non-strict saddle points of the problem.

Finally, we show that there are no other stationary points besides the ones identified above. A necessary condition for $\nabla_{\mtx{W}}\calL\rbr{\mtx{\theta}} = \mtx{0}$ is that any linear combination of the gradient rows must vanish. Specifically, for $v_1, v_2 > 0$, we have:
\begin{align*}
    \begin{bmatrix}
        \frac{1}{v_1}, & \frac{1}{v_2} 
    \end{bmatrix} \nabla_{\mtx{W}}\calL\rbr{\mtx{\theta}} = \mtx{0}.
\end{align*}
By substituting the gradient expression, this implies:
\begin{align*}
    \theta \rbr{v_1 \w_1 + v_2 \w_2} = \sin\theta \rbr{v_1 \twonorm{\w_1} \bar{\w}_2 + v_2 \twonorm{\w_2} \bar{\w}_1}.
\end{align*}
Note that the vectors $v_1 \w_1 + v_2 \w_2$ and $v_1 \twonorm{\w_1} \bar{\w}_2 + v_2 \twonorm{\w_2} \bar{\w}_1$ have identical norms. Taking the norm of both sides, the equality holds only if $|\theta| = |\sin\theta|$, which implies $\theta = 0$, or if the vectors themselves are zero ($v_1 \w_1 + v_2 \w_2 = \mtx{0}$). 

The case $\theta = 0$ corresponds to the non-strict saddle points previously identified. The case $v_1 \w_1 + v_2 \w_2 = \mtx{0}$ corresponds to the global optima where the two neurons are anti-aligned ($\theta = \pi$) such that their combined contribution exactly implements the target $\a$. Consequently, there are no other stationary points in the optimization landscape. This completes the proof of the theorem.